\documentclass[11pt,twoside]{article}

\usepackage[abbrvbib, preprint]{jmlr2e}

\usepackage{epsf}
\usepackage{fancyhdr}
\usepackage{graphics}
\usepackage{graphicx}
\usepackage{psfrag}
\usepackage{float}
\usepackage{multirow}

\usepackage{multirow}
\usepackage[linesnumbered,ruled]{algorithm2e}

\usepackage{color}

\usepackage{amsfonts}
\usepackage{amsmath}
\usepackage{amssymb,bbm}

\usepackage{hyperref}
\usepackage{nicefrac}

\usepackage{chngpage}

 \usepackage{tabularx}%

\usepackage{enumitem}
\usepackage{booktabs}
\usepackage{caption, subcaption}
\usepackage{aligned-overset}

\allowdisplaybreaks
\usepackage{mathtools}

\DeclareMathOperator{\trace}{trace}

\usepackage{lastpage}
\jmlrheading{23}{2022}{1-\pageref{LastPage}}{10/21; Revised 6/22}{7/22}{21-1184}{Keru Wu, Scott Schmidler and Yuansi Chen}
\ShortHeadings{title}{all authors LAST names}

\ShortHeadings{Minimax Mixing Time of MALA}{Wu, Schmidler and Chen}
\firstpageno{1}

\begin{document}


\title{Minimax Mixing Time of the Metropolis-Adjusted Langevin Algorithm for Log-Concave Sampling}

\author{\name Keru Wu \email keru.wu@duke.edu \\
       \name Scott Schmidler \email scott.schmidler@duke.edu \\
       \name Yuansi Chen \email yuansi.chen@duke.edu \\
       \addr Department of Statistical Science\\
       Duke University\\
       Durham, North Carolina, USA}

\editor{Anthony Lee}

\maketitle

\begin{abstract}
We study the mixing time of the Metropolis-adjusted Langevin algorithm (MALA) for sampling from a log-smooth and strongly log-concave distribution. We  establish its optimal minimax mixing time under a warm start. Our main contribution is two-fold. First, for a $\dims$-dimensional log-concave density with condition number $\condi$, we show that MALA with a warm start mixes in $\TO(\condi \sqrt{\dims})$ iterations up to logarithmic factors. This improves upon the previous work on the dependency of either the condition number $\condi$ or the dimension $\dims$. Our proof relies on comparing the leapfrog integrator with the continuous Hamiltonian dynamics, where we establish a new concentration bound for the acceptance rate. Second, we prove a spectral gap based mixing time lower bound for reversible MCMC algorithms on general state spaces. We apply this lower bound result to construct a hard distribution for which MALA requires at least $\TOmega(\condi \sqrt{\dims})$ steps to mix. The lower bound for MALA matches our upper bound in terms of condition number and dimension. Finally, numerical experiments are included to validate our theoretical results.
\end{abstract}


\section{Introduction}

Drawing random samples from a distribution is an essential challenge in various fields such as Bayesian statistics, operations research and machine learning \citep{andrieu2003introduction,robert2013monte}. Among all the sampling methods, Markov Chain Monte Carlo (MCMC) algorithms stand out by enabling a wide range of applications~\citep{plummer2003jags,carpenter2017stan}, especially for those involving high-dimensional target distributions. Many Metropolis-Hastings sampling algorithms have been proposed and theoretically studied since the fundamental work of \citet{metropolis1953equation} and the more general results by \citet{hastings1970monte}. Popular MCMC algorithms for sampling from continuous distributions include the Metropolized random walk (MRW) \citep{mengersen1996rates, roberts1996geometric}, the Metropolis-Adjusted Langevin Algorithm (MALA) \citep{ roberts1996exponential,roberts1998optimal,roberts2002langevin} and the Hamiltonian Monte Carlo (HMC) \citep{neal2011mcmc}.

Despite the wide adoption of these MCMC algorithms, establishing the exact mixing time of many algorithms has been challenging even in simple settings. In particular, for the problem of sampling from log-smooth and strongly log-concave distributions (that is, $\dims$-dimensional distributions with a density  $\target(\cdot) \propto e^{-\targetf(\cdot)}$ where $\targetf$ is $\smoothparam$-smooth and $\scparam$-strongly convex), the optimal mixing time results were not well understood for MALA or HMC. Nevertheless, the rationale behind MALA is simple: it constructs a Markov chain which is the Euler discretization of the continuous Langevin dynamics and then applies the Metropolis-Hastings accept-reject step to ensure convergence to the correct stationary distribution. While the continuous Langevin dynamics is well understood (see~\citet{bakry2014analysis}), analyzing the discretization and the Metropolis-Hastings step is far from settled.

The study of mixing time on log-concave distributions is important for practice, because these distributions often appear in statistical modeling and bad sampling outputs result in bad statistical estimates. Typical examples of log-concave distributions in statistics include multivariate Gaussian distributions and Bayesian logistic regression models.
On the other hand, it is a vital question to ask where the theoretical computational gap between sampling from a distribution and optimizing for its maximum lies. To be precise, the setting of the log-smooth and strongly log-concave distribution is closely related to the smooth and convex optimization setting in the convex optimization literature. Observing the resolution of optimal convergence rates for various optimization algorithms with precise upper and lower bounds in the latter literature~\citep{nemirovskij1983problem, nesterov2003introductory}, it is natural to wonder how to accomplish the same for sampling algorithms. In fact, \citet{dalalyan2017theoretical} was intrigued by the same question, which drove him to lay the groundwork on non-asymptotic guarantees for sampling algorithms in high dimensional log-concave settings. Motivated by the need in practice and the lack of theoretical understanding, in this work, we focus on sampling from log-smooth and strongly log-concave distributions and aim to determine the optimal mixing time for MALA.

\subsection{Related work}
The study of the mixing time of MALA and related algorithms started in the 80s~\citep{parisi1981correlation}. And it has recently witnessed a surge on non-asymptotic analyses starting from the work of \citet{dalalyan2017theoretical}.

MALA is the Metropolized version of a simpler algorithm called unadjusted Langevin algorithm (ULA). ULA can be seen as the Euler discretization of the continuous Langevin dynamics without Metropolis-Hastings steps. The study of the mixing time of ULA provides a source of inspiration for understanding the discretization step in MALA \citep{parisi1981correlation, grenander1994representations, dalalyan2017theoretical, durmus2017nonasymptotic, cheng2018convergence, durmus2019high, erdogdu2021convergence}. In particular, under the same log-concave and log-smooth setting in this paper, it was first shown by \citet{dalalyan2017theoretical} that ULA converges in total variation (TV) distance in $\TO(\condi^2\dims\tole^{-2})$ steps. The follow-up work by \citet{durmus2017nonasymptotic} extended this result by studying ULA with decreasing step size in TV distance. Similar results for ULA were proved under Kullback-Leibler (KL) divergence, \mbox{2-Wasserstein} distance \citep{cheng2018convergence, durmus2019high}, and more recently $\chi^2$-divergence \citet{erdogdu2021convergence}. 

Higher order discretization of the continuous Langevin dynamics are also studied in the similar setting, such as underdamped Langevin MCMC \citep{cheng2018sharp,cheng2018underdamped,eberle2019couplings,dalalyan2020sampling,ma2021there}, Hessian-free high-resolution Nesterov acceleration \citep{li2020hessian} and high-order Langevin \citep{mou2021high}. Compared to ULA, the underdamped Langevin MCMC achieves $\tole$ error tolerance with better dimension dependency and error dependency, as shown in \citet{cheng2018underdamped} where they derived a bound of order $\TO(m^{-1/2}\condi^2\dims^{1/2}\tole^{-1})$ in \mbox{2-Wasserstein} distance. This condition number dependency was further improved in \citet{dalalyan2020sampling}. It is worth noting that \citet{ma2021there} interpreted the underdamped Langevin MCMC as a Nesterov acceleration in KL divergence, explaining its faster convergence rate when compared to ULA. However, one limitation of employing unadjusted sampling algorithms is the polynomial dependence of their mixing time on $\tole^{-1}$, which can lead to impractical run times when high quality samples are required. Additionally, the resulting ergodic averages are asymptotically biased, which makes it difficult to choose a stopping time in practice.

Introduced by~\citet{besag1994comments}, the Metropolis-Hastings step in MALA ensures that the Markov chain has the correct stationary distribution. The exponential convergence of Langevin diffusion and its discretization with Metropolis-Hastings was first established by \citet{roberts1996exponential}. This result highlights the convergence difference between unadjusted and Metropolis-Hastings-adjusted sampling algorithms. In a follow-up work, \citet{roberts1998optimal} studied the mixing time of MALA from an asymptotic viewpoint under high order smoothness assumptions. Non-asymptotic mixing time bounds of MALA were later established in \citet{bou2013nonasymptotic} with implicit dimension $\dims$ and error dependency. More recently, in \citet{dwivedi2018log, chen2020fast}, an non-asymptotic mixing time upper bound for MALA was derived in the log-concave sampling setting with a logarithmic dependency on $\tole^{-1}$. Specifically, they show that the MALA mixing time with an appropriate step-size choice is bounded from above by $\bigO\parenth{\max\braces{\condi^{3/2}d^{1/2},\condi \dims } \cdot\log(\tole^{-1})}$ under either a warm start or a Gaussian initialization. A warm start is an initial distribution where maximal ratio between the initial distribution and target distribution is bounded by a constant; see Section~\ref{sec:mc_mixing} for a formal definition.

This bound is further improved in \citet{lee2020logsmooth}. They showed that from a Gaussian start, the MALA mixing time is upper bounded by $\TO(\condi \dims)$ with a log-polynomial dependency on $\tole^{-1}$, using a better log-smooth gradient concentration inequality. \citet{mangoubi2019nonconvex} showed that it is possible to obtained a better dimension dependency of order $\bigO(\dims^{2/3})$ if additional assumptions on higher order smoothness of $\targetf$ are imposed. Later \citet{chewi2021optimal} showed in the log-smooth and strongly log-concave sampling setting that MALA mixes in $\TO(\condi^{3/2} \dims^{1/2})$ steps from a warm start, which provides the state-of-the-art dimension dependency for the MALA mixing time. After observing these upper bounds, it is natural to ask whether it is possible to establish a mixing time upper bound which shares the best condition number and dimension dependency of all.

Other than studying the mixing time upper bound directly, the best way to refute the possibility of a certain upper bound is to establish mixing time lower bounds. For Markov chains in discrete spaces, there are well-established generic techniques for mixing time lower bounds, such as geometric lower bounds, spectral lower bounds and log-Sobolev lower bounds \citep{ borgs2003statistical, wilson2004mixing, montenegro2006eigenvalues, diaconis1996logarithmic}; see Section~5 of \citet{montenegro2006mathematical} and Section~7 of \cite{levin2017markov} for more details. For general state spaces, although some results in discrete spaces can be directly extended to general state spaces, the lower bound literature is relatively more disorganized. For example, lower bounds have been established using diverse proof techniques for MCMC with parallel and simulated tempering \citep{woodard2009sufficient} and adaptive MCMC methods \citep{schmidler2011lower}. As for MALA, mixing time lower bounds are typically obtained from considering hard log-smooth and strongly log-concave distributions. \citet{chewi2021optimal} took an indirect approach to argue the tightness of their mixing time upper bound. Instead of establishing a mixing time lower bound directly, they constructed an adversarial target distribution and showed that its spectral gap upper bound matches their spectral gap lower bound in terms of dimension dependency $\TOmega(\dims^{1/2})$ from a warm start. From their proof, it is not straight-forward to check whether their condition number dependency is also tight. In the case of MALA under an exponentially-warm start, \citet{lee2021lower} constructed a hard initialization and a hard target distribution which resulted in a nearly-tight mixing time lower bound of order $\TOmega(\condi \dims)$. It remains unclear whether taking the smaller exponents on dimension and condition number in both bounds result in the lower bound for MALA from a warm start.

\subsection{Our contribution}

This paper makes two main contributions. First, we show that under a warm start MALA converges in $\TO(\condi \dims^{1/2})$ iterations in the log-smooth and strongly log-concave setting; see \mbox{Table~\ref{tab:1}} for a detailed comparison with previous work.
This bound improves upon previous results obtained by \citet{chewi2021optimal}, $\TO( \condi^{3/2}\dims^{1/2})$, in terms of dependence on $\condi$. Its linear condition number dependency also matches mixing time shown in \citet{lee2020logsmooth} for MALA under a Gaussian start. Consequently, we sharpen both dependencies in the upper bound $\TO(\max(\condi^{3/2} \dims^{1/2}, \condi \dims))$ obtained by \citet{dwivedi2018log,chen2020fast} when the chain has a warm initialization.

\begin{table}[ht]
\centerline{
\begin{tabular}{ c c c}
\specialrule{.1em}{.05em}{.05em}
  & MALA initialization & Mixing Time Upper Bound \\\hline
 \citet{dwivedi2018log}& \multirow{ 2}{*}{warm / $\Normal(x^*,\smoothparam^{-1}\Ind_\dims)$} &\multirow{ 2}{*}{ $\max\{\condi^{\frac32}d^{\frac12},\condi \dims\}$}\\
 \citet{chen2020fast}& &\\\hline
   \citet{lee2020logsmooth} & $\Normal(x^*,\smoothparam^{-1}\Ind_\dims)$  &$\condi\dims$\\\hline
   \citet{chewi2021optimal}& warm & $ \condi^{\frac32}\dims^{\frac12}$\\\hline
  this work & warm  & $\condi \dims^{\frac12}$\\
 \specialrule{.1em}{.05em}{.05em}
  \end{tabular}}
  \caption{Summary of $\tole$-mixing time in TV distance for MALA with a $\smoothparam$-log-smooth and $\scparam$-strongly log-concave target under a warm start or a Gaussian initialization $\Normal(x^*,{\smoothparam}^{-1}\Ind_\dims)$, where $x^*$ denotes the unique mode of the target density. These statements hide logarithmic factors in $\dims,\tole^{-1}$ and $\condi =\smoothparam/\scparam$.}
  \label{tab:1}
\end{table}

Second, we establish an explicit mixing time lower bound in $\chi^2$-divergence for reversible Markov chains in general state spaces, and apply this result to obtain a matching lower bound $\TOmega(\condi\dims^{1/2})$ for MALA under a warm start.  The lower bound proof works for reversible Markov chains in general state spaces and can be of independent interest.

In addition to the two theoretical contributions, we also provide numerical experiments to demonstrate the best choice of step size in terms of condition number dependency and dimension dependency in various settings.

\vspace{0.5em}
\noindent \textbf{Organization:} The remainder of the paper is organized as follows. In Section~\ref{sec:Background}, we provide background on Markov chain Monte Carlo, mixing time analysis, the set-up of log-concave sampling and an introduction to MALA and Hamiltonian dynamics. Section~\ref{sec:Main} is devoted to our main results on the minimax mixing time of MALA. In Section~\ref{sec:Upper}, we sketch the proof of our upper bound by establishing a high probability bound on the acceptance rate of MALA. In Section~\ref{sec:Lower}, we prove a spectral lower bound for general state space Markov chains, and use this result to obtain a matching lower bound for MALA. Section~\ref{sec:Experiments} consists of numerical experiments that we perform to verify the correctness of our theoretical results for a difficult target distribution. Proofs of main theorems and lemmas are in Section~\ref{sec:Proof}; the proof of other technical lemmas are deferred to appendices. Finally, we conclude our results and discuss future directions in Section~\ref{sec:Discussion}.

\vspace{0.5em}
\noindent \textbf{Notations:} We use $x_k$ to denote the $k$-th obtained sample from the Markov chain. We use $\vecnorm{x}{2}$ to denote the Euclidean norm of a $d$-dimensional vector $x$, $\coord{x}{i}$ to denote its $i$-th coordinate, and $\coord{x}{-i}$ to denote the vector without the $i$-th coordinate. The big-O notation $\bigO(\cdot)$ and big-Omega notation $\Omega(\cdot)$ are used to denote asymptotic bounds ignoring constants. For example, we write $g_1(x)=\bigO(g_2(x))$ if there exists a universal constant $\constc>0$ such that $g_1(x)\leq \constc g_2(x)$ when $x$ is large enough. Adding a tilde above these notations such as $\TO(\cdot)$, $ \TOmega(\cdot)$ and $\TTheta(\cdot)$ denotes asymptotic bounds ignoring logarithmic factors for all symbols. We use $\text{poly(x)}$ to denote a polynomial of $x$.

\section{Background and problem set-up}\label{sec:Background}
In this section, we first introduce the background needed for carrying out the mixing time analysis of Markov chains. Then we set up our theoretical framework by defining log-smoothness, strongly log-concavity and warmness in Section~\ref{sec:log_concave}. Finally, we formally introduce the Metropolis-adjusted Langevin algorithm (MALA) and Hamiltonian Monte Carlo (HMC) dynamics in Section~\ref{sec:MALA}.

\subsection{Markov chain and mixing time}\label{sec:mc_mixing}
Consider the problem of sampling from a distribution $\target$ on a general state space $\statespace$. A standard class of sampling methods is to construct an irreducible and aperiodic discrete-time Markov chain with an initial distribution $\initial$ and with $\target$ as its stationary distribution; see for example, the book by \citet{meyn2012markov} for details. To obtain samples from the target distribution within certain error tolerance, one simulates the chain for multiple steps, number of which is determined by the mixing time analysis.

Let the time-homogeneous Markov chain be defined on a general state space $(\statespace, \borel(\statespace))$ associated with a transition kernel $\kernel:\statespace \times \borel(\statespace) \rightarrow \real_{\geq 0}$, where $\borel(\statespace)$ denotes the Borel-sigma algebra on $\statespace$. The $k$-step transition kernel $\kernel^k$ is defined recursively by $\kernel^{k}(x,dy)=\int_{z\in\statespace}\kernel^{k-1}(x,dz)\kernel(z,dy)$. The Markov chain is assumed to be reversible, meaning
\begin{align*}
    \kernel(x,dy)\target(dx)  =  \kernel(y, dx)\target(dy) .
\end{align*}
We define the expectation and the variance of a function $f$ with respect to $\target$ as 
\begin{align*}
    \Exs_\target[f]\defn\int_{x\in\statespace}f(x)\target(dx), \ \Var_\target[f]\defn\Exs_\target[f^2]-\parenth{\Exs_\target[f]}^2.
\end{align*}
In the following, we define several notions related to the mixing time analysis of a Markov chain.

\vspace{1em}
\noindent\textbf{Transition Operators:}  We use $\transition$ to denote the transition operator of the Markov chain as
\begin{align}
  \transition(\mu)(S) = \int_{y\in \statespace} \mu(dy)\kernel(y, S), \ \ \forall S\in\borel(\statespace).
\end{align}
When $\mu$ is the probability distribution of the current state of the Markov chain, $\transition(\mu)$ denotes the distribution at the next state of the chain. We use $\initial$ and $\mu_n$ to denote the initial distribution and the $n$-step distribution of the Markov chain, that is, $\mu_{n}=\transition^n(\mu_0)$.

\vspace{1em}
\noindent\textbf{Mixing Time:}
In this paper we consider two ways to quantify the mixing time. One is based on the total variation (TV) distance, and the other is based on the $\chi^2$-divergence. Let $\mu$ be a probability distribution. Its $\mathcal{L}_p$-divergence with respect to the target distribution $\target$ is defined as
\begin{subequations}
\begin{align}
    \pdist{p}{\mu}{\target}:=\parenth{\int_{x\in\statespace}\abss{\frac{d\mu}{d\target}(x)-1}^p\target(dx)}^{\frac1p}.
\end{align}
 For $p=1$, we get the total variation distance $\tvdist{\mu}{\target}=\pdist{1}{\mu}{\target}/2$. For $p=2$, $\chidist{\mu}{\target}$ corresponds to the $\chi^2$-divergence. Note that the $\chi^2$-divergence can be controlled by the total variation distance if the two distributions have bounded ratio. Specifically, if there exists $\warmparam\geq1$ such that $\mu(S)/\target(S)\leq \warmparam$ for all $S\in\borel(\statespace)$, we have $\chidist{\mu}{\target}^2\leq 2\warmparam\cdot \tvdist{\mu}{\target}$. With the $\mathcal{L}_p$-divergence, the $\mathcal{L}_p$ mixing time of the Markov chain with initial distribution $\initial$ is defined as
\begin{align}
    \pmix{p}{\tole}{ \initial}=\inf\braces{n\in\mathbb{N}\ \big|\ \pdist{p}{\mu_n}{\target}\leq\tole}.
\end{align}
\end{subequations}
The mixing time in total variation distance is denoted by $\tvmix{\tole}{\initial}$ similarly.

\vspace{1em}
\noindent\textbf{Dirichlet form:} The mixing property of Markov chain is closely related to its Dirichlet form introduced as follows. Let $L_2(\target)$ be the space of square integrable functions under the density $\target$. We define the $L_2$-norm on $L_2(\target)$ by
\begin{align}
    \vecnorm{f}{2,\target}^2 = \int_{x\in \mathcal X}f(x)^2 \target(dx),
\end{align}
 The \textit{Dirichlet form} $\dirichlet_\kernel : L_2(\target)\times L_2(\target) \rightarrow \real$ associated with the transition kernel $\kernel$ is given by
\begin{align}
    \dirichlet_\kernel(g,h) = \frac 12 \int_{x,y\in \statespace^2}(g(x)-h(y))^2 \kernel(x,dy)\target(dx).
\end{align}

\vspace{1em}
\noindent\textbf{$\condS$-Conductance: } Since it is difficult to compute the spectral gap for a specific Markov chain, conductance is usually used as an alternative for analysis. For scalar $\condS\in(0, 1/2)$, we define the $\condS$-conductance as
\begin{align}
    \conductance_\condS = \inf_{\targetdistri(S)\in(\condS,1/2]}\frac{\int_S\mathcal T_x(S^c)\target(dx)}{\targetdistri(S)-\condS}.\label{eq:s_conductance}
\end{align}
In this definition, $\transition_x$ is the shorthand for $\transition(\delta_x)$, the transition distribution at $x$, where $\delta_x$ denotes the Dirac distribution at $x$. We have $\transition_x(\cdot)=\kernel(x,\cdot)$ by definition.

\vspace{1em}
\noindent\textbf{Lazy chain:} We say that a Markov chain is $\lazy$-lazy if for each iteration the chain stays in the same state with probability at least $\lazy$. Since laziness only slows down the convergence rate by a constant factor, we study $1/2$-lazy Markov chains in this paper for convenience of theoretical analysis. We use $\transitionbf_x$ to denote the transition distribution before applying the lazy step. By definition, we have
\begin{align}\label{eq:lazy}
    \transition_x(S)=\frac12 \delta_x(S)+\frac12\transitionbf_x(S), \ \ \ \forall S\in\borel(\statespace).
\end{align}

\subsection{Log-concave sampling under a warm start}\label{sec:log_concave}
From now on we assume $\statespace=\reald$ unless otherwise specified. A differentiable function $f$ on $\reald$ is said to be $\smoothparam$-smooth and $\scparam$-strongly convex if
\begin{subequations}
\begin{align}
    &f(y)\leq f(x)+\gradf(x)^\top(y-x)+\frac{\smoothparam}{2}\vecnorm{x-y}{2}^2\ \ \  \text{ for all }x,y\in\real^\dims
\end{align}
and
\begin{align}
    &f(y)\geq f(x)+\gradf(x)^\top(y-x)+\frac{\scparam}{2}\vecnorm{x-y}{2}^2\ \ \  \text{ for all }x,y\in\real^\dims.
\end{align}
\end{subequations}
The condition number $\condi$ of such function $f$ is defined as $\condi:=\smoothparam/\scparam$. For an $\scparam$-strongly convex function, its global minimum is uniquely defined. We denote the global minimum of $f$ by $x^*\defn {\text{argmin}}_{{x\in\real^\dims}}f(x)$. The target distribution $\target$ is said to be $\smoothparam$-log-smooth and $\scparam$-strongly log-concave if it admits a density $\target(x)\propto\exp(-\targetf(x))$, and $\targetf$ is $\smoothparam$-smooth and $\scparam$-strongly convex.

\vspace{0.5em}
\noindent\textbf{Warmness:} We say that the initial distribution $\initial$ is $\warmparam$-warm if it satisfies
\begin{align}
    \sup_{S\in\borel(\reald)}\frac{\initial(S)}{\targetdistri(S)}\leq \warmparam. 
\end{align}
As the warmness parameter $\warmparam$ decreases, the initial distribution is closer to the target and the task of sampling becomes easier. For a sequence of target distributions, we say that a corresponding sequence of $\mu_0$ is constant-warm, if $\log(\warmparam)$ is a polylogarithmic function of the condition number $\condi$ and the dimension $\dims$ of the target distribution. If $\warmparam$ is exponential in $\condi$ or $\dims$, we say that $\mu_0$ is exponentially-warm. In this paper, when we say ``a warm start", we refer to the case of a constant-warm initialization. One practical choice of the initial distribution is the Gaussian initialization $\initial=\Normal(x^*,\smoothparam^{-1}\Ind_\dims)$ \citep{dwivedi2018log, lee2020logsmooth}. It is exponentially-warm, as shown in \citet{dwivedi2018log} that the warmness of this Gaussian start satisfies $\warmparam\leq \condi^{\dims/2}$. In practice, however, a (constant-)warm start is not always available. Despite this limitation, we focus on the case of warm starts in the interest of theoretical understanding of the MALA algorithm, and comparing mixing time under a warm start with an exponentially-warm start in previous papers~\citep{dwivedi2018log, chen2020fast, lee2020logsmooth}.

\subsection{MALA and HMC dynamics}\label{sec:MALA}

Metropolis-adjusted Langevin algorithm (MALA) was original proposed by~\citet{besag1994comments} and its properties were examined in detail by~\citet{roberts1996exponential}. Given a state $x_k$ at the $k$th iteration, MALA proposes a new state $y_{k+1}\sim\Normal(x_k-h\gradf(x_k),2h\Ind_\dims)$. Then it decides to accept or reject $y_{k+1}$ using a Metropolis-Hastings correction; see Algorithm~\ref{alg:mala}. We use $\proposal_x=\Normal(x-h\gradf(x),2h\Ind_\dims)$ to denote the proposal kernel of MALA at $x$.

\begin{algorithm}[ht]
  \SetAlgoLined
  \KwIn{Initial point $x_0$ from a starting distribution $\initial$, step size $h$, number of steps $n$}
  \KwOut{Sequence of samples $x_1,x_2,\ldots, x_n$}
  \For{$k=0,1,\ldots, n-1$}{
   Draw from the proposal distribution $y_{k}\sim\Normal(x_k-h\gradf(x_k),2h\Ind_d)$\;
   Compute the acceptance rate $\alpha_k\leftarrow\min\braces{\dfrac{\exp\parenth{-f(y_{k})-\vecnorm{{}x_{k}-y_{k}+h\gradf(y_{k})}{2}^2/4h}}{\exp\parenth{-f(x_k)-\vecnorm{y_{k}-x_k+h\gradf(x_k)}{2}^2/4h}},1}$\;
   Draw $u\sim \text{Unif}[0,1]$\;
    \eIf{$u<\alpha_k$}{
      Accept the proposal: $x_{k+1}\leftarrow y_{k}$\;
      }{
      Reject the proposal: $x_{k+1}\leftarrow x_k$\;
      }
    }
  \caption{Metropolis-Adjusted Langevin Algorithm (MALA)}
  \label{alg:mala}
\end{algorithm}
The MALA algorithm has a close connection to the Hamiltonian Monte Carlo (HMC) sampling algorithm from the physics literature, popularized in statistics by \citet{neal2011mcmc}. Define the Hamiltonian function $H:\real^d \times \real^d\rightarrow \real$ as
\begin{align}
    H(q, p) = f(q)+\frac12\|p\|_2^2.
\end{align}
The following continuous Hamiltonian dynamics describes the trajectory of $(q_t, p_t)\in\real^d\times\real^d$ for $t \geq 0$. Starting from the initial state $(q_0,p_0)$, the pair $(q_t,p_t)$ satisfies
\begin{subequations}
\begin{align}
    \frac{dq_t}{dt}&=\frac{\partial H}{\partial p}(p_t,q_t)=p_t\\
    \frac{dp_t}{dt}&=-\frac{\partial H}{\partial q}(p_t,q_t)=-\gradf(q_t).
\end{align}
\end{subequations}
Note that the Hamiltonian is preserved throughout the continuous Hamiltonian dynamics $\frac{dH(q_t, p_t)}{dt} = 0$. The solution of the continuous Hamiltonian dynamics satisfies, for $t\in [0,\step]$,
\begin{subequations}
\begin{align}
    q_t&=q_0+tp_0-\int_0^t\int_0^s\gradf(q_\tau)d\tau ds\label{eq:HMC_q}\\
    p_t&= p_0 -\int_0^t \gradf(q_s)ds.\label{eq:HMC_p}
\end{align}
\end{subequations}
In practice, it is difficult to obtain an explicit formula for the solution $(q_t, p_t)$. To discretize the continuous Hamiltonian dynamics,  \citet{neal2011mcmc} proposed to use the leapfrog or St\"{o}rmer-Verlet discretization, known as the Hamitonian Monte Carlo (HMC) algorithm. Starting from $(\cq_0, \cp_0)$, a single leapfrog step satisfies
\begin{subequations}
\begin{align}
    \hp_{\step/2} &= \cp_0 -\frac \step 2\nabla \targetf(\cq_0)\\
    \hq_\step &= \cq_0 + \step \hp_{\step/2}\label{eq:leapfrog_q}\\
    \hp_\step &= \hp_{\step/2}-\frac \step2 \nabla \targetf(\hq_\step).\label{eq:leapfrog_p}
\end{align}
\end{subequations}
Mathematically, starting from an initial point $q_0=x_0$ and letting $p_0\sim\Normal(0,\Ind_d)$, the proposal of MALA is the same as $\hat q_\step$ in Equation~\eqref{eq:leapfrog_q}, where the original step size $h$ in MALA and the step size $\step$ in HMC are related through
\begin{align}
    h=\frac{1}{2}\step^2.
\end{align}
With the above observation, one step of MALA can be regarded as one iteration of HMC with a single leapfrog step. This property allows us to study the acceptance rate of MALA via studying the acceptance rate of HMC with leapfrog discretization.


\section{Main results}\label{sec:Main}

In this section, we state our main results on the minimax mixing time of MALA. Let $\mathcal{A}_{\dims, \smoothparam, \scparam} = \braces{\mu \propto e^{-f} \mid f: \real^\dims \mapsto \real, \text{ twice-differentiable}, \smoothparam\text{-smooth and } \scparam\text{-strongly-concovex}}$ denote the collection of all $\dims$-dimensional $\smoothparam$-log-smooth and $\scparam$-strongly log-concave distributions. Define the \textit{minimax mixing time} of MALA in total variation distance as
\begin{align}
    \minimaxmixingtime(\dims,\smoothparam ,\scparam ,\tole,\warmparam):=\min_{h>0}\max_{\target\in \mathcal{A}_{\dims,\smoothparam ,\scparam }} \max_{\warmparam\text{-warm }\initial} \tvmix{\tole}{\initial}.
\end{align}
This minimax definition considers the best step size that minimizes the worst time mixing among all $\smoothparam$- log-smooth and $\scparam$-strongly log-concave distributions under an $\warmparam$-warm initialization. Our first theorem provides an upper bound of the minimax mixing time.

\begin{theorem}\label{thm:main}
Let $\target$ be a $\dims$-dimensional $\smoothparam$-log-smooth and $\scparam $-strongly log-concave target distribution. Define the condition number $\condi=\smoothparam /\scparam $. There exist universal constants $\constUniversal, \constUniversall, \constUniversalll >0$, such that for any $\warmparam$-warm initial distribution $\initial$ and any error tolerance $\tole\in(0,1)$, the $\tole$-mixing time of $1/2$-lazy MALA with step size
\begin{align}
h=\frac{\constUniversal}{\smoothparam \sqrt{\dims} \log^2 \parenth{\max\braces{ \condi, \dims,\frac{\warmparam}{\tole}, \constUniversalll}}}
\end{align}
is bounded by
\begin{align}
\tvmix{\tole}{\initial}\leq \constUniversall \condi \sqrt{\dims} \log^3 \parenth{\max\braces{\condi, \dims, \frac{\warmparam}{\tole}, \constUniversalll}}.
\end{align}
\end{theorem}

See Section~\ref{proof:main} for the proof. It is a direct corollary of Theorem~\ref{thm:general} which states a mixing time upper bound assuming $\target$ is log-smooth and log-concave but not necessarily strongly log-concave, and satisfies an isoperimetric inequality. The proof mainly relies on bounding the $\condS$-conductance $\conductance_\condS$ from below so that we can control $\tvdist{\mu_n}{\target}$. To do so, we develop a new concentration bound on the acceptance rate of MALA in Lemma~\ref{lem:acceptance rate}. By Theorem~\ref{thm:main}, we obtain an upper bound on the minimax mixing time
\begin{align}
    \minimaxmixingtime(\dims,\smoothparam ,\scparam ,\tole,\warmparam)\leq \constUniversall \condi \sqrt{\dims} \log^3 \parenth{\max\braces{ \condi,\dims,  \frac{\warmparam}{\tole}, \constUniversalll}}.
\end{align}

Theorem \ref{thm:main} provides an explicit choice for the step size $h$, which ensures that approximately $\TO\parenth{\condi \dims^{1/2}}$ iterations is enough to achieve an $\tole$ error tolerance in total variation distance. This result improves previous upper bound in terms of either the condition number dependency \citep{chewi2021optimal} or the dimension dependency \citep{dwivedi2018log}. Note that recently \citet{lee2021structured} developed a reduction framework which can improve the condition number dependency from polynomial to linear for sampling algorithms. This technique can improve the mixing time obtained in \citet{chewi2021optimal} to $\TO(\condi\dims^{1/2})$, however, the resulting algorithm will be different from the original MALA and it is not trivial to implement the algorithm in practice. 

In terms of dependencies on the error tolerance parameter $\tole$ and the warmness parameter $\warmparam$, the bound above may not be tight. \cite{dwivedi2018log, chen2020fast} prove an upper bound with dependency $\log(1/\tole)$ and $\log\log(\warmparam)$ for all initial distributions. Under an exponentially-warm start, additional dimension dependency will be introduced by the $\log(\warmparam)$ term in Theorem~\ref{thm:main}, while it can be ignored under a warm start. We leave future work to achieve better dependencies on these two parameters.

The second theorem provides a matching lower bound on the mixing time.
\begin{theorem}\label{thm:main2}
Under the same assumptions in Theorem~\ref{thm:main}, further assume that $3\leq \condi\leq \alpha\cdot \dims^\beta$ for some $\alpha,\beta>0$. There exists an integer $\integer_{\alpha,\beta}>0$ such that for any $\dims>\integer_{\alpha,\beta}$, $\warmparam\geq 12$ and $\tole\in(0,1)$, we have
\begin{align}
    \minimaxmixingtime(\dims,\smoothparam ,\scparam ,\tole,\warmparam)\geq \frac{\constUniversallll\condi\sqrt{\dims}}{\max\braces{\log\condi, \log \dims}^3}\log\parenth{\frac \constUniversalllll \tole},
\end{align}
where $\constUniversallll,\constUniversalllll>0$ are universal constants.
\end{theorem}

See Section~\ref{proof:main2} for the proof. In the proof we first establish a spectral gap based mixing time lower bound in $\chi^2$-divergence for reversible Markov chains on general state spaces. Then we use it to analyze the mixing time for the worst-case log-concave distributions with specifically designed warm initializations. The worst-case distributions are perturbed Gaussian distributions adapted from \citet{lee2020logsmooth} and \citet{chewi2021optimal}. Note that the warmness parameter $\warmparam$ does not appear in the lower bound, as we implicitly assume a warm start. \citet{chewi2021optimal} bounded the spectral gap of the perturbed Gaussian distribution to suggest the optimal choice of step size. However, a rigorous argument for relating the spectral gap and the mixing time for continuous state Markov chains was missing. Note that our lower bound result is not directly comparable to that in \citet{lee2021lower}, because they assume an exponentially-warm initialization with the warmness growing exponentially in dimension.

The upper bound in Theorem~\ref{thm:main} and the lower bound in Theorem~\ref{thm:main2} match perfectly after ignoring logarithmic factors. That is, under a warm start, assuming that $3\leq\condi = \bigO(\text{poly}(\dims))$ and ignoring all logarithmic factors in $\dims, \condi, \warmparam$ and $\tole$, these two bounds match in terms of dimension and condition number dependency. We thus obtain a minimax mixing time of order $\TTheta\parenth{\condi\dims^{1/2}}$.


\subsection{Mixing time upper bound}\label{sec:Upper}
In this section we outline the key steps to obtain Theorem~\ref{thm:main}. In Section~\ref{sec:sketch}, we state Theorem~\ref{thm:general} which upper bounds the mixing time when $\target$ is log-concave and log-smooth, and satisfies an isoperimetric inequality. Theorem~\ref{thm:main} immediately follows from this theorem. In Section~\ref{sec:accept_rate}, we analyze the acceptance rate of MALA and establish a new high probability concentration bound for it. The new high probability concentration bound is the key to the improved acceptance rate analysis and is also the main reason behind the improved dimension and condition number dependency in Theorem~\ref{thm:main}.

\subsubsection{Mixing time upper bound via isoperimetric inequality}\label{sec:sketch}
Instead of assuming the target distribution $\target$ to be both log-smooth and strongly log-concave, we consider a more general setting where we only assume that $\target$ is log-smooth and log-concave but not necessarily strongly log-concave, and $\target$ satisfies an isoperimetric inequality. A distribution $\target$ in $\real^\dims$ is said to satisfy an isoperimetric inequality with \textit{isoperimetric constant} $\isop(\target)$, if given any partition $S_1,S_2,S_3$ of $\real^\dims$, we have
\begin{align}
    \targetdistri(S_3)\geq\isop(\target)\cdot \dist(S_1,S_2)\cdot \targetdistri(S_1)\targetdistri(S_2),
\end{align}
where the distance between two sets $S_1$, $S_2$ is defined as $\dist(S_1,S_2)=\inf_{x\in S_1,y\in S_2}\vecnorm{x-y}{2}$. In particular, when $\target$ is $\scparam $-strongly log-concave, \citet{cousins2014cubic} (Theorem 4.4) proved that $\target$ satisfies an isoperimetric inequality with isoperimetric constant $\log2\cdot\sqrt{\scparam }$; See also \citet{ledoux2001concentration}. Thus an upper bound of mixing time based on the isoperimetric constant can be applied to strongly log-concave targets.

The following theorem provides a mixing time upper bound for MALA applied to log-concave and $\smoothparam$-log-smooth target distributions satisfying an isoperimetric inequality.

\begin{theorem}\label{thm:general}
Let $\target$ be a $\dims$-dimensional log-concave and $\smoothparam$-log-smooth distribution satisfying an isoperimetric inequality with isoperimetric constant $\isop(\target)$. There exist universal constants $\constUniversal, \constUniversall, \constUniversalll>0$, such that for any $\warmparam$-warm initial distribution $\initial$ and any error tolerance $\tole\in(0,1)$, the $1/2$-lazy MALA with step size
\begin{align}
 h & = \frac{\constUniversal}{\smoothparam \sqrt{\dims} \cdot \log^2\parenth{\max\braces{\dims, \frac{\smoothparam}{\isop(\target)^2},  \frac{\warmparam}{\tole}, \constUniversalll}}}
\end{align}
has mixing time given by
\begin{align}
    \tvmix{\tole}{\initial}\leq\constUniversall\cdot\max\braces{\frac{\smoothparam\sqrt{\dims}}{\isop(\target)^2} \cdot \log^3\parenth{\max\braces{\dims, \frac{\smoothparam}{\isop(\target)^2},  \frac{\warmparam}{\tole}, \constUniversalll}}, \log\parenth{\frac{2\warmparam}{\tole}}} .
\end{align}
\end{theorem}

See Section~\ref{proof:general} for the proof of the above theorem. Theorem~\ref{thm:main} follows directly by noticing that $\scparam$-strongly log-concave distributions satisfy an isoperimetric inequality with constant $\log2\cdot\sqrt{\scparam}$. The proof of Theorem~\ref{thm:main} is provided in Section~\ref{proof:main}.

Here we sketch the proof of Theorem~\ref{thm:general}. We follow the framework of bounding the conductance of Markov chains to analyze mixing times~\citep{sinclair1989approximate,lovasz1993random}. The basic idea is to bound the $\condS$-conductance $\conductance_\condS$, as implied by the following lemma in \citet{lovasz1993random}.

\begin{lemma}\label{lem:s_cond}
Consider a reversible $1/2$-lazy Markov chain with stationary distribution $\target$ and initial distribution $\initial$. Let $0<\condS<1/2$ and $H_\condS \defn \sup\braces{\abss{\mu_0(B) - \target(B)} : \target(B) \leq \condS}$. Then 
\begin{align*}
    \tvdist{\mu_n}{\target}\leq H_\condS + \frac{H_\condS}{\condS}(1-\frac{\conductance_\condS^2}{2})^n,
\end{align*}
where the $\condS$-conductance $\conductance_\condS$ is defined in equation~\eqref{eq:s_conductance}.
\end{lemma}

Instead of bounding the conductance for all $x\in\reald$, the notion of $\condS$-conductance enables us to consider a high probability region $\goodqsetcond$ with good mixing behavior for analysis. Previously in \citet{dwivedi2018log}, $\goodqsetcond$ was defined as a convex set with bounded $\vecnorm{x}{2}$ such that the gradient norm $\vecnorm{\gradf(x)}{2}$ is small. More recently, \citet{lee2020logsmooth} considered $\goodqsetcond$ to be the not-necessarily-convex region where $\vecnorm{\gradf(x)}{2}$ is small and obtained a better $\condi$ dependency via blocking conductance. To avoid issues caused by the non-convexity of $\goodqsetcond$, their conductance argument introduced additional logarithmic dependency on the error tolerance $\tole$ in the final bound. In our work, we still apply the $\condS$-conductance argument. But in our definition of $\goodqsetcond$, we require not only $\vecnorm{\gradf(x)}{2}$ to be bounded, but also the acceptance rate of MALA to be above some positive constant. This region $\goodqsetcond$ allows us to achieve both linear $\condi$ dependency and $\dims^{1/2}$ dimension dependency in the case of a warm start. Since our region $\goodqsetcond$ is not guaranteed to be convex as in~\citet{lee2020logsmooth}, we also introduce additional logarithmic factors on $\tole$ compared to \citet{dwivedi2018log} and~\citet{chen2020fast}.
 
In the next section we show how we control the acceptance rate and develop a high probability concentration bound for it.

\subsubsection{Acceptance rate of MALA}\label{sec:accept_rate}
As introduced in Section \ref{sec:MALA}, viewing MALA as a special case of HMC allows us to write down its acceptance rate as follows
\begin{align}\label{eq:accept_rate}
\min\braces{\exp\parenth{-f(\hat q_\step)-\frac{1}{2}\vecnorm{\hat p_\step}{2}^2 + f(q_0) + \frac12\vecnorm{p_0}{2}^2},1 }.
\end{align}
Then, we use the fact that the continuous HMC dynamics conserves the Hamiltonian, that is,
\begin{align*}
f(q_t)+\frac12\vecnorm{p_t}{2}^2=f(q_0)+\frac12\vecnorm{p_0}{2}^2, \ \ \forall t>0.
\end{align*}
Combined with Equation~\eqref{eq:accept_rate}, the MALA acceptance rate can be equivalently written as
\begin{align}
\min\braces{\exp\parenth{-\targetf(\hat q_\step)-\frac12\vecnorm{\hat p_\step}{2}^2 + \targetf(q_\step) + \frac12\|p_\step\|_2^2},1}.
\end{align}
The key to control the MALA acceptance rate is to control the discretization error when each step of HMC is discretized with the leapfrog scheme. The next lemma provides a control of the the MALA acceptance rate with high probability.

\begin{lemma}
  \label{lem:acceptance rate}
  Assume the negative log density $\targetf$ is $\smoothparam$-smooth and convex. For any $\deltaf\in(0,1)$, there exists a set $\goodqpset \subset \real^\dims \times \real^\dims$ with $\Prob_{\cq_0 \sim \target, \cp_0 \sim \Normal(0, \Ind_\dims)}((\cq_0, \cp_0) \in \goodqpset) \geq 1-\deltaf$, such that for $(\cq_0, \cp_0) \in \goodqpset$ and the step-size choice $\step^2\smoothparam \leq 1$, we have
  \begin{align}\label{eq:accept_high_prob}
  \begin{split}
    &\quad -\targetf(\hq_\step)-\frac12\vecnorm{\hp_\step}{2}^2 + \targetf(\cq_\step) + \frac12\vecnorm{\cp_\step}{2}^2 \\
    &\geq -100 \parenth{4 + \log\parenth{\frac{2\dims}{\deltaf}}}^2 \step^2 \smoothparam \dims^{\frac12}   - 8 \parenth{\sqrt{\dims}+\log\parenth{\frac{12}{\deltaf}}}^2\step^4 \smoothparam^2.
  \end{split}
  \end{align}
  where $(\hat q_\step,\hat p_\step)$ are the states of HMC after one step of leapfrog discretization starting from $(q_0, p_0)$ in Equation~\eqref{eq:leapfrog_q} and~\eqref{eq:leapfrog_p}, and $(q_\step, p_\step)$ are solutions at time $\step$ of the continuous Hamiltonian dynamics starting from $(q_0, p_0)$ in Equation~\eqref{eq:HMC_q}, \eqref{eq:HMC_p}.
\end{lemma}
See Section~\ref{proof:acceptance rate} for the proof of this lemma.

Lemma~\ref{lem:acceptance rate} establishes a high probability bound on the exponent in the acceptance rate. If we ignore constants and logarithmic factors for now, taking $\step^2$ roughly $1/(\smoothparam\dims^{1/2})$ suffices to keep the acceptance rate above some positive constant for any $(\cq_0,\cp_0) \in \goodqpset$. In our proof of Theorem~\ref{thm:general}, we construct our good region $\goodqsetcond$ based on set $\goodqpset$ introduced in this lemma.


\subsection{Mixing time lower bound}\label{sec:Lower}

In this section we outline the key steps to obtain the mixing time lower bound in Theorem~\ref{thm:main2}.
In Section~\ref{sec:spectral_lower_bound}, we lower bound the mixing time in $\chi^2$-divergence for reversible Markov Chains on general state spaces. We show that obtaining this lower bound can be reduced to controlling the spectral gap. In Section~\ref{sec:worst_case_example}, we show that the spectral gap of MALA can be upper bounded by considering a special perturbed Gaussian distribution. Theorem~\ref{thm:main2} then follows from these two parts.

\subsubsection{Spectral lower bound}\label{sec:spectral_lower_bound}

In this section, the state space $\statespace$ can be any general state space. For reversible Markov chains on general state spaces, we establish the following lower bound on the $\chi^2$-divergence via Dirichlet form and spectral gap. Similar results appeared in previous work, see Lemma 3.1 in \cite{goel2006mixing}, Proposition 4.2 in \cite{coulhon2001geometric}, Proposition 4.3 in \cite{coulhon1997diagonal}, and remark on Page 524 in \cite{coulhon1996ultracontractivity} for details. However, most of these results make use of eigenvalues of the Markov operator, and do not directly extend to infinite-dimensional operators. Consequently, none of these results are directly applicable in our setting; instead we prove a new lower bound on mixing time which combines ideas from these previous results.

\begin{theorem}\label{thm:var} Let $\kernel$ be the kernel of a reversible Markov chain with invariant distribution $\target$. For any initial distribution $\initial\ll\target$ satisfying $\chidist{\mu_0}{\target}<\infty$, let $h_0=d\mu_0/d\target$, then we have
\begin{align}
    \chidist{\mu_n}{\target}^2 \geq \chidist{\mu_0}{\target}^2 \cdot \parenth{1-\frac{\dirichlet_{\kernel ^2}(h_0,h_0)}{\chidist{\mu_0}{\target}^2}}^n.
\end{align}
\end{theorem}
See Section \ref{proof:var} for the proof of this theorem. The proof is inspired by \cite{coulhon1997diagonal} on lower bounds for heat kernels and Markov chains.

 This lower bound has an exponential rate of convergence, and its rate depends on the Dirichlet form of the two-step transition kernel $\kernel ^2$. Note that $\chidist{\mu_0}{\target}^2=\Var_\target[h_0]$, so ${\dirichlet_{\kernel ^2}(h_0,h_0)/\chidist{\mu_0}{\target}^2}$ is the spectral gap of the function $h_0$ under the two-step transition kernel $\kernel^2$. The two-step transition kernel $\kernel^2$ can be related to the transition kernel $\kernel $ via $\dirichlet_{\kernel ^2}(f,f)\leq 2\dirichlet_\kernel (f,f)$ (see Lemma~\ref{lem:dirichlet}). This observation allows us to obtain the following corollary on mixing time lower bound in $\chi^2$-divergence.

\begin{corollary}\label{corollary:gap}Let $\kernel$ be the kernel of a reversible Markov chain with invariant distribution $\target$. For any $\tole>0$ and any initial distribution $\initial\ll\target$ satisfying $\chidist{\mu_0}{\target}<\infty$, let $h_0=d\initial/d\target$, if the spectral gap ${\dirichlet_{\kernel }(h_0,h_0)}/{\chidist{\initial}{\target}^2}\leq 1/2 $, then its mixing time in $\chi^2$-divergence has a lower bound
\begin{align}
    \chimix{\tole}{\initial}\geq{2}\parenth{-\log\parenth{1-\frac{2\dirichlet_{\kernel }(h_0,h_0)}{\chidist{\mu_0}{\target}^2}}}^{-1}\log\frac{\chidist{\initial}{\target}}{\tole}.
\end{align}
Consequently, if the spectral gap ${\dirichlet_{\kernel }(h_0,h_0)}/{\chidist{\initial}{\target}^2}\leq 1/4 $, the mixing time satisfies
\begin{align}
    \chimix{\tole}{\initial}\geq\frac{1}{2} \parenth{\frac{\dirichlet_{\kernel }(h_0,h_0)}{\chidist{\initial}{\target}^2}}^{-1} \log\frac{\chidist{\initial}{\target}}{\tole}.
\end{align}
\end{corollary}
See Section~\ref{proof:gap} for the proof. 

According Corollary~\ref{corollary:gap}, the mixing time lower bound depends on
the spectral gap of a function $h_0$ and the logarithmic ratio between the initial $\chi^2$-divergence and error tolerance. Given an initial distribution, establishing mixing time lower bounds can be reduced to finding a function $h_0$ such that the spectral gap is upper bounded.

\subsubsection{A worst-case example}\label{sec:worst_case_example}
Following Corollary~\ref{corollary:gap}, here we construct a worst-case example where the spectral gap can be upper bounded with tight dependency on both the condition number $\condi$ and the dimension $\dims$. Note that \citet{lee2020logsmooth} showed that Gaussian distribution has the tight $\condi$ dependency, and \citet{chewi2021optimal} showed that a perturbed Gaussian distribution has the tight dimension dependency. We combine these two ideas and introduce the following worst-case target distribution. Let $x=(\coord{x}{1}, \ldots,\coord{x}{d+1})\in\real^{d+1}$. For $\deltas\in(0,1/4)$, define
\begin{align}
    f_\deltas(x)=\frac{\smoothparam }{2}\sum_{i=1}^{d}\coord{x}{i}^2-\frac{1}{2d^{\frac12-2\deltas}}\sum_{i=1}^{d}\cos\parenth{d^{\frac14-\deltas}\smoothparam ^{\frac12}\coord{x}{i}}+\frac \scparam 2 \coord{x}{d+1}^2\label{eq:worst},
\end{align}
where we assume $\smoothparam\geq 2\scparam$. It is not hard to show that $f_\deltas(x)$ is $3\smoothparam /2$-smooth and $\scparam $-strongly convex by calculating the Hessian of $f_\deltas(x)$. The Hessian is diagonal with the first $d$ main diagonal elements being $\smoothparam(1+\cos(\dims^{1/4-\deltas}\smoothparam^{1/2}\coord{x}{i})/2)\in[\smoothparam/2, 3\smoothparam/2]$ and $m$ being the last element.

This worst-case distribution is adapted from a Gaussian distribution following two steps. First, we create $\dims$ dimensions with variance $1/\smoothparam$ and another dimension with variance $1/\scparam$. This step is intended to make the condition number roughly $\condi$. Second, cosine terms are added to the first $\dims$ dimensions such that the second order derivative switches between $\smoothparam/2$ and $3\smoothparam/2$ frequently, and the third derivative can go to infinity as $\dims$ grows. As we will see later, such perturbation makes the sampling process challenging for MALA, and as a result the best step size is of order $O(\dims^{-1/2})$.

\begin{lemma}\label{lem:worst}Consider the target distribution $\target(x)\propto \exp(-f_\deltas(x))$ where $\deltas\in(0,1/4)$ may vary with $\dims$. Fix the warmness $\warmparam =12$.
    \begin{enumerate}[label=(\alph*)]
        \item There exists an $\warmparam $-warm initial distribution $\initial$ with $\chidist{\initial}{\target}>1/2$,
        such that for any $h\in\parenth{0,1/m}$, the spectral gap of $h_0=d\initial/d\target$ under the MALA transition kernel $\kernel$ with step size $h$ satisfies\label{state:a}
        \begin{align}
            \frac{\dirichlet_{\kernel}(h_0,h_0)}{\chidist{\initial}{\target}^2}&\leq18\scparam h.
        \end{align}
        \item Further assume that $\deltas\in(0,1/{20})$ and $\dims^\deltas\geq\max\braces{\log\dims/2+6,10}$. There exists an $\warmparam $-warm initial distribution $\initial$ with $\chidist{\initial}{\target}>1/2$, such that for any $h\in\parenth{1/(\smoothparam d^{1/2-3\deltas}), \infty}$, the spectral gap of $h_0=d\initial/d\target$ under  the MALA transition kernel $\kernel$ with step size $h$ satisfies\label{state:b}
        \begin{align}
            \frac{\dirichlet_{\kernel}(h_0,h_0)}{\chidist{\mu_0}{\target}^2}&\leq48\exp(-\frac{\dims^{4\deltas}}{16384}).
        \end{align}
    \end{enumerate}
\end{lemma}
See Section \ref{proof:worst} for the proof, part of the which is adapted from \citet{chewi2021optimal}.

Lemma~\ref{lem:worst} bounds the spectral gap of MALA. Compared to results in \citet{chewi2021optimal}, this lemma further includes the smoothness parameter $\smoothparam$ and the strong convexity parameter $\scparam$. Also, Lemma~\ref{lem:worst}\ref{state:b} holds for a much larger range of step-size choices, whereas \citet{chewi2021optimal} requires $h\leq \dims^{-1/3}$. To turn these spectral gap upper bounds into rigorous mixing lower bounds, we combine this lemma with Corollary~\ref{corollary:gap} and obtain Theorem~\ref{thm:main2}.

\section{Numerical experiments}\label{sec:Experiments}

In this section, we apply MALA in various simulation settings to validate our theoretical results for the worst-case example in Section~\ref{sec:worst_case_example}. We consider $\target(x) \propto \exp(-f_\deltas(x))$ as the target distribution, where the negative log density $f_\deltas$ is defined in Equation~\eqref{eq:worst}. We aim to find the best step size for this special distribution. For simplicity we replace $\dims$ with $\dims-1$ in $f_\deltas$ so that the dimension of the state space is exactly $\dims$. In the following, we fix $\deltas=1/40$ which satisfies the conditions in Lemma~\ref{lem:worst}. For this difficult target distribution, we show that the best step size should have $\dims^{-1/2}$ dimension dependency under a warm start in Section~\ref{sec:dimension}. In terms of the condition number dependency, we illustrate that the best step size should be $\smoothparam^{-1}\dims^{-1/2}$ under a warm start in Section~\ref{sec:condition}.

\subsection{Dimension dependency}\label{sec:dimension}
We fix the smoothness parameter $\smoothparam =1$ and strong convexity parameter $\scparam =1$, and vary the dimension $d$ in this section. To measure the convergence of the chain, we consider two metrics. The first is the accept-reject rate of the chain, and the second is a proxy for the mixing time. It is defined as 
\begin{align*}
    \hat \tau =\min_{n\geq 1}\braces{n:\mid\hat q_{n,0.9}-q_{0.9}|\leq 0.05 },
\end{align*}
where $\hat q_{n,0.9}$ is the 90\% quantile of the last dimension of the first $n$ samples from MALA, and $q_{0.9}$ is the actual 90\% quantile of the last dimension of the target distribution. If the accept-reject rate is close to zero, then the chain needs a large number of iterations to mix. If the error in 90\% quantile in the last dimension is large, then the chain has not mixed yet.

\begin{figure}[ht]
        \centering
        \begin{subfigure}[b]{0.45\textwidth}
            \centering
            \includegraphics[width=\textwidth]{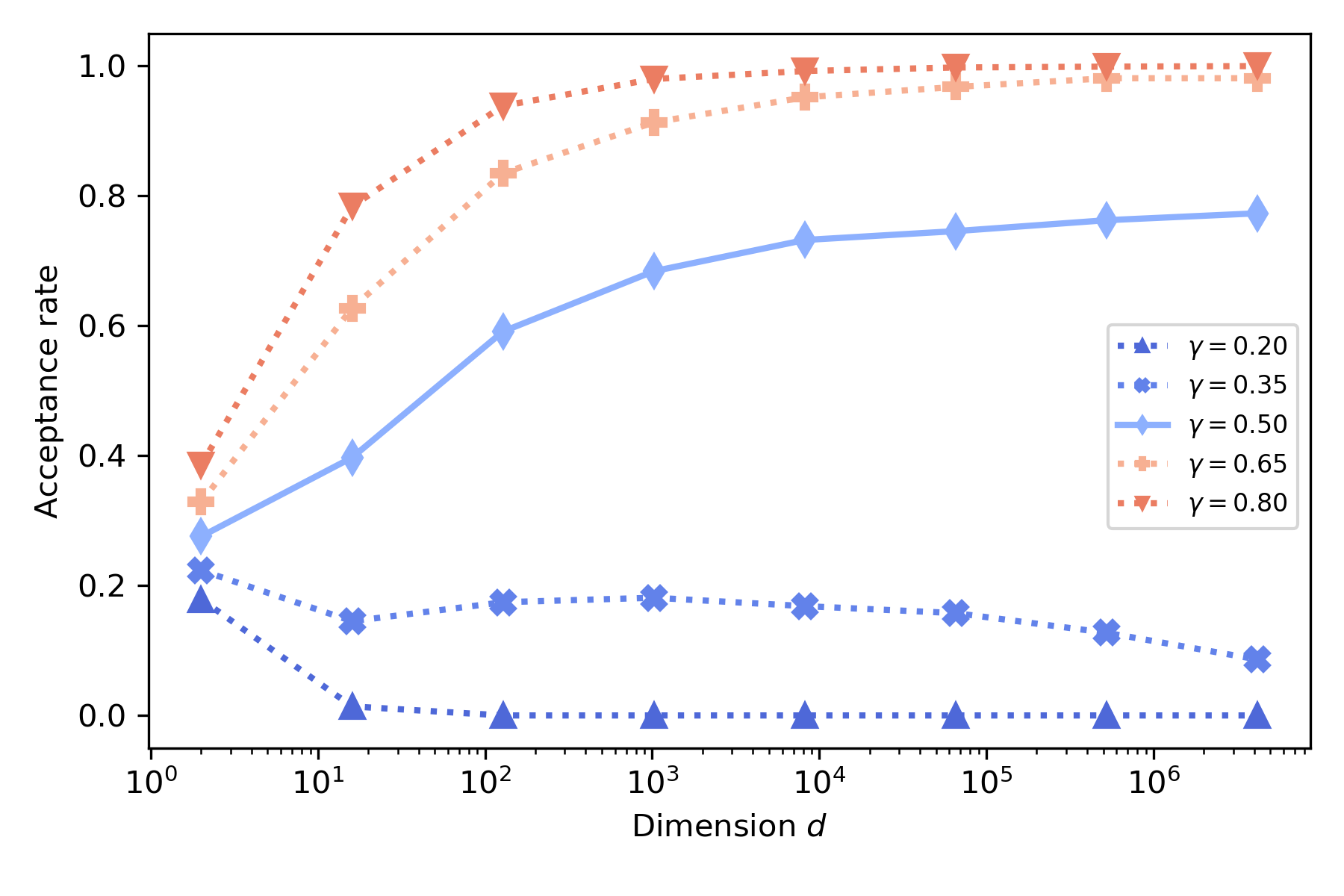}
            \caption[]%
            {{\small Acceptance rate (warm start)}}
        \end{subfigure}
        \hspace{2em}
        \begin{subfigure}[b]{0.45\textwidth}
            \centering
            \includegraphics[width=\textwidth]{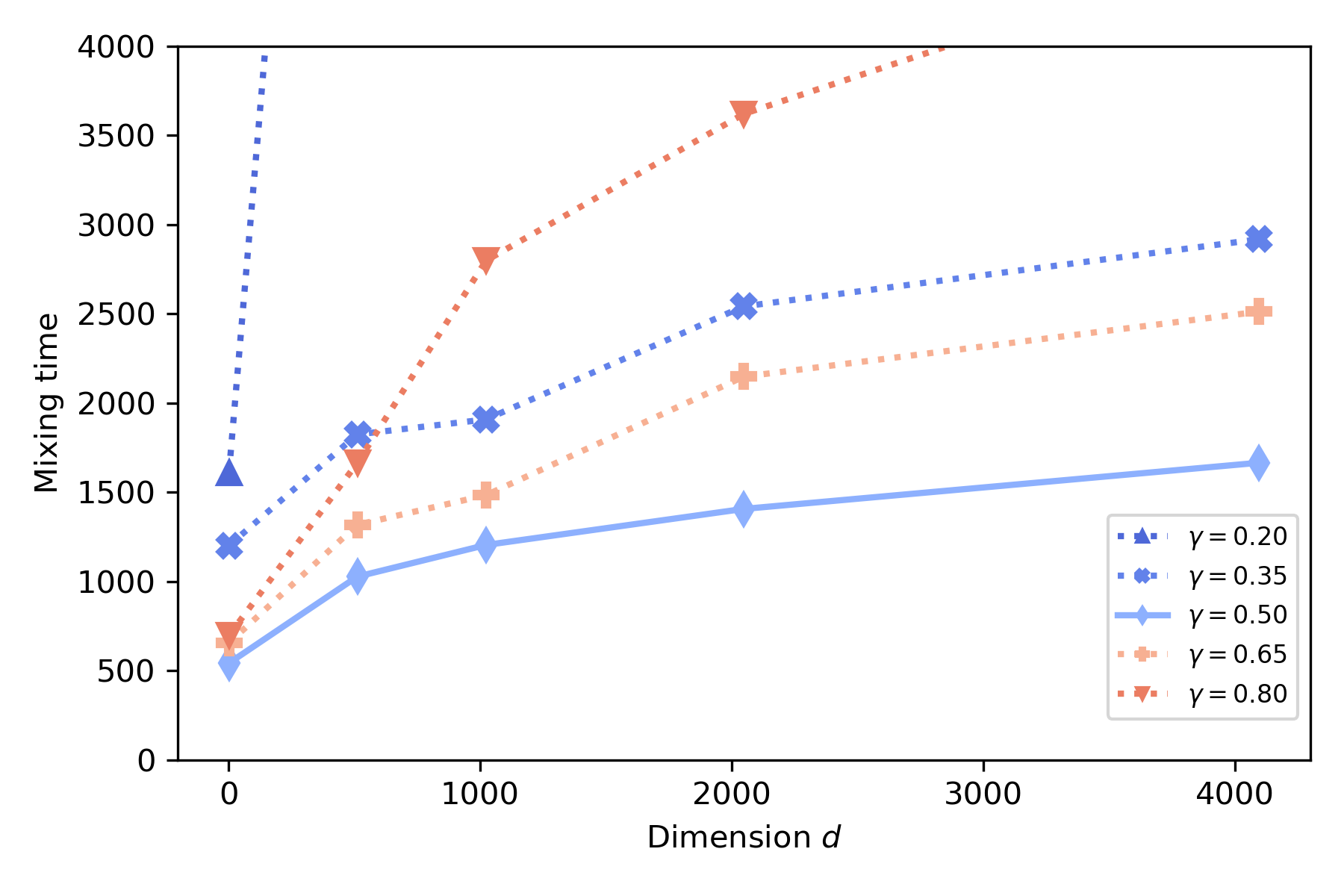}
            \caption[]%
            {{\small Mixing time (warm start)}}
        \end{subfigure}
        \vskip\baselineskip
        \begin{subfigure}[b]{0.45\textwidth}
            \centering
            \includegraphics[width=\textwidth]{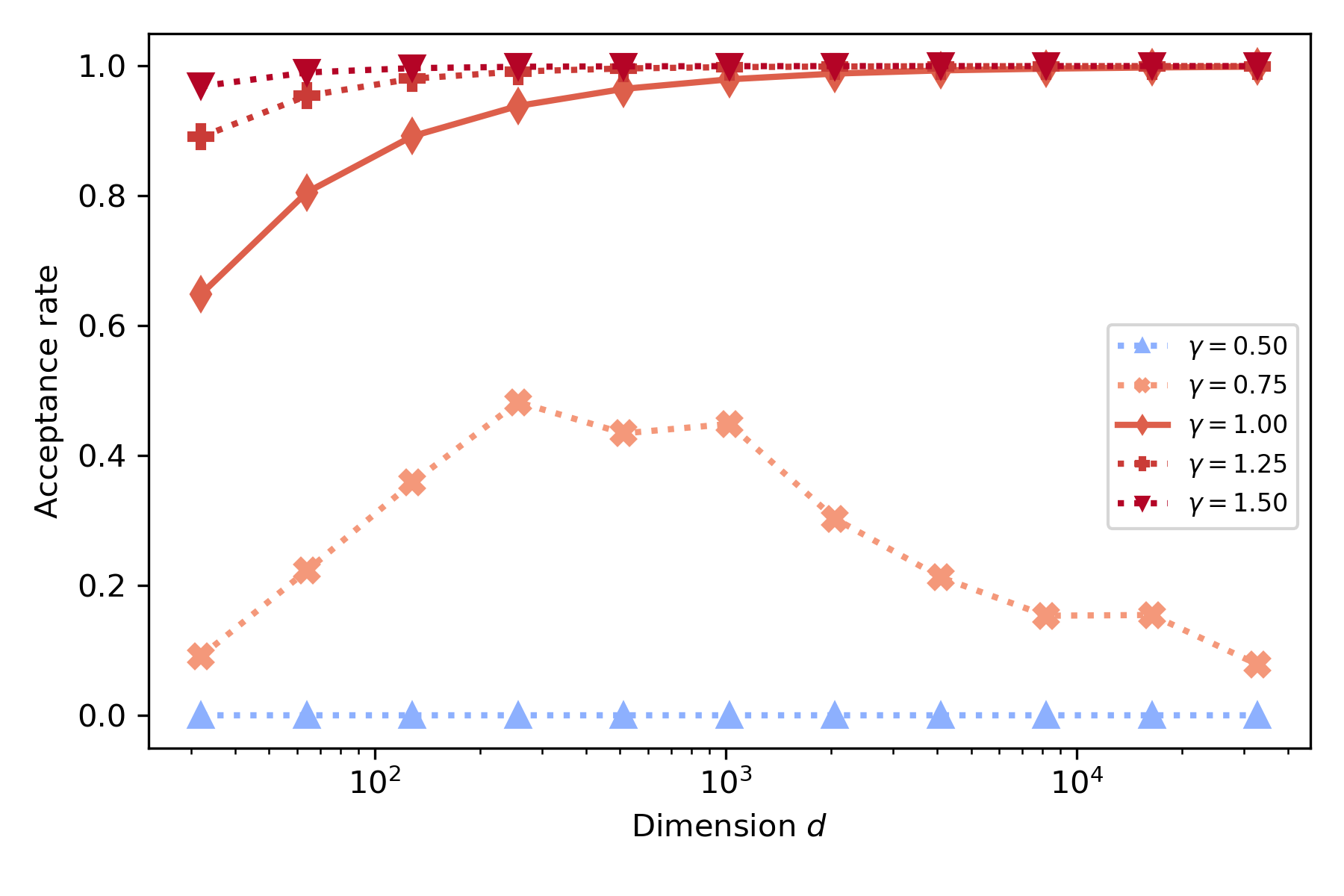}
            \caption[]%
            {{\small Acceptance rate (Gaussian start)}}
        \end{subfigure}
        \hspace{2em}
        \begin{subfigure}[b]{0.45\textwidth}
            \centering
            \includegraphics[width=\textwidth]{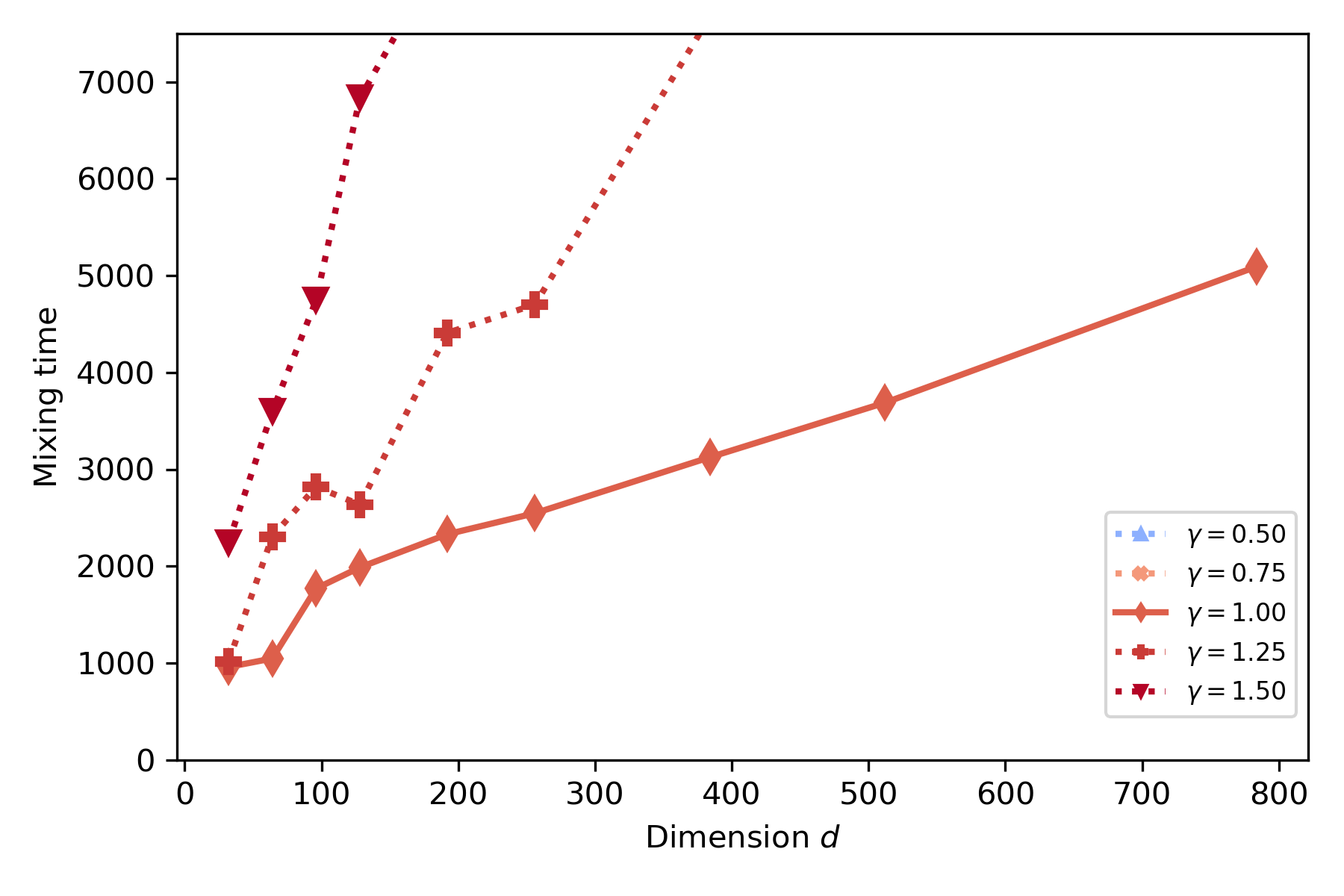}
            \caption[]%
            {{\small Mixing time (Gaussian start)}}
        \end{subfigure}
        \caption[  ]
        {\small Acceptance rate and mixing time of MALA with $\target(x)\propto f_\deltas(x)$  using step size $h=\dims^{-\gamma}$. Panels (a) and (b) show that under a warm initialization, step size $d^{-0.5}$ has a non-vanishing acceptance rate and the best mixing time. Panels (c) and (d) show that under an exponentially-warm initialization,  $\dims^{-1}$ is the best step size. In (d), mixing times for $\gamma=0.5$ and $\gamma=0.75$ are not shown in the plot because their corresponding mixing times are too large and out of range.}
        \label{fig:perturbed_dimension}
\end{figure}

To illustrate how different initializations leads to different mixing times and choice of the best step size, we experiment with both a (constant-)warm start and an exponentially-warm start. The warm initialization is obtained by constraining the target distribution on the set $G=\{x:\sqrt{\smoothparam }\vecnorm{\coord{x}{-d}}{2}\leq \sqrt{d-1}, \ \sqrt{\scparam}|\coord{x}{d}|\leq1\}$, which is a simplified approximation of the warm start we used in the proof of Lemma~\ref{lem:worst} in Section~\ref{proof:worst}. In order to generate initial samples from this distribution, we take advantage of the fact that its density can be written as a product. Hence we can simulate each dimension of the target distribution separately using MALA with sufficient number of steps until convergence, and take samples that are in set $G$. The exponentially-warm start we use is $\Normal(0,\Ind_\dims/1000)$, a Gaussian distribution with most generated samples close to 0. It is not hard to show that the warmness of this initialization has exponential dependency in $\dims$; see \cite{dwivedi2018log} for details.

For the warm start, we consider step size $h=d^{-\gamma}$ for $\gamma\in\braces{0.2,0.35,0.5,0.65,0.8}$. For the Gaussian start, we use $h=d^{-\gamma}$ for $\gamma\in\braces{0.5,0.75,1.0,1.25,1.5}$. In each experiment, we simulate 200 chains and calculate the average of each metric.

Figure~\ref{fig:perturbed_dimension} shows the results of these experiments. From panels (a) and (b) we see that under a warm start, $\dims^{-0.5}$ is the best possible choice of step size. When the step size is too large $h=d^{-0.2}$ or $d^{-0.35}$, the acceptance rate tends to zero when $d$ is large. Additionally, the mixing time in these two cases are larger than that in the case of $h=d^{-0.5}$. We remark that for $\gamma=0.35$ the decrease in acceptance rate is very slow. It might require experiments with even larger $\dims$ to see its acceptance rate to become very close to 0. Such a large $\dims$ may rarely appear in practice, and using a slightly larger step size such as $\gamma=1/3$ in \cite{roberts1998optimal} may not hurt the performance of MALA significantly. When the step size is less than $d^{-0.5}$, the acceptance rate is always kept above some positive constant. Additionally, the mixing time in these cases have a non-exponential growth and are still larger than that in the case of $h=d^{-0.5}$. These results match our main results in Section~\ref{sec:Main} as well as the theoretical results in \citet{chewi2021optimal}. On the other hand, if the chain is initialized under a Gaussian start with an exponential warmness, panels (c) and (d) show that the step size $\dims^{-1}$ becomes the best possible choice of step size. The observation that step-size choice $\dims^{-0.5}$ has close-to-zero acceptance rate under this exponentially-warm start agrees with the mixing time lower bound established in \citet{lee2021lower}.

\subsection{Condition number dependency}\label{sec:condition}
We fix the dimension $\dims$ to be $32$ and vary the condition number by changing $\smoothparam $ while retaining $\scparam=1$. The experimental setup is the similar to that of the previous section, except that here we only consider a warm start for simplicity.

\begin{figure}[ht]
        \centering
        \begin{subfigure}[b]{0.45\textwidth}
            \centering
            \includegraphics[width=\textwidth]{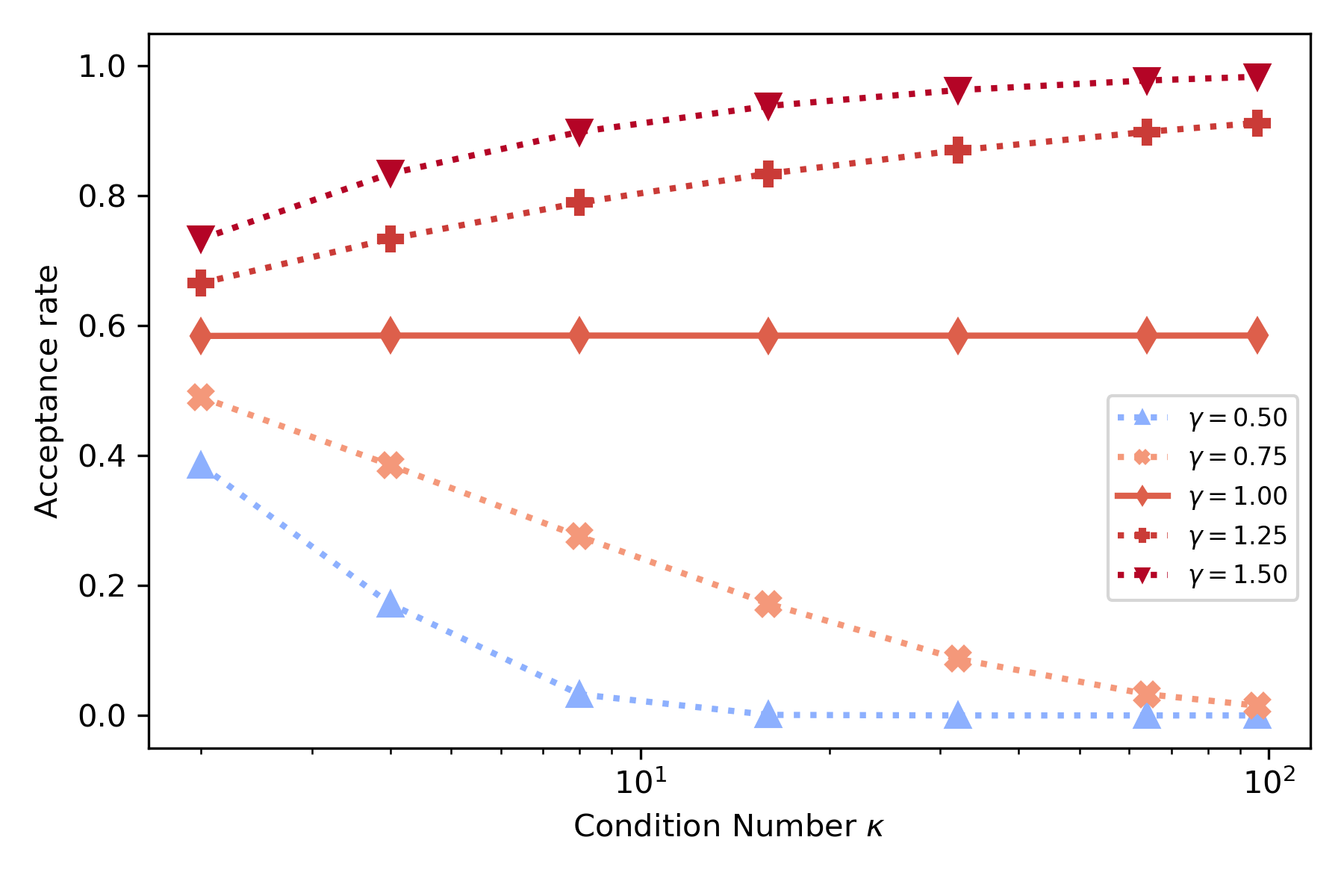}
            \caption[]%
            {{\small Acceptance rate (warm start)}}
        \end{subfigure}
        \hspace{2em}
        \begin{subfigure}[b]{0.45\textwidth}
            \centering
            \includegraphics[width=\textwidth]{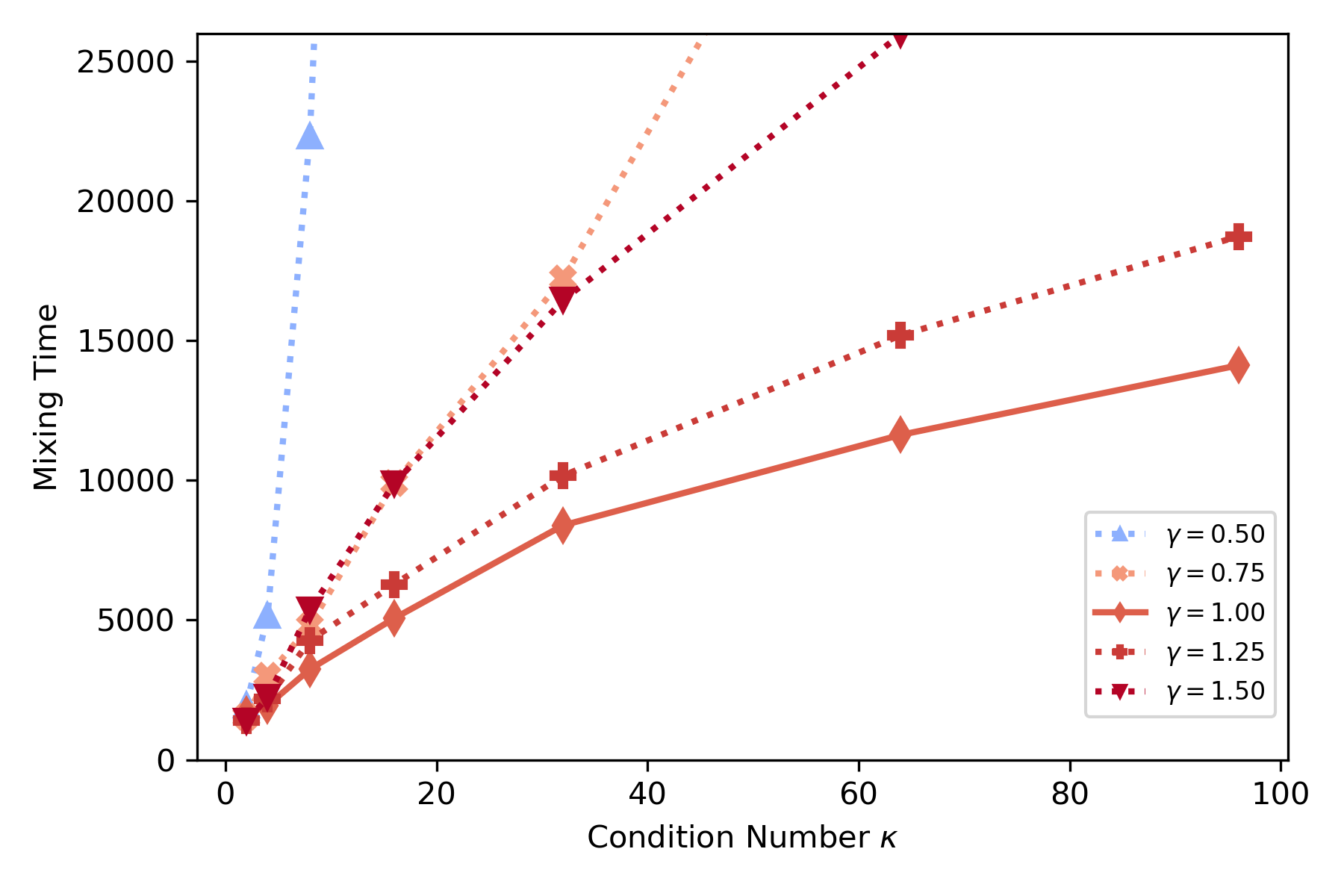}
            \caption[]%
            {{\small Mixing time (warm start)}}
        \end{subfigure}
        
        \caption[  ]
        {\small Acceptance rate and mixing time of MALA with $\target(x)\propto f_\deltas(x)$  using step size $h=\condi^{-\gamma}\dims^{-\frac12}$ under a warm start. Step size with a linear condition number dependency has a non-vanishing acceptance rate and the best mixing time.}
        \label{fig:perturbed_condition}
\end{figure}

The chain is simulated for step sizes with the same dependency on the dimension but varying dependency on the condition number: $h=\smoothparam^{-1}\condi^{1-\gamma}\dims^{-1/2}=\condi^{-\gamma}\dims^{-1/2}$ for $\gamma\in\{0.5, 0.75, 1.0, 1.25,1.5\}$. We observe in Figure \ref{fig:perturbed_condition} that the step size with condition number dependency $\condi^{-0.5}$ or $\condi^{-0.75}$ leads to vanishing acceptance rates and large mixing times. On the other hand, using a step size less than or equal to $\condi^{-1}\dims^{-1/2}$ keeps the acceptance rate above some positive constant and has a much shorter mixing time. As predicted by our theoretical results in Section~\ref{sec:Main}, the step size with a linear $\condi$ dependency turns out to be the best in terms of mixing time. Combing these results with the previous experiments on dimension dependency in Section~\ref{sec:dimension}, we conclude that the best choice of step size should be roughly $\smoothparam^{-1}\dims^{-1/2}$, which matches the step size in Theorem~\ref{thm:main}.


\section{Proofs}\label{sec:Proof}
In this section we prove main theorems. Section~\ref{proof:acceptance rate}, \ref{proof:general} and \ref{proof:main} are devoted to the mixing time upper bound.  Proof of lower bound related results are in Section~\ref{proof:var}, \ref{proof:gap}, \ref{proof:worst} and \ref{proof:main2}. 

\subsection{Proof of Lemma \ref{lem:acceptance rate}}\label{proof:acceptance rate}

Given $(\cq_0, \cp_0) \in \real^\dims \times \real^\dims$, define the following quantity on the squared gradient norm interpolated for $t \geq 0$

\begin{align}
  \label{eq:Gt}
  \interG_t(\cq_0,\cp_0)\defn\nabla \targetf(\cq_0+t \cp_0)^\top\nabla \targetf(\cq_0+t\cp_0).
\end{align}
For simplicity, if the dependence on $(\cq_0, \cp_0)$ is clear from the context, we use $\interG_t$ as a shorthand for $\interG_t(\cq_0,\cp_0)$.
In the following, we separate the exponent in Lemma~\ref{lem:acceptance rate} into two parts and bound the two parts with two lemmas.
\begin{lemma}
  \label{lem:accept_rate_kinetic_diff}
  Assume the negative log density $\targetf$ is $\smoothparam$-smooth. Given $(\cq_0, \cp_0) \in \real^\dims \times \real^\dims$, for the step-size choice satisfying $\step^2 \smoothparam \leq 1$, we have
  \begin{align*}
    \targetf(\hq_\step) - \targetf(\cq_\step) \leq \frac{1}{2}\int_0^\step\int_0^s (\interG_\tau - \interG_0) d\tau ds + \frac{3}{4}\step^4\smoothparam^2 \parenth{\vecnorm{\cp_0}{2} + \frac{1}{\sqrt{\smoothparam}}\vecnorm{\gradf(\cq_0)}{2} }^2
  \end{align*}
  where $\cq_\step$ and $\hq_\step$ are defined in the same way as in Lemma~\ref{lem:acceptance rate}.
\end{lemma}

\begin{lemma}
  \label{lem:accept_rate_potential_diff}
  Assume the negative log density $\targetf$ is $\smoothparam$-smooth and convex. For any $\deltaf\in(0,1)$, there exists a set $\goodqpset \subset \real^\dims \times \real^\dims$ with $\Prob_{\cq_0 \sim \target, \cp_0 \sim \Normal(0, \Ind_\dims)}((\cq_0, \cp_0) \in \goodqpset) \geq 1-\deltaf$, such that for $(\cq_0, \cp_0) \in \goodqpset$ and the step-size choice $\step^2\smoothparam \leq 1$, we have
\begin{align*}
  \quad \frac{1}{2} \vecnorm{\hp_\step}{2}^2 - \frac{1}{2} \vecnorm{\cp_\step}{2}^2 &\leq \frac{1}{2} \int_0^\step (s-\step) \parenth{\interG_s - \interG_0} ds +100 \parenth{4 + \log\parenth{\frac{2\dims}{\deltaf}}}^2 \step^2 \smoothparam \dims^{\frac12}\\
  &\quad   + \frac{5}{4} \step^4 \smoothparam^2 \parenth{\normp + \frac{1}{\sqrt{\smoothparam}} \normgradq}^2,\\
    \vecnorm{\cp_0}{2} &\leq \sqrt{\dims} + \log\parenth{\frac{12}{\deltaf}  },\\
    \vecnorm{\gradf(\cq_0)}{2} &\leq \sqrt{\smoothparam} \parenth{\sqrt{\dims} + \log\parenth{\frac{12}{\deltaf}  }}.
\end{align*}
  where $\cp_\step$ and $\hp_\step$ are defined in the same way as in Lemma~\ref{lem:acceptance rate}.
\end{lemma}

See Appendix~\ref{sec:ABCD} for the proofs of Lemma~\ref{lem:accept_rate_kinetic_diff} and~\ref{lem:accept_rate_potential_diff}.

Note that using integration by parts, we have
\begin{align*}
    \int_0^\step s(G_s-G_0)ds&=\step \int_0^\step (G_\tau -G_0)d\tau - \int_0^\step \int_0^s (G_\tau - G_0)d\tau ds.
\end{align*}
Thus the first terms on the right hand sides of Lemma~\ref{lem:accept_rate_kinetic_diff} and~\ref{lem:accept_rate_potential_diff} get cancelled when we sum them up. Plugging bounds of $\vecnorm{\cp_0}{2}$ and $\vecnorm{\gradf(\cq_0)}{2}$ from Lemma~\ref{lem:accept_rate_potential_diff} into Lemma~\ref{lem:accept_rate_kinetic_diff} and~\ref{lem:accept_rate_potential_diff}, we conclude Lemma~\ref{lem:acceptance rate}.

\subsection{Proof of Theorem \ref{thm:general}}\label{proof:general}

Since $\initial$ is $\warmparam$-warm, we have $H_\condS=\sup\braces{\abss{\mu_0(B) - \target(B)} : \target(B) \leq \condS}\leq \warmparam\condS$. Applying Lemma~\ref{lem:s_cond}, we get 
\begin{align*}
    \tvdist{\mu_n}{\target}\leq \warmparam \condS + \warmparam e^{-\frac n 2\conductance_\condS^2}.
\end{align*}
Then $\tvdist{\mu_n}{\target}\leq \tole$ follows from taking
\begin{align}
    \label{eq:s_cond}
    \condS=\frac{\tole}{2\warmparam }, \ n\geq \frac{2}{\conductance_\condS^2}\log \frac{2\warmparam }{\tole}.
\end{align}
The rest of the proof is dedicated to controlling the $\condS$-conductance $\conductance_\condS$. Let $\goodqpset$ be the set introduced in Lemma~\ref{lem:acceptance rate} satisfying Equation~\eqref{eq:accept_high_prob} for any $(\cq_0, \cp_0) \in \goodqpset$. Define
\begin{align*}
  \goodqsetcond &\defn \brackets{q_0 \in \real^\dims: \Prob_{ \cp_0 \sim \Normal(0, \Ind_\dims)} ((\cq_0, \cp_0) \in \goodqpset) \geq \frac{7}{8}}.
\end{align*}
Because $\Prob_{\cq_0 \sim \target, \cp_0 \sim \Normal(0, \Ind_\dims)} ((\cq_0, \cp_0) \in \goodqpset) \geq 1-\deltaf$ by Lemma~\ref{lem:acceptance rate}, we have $\Prob_{\cq_0 \sim \target}(\cq_0 \in \goodqsetcond) \geq 1-8\deltaf$. Otherwise, using the law of total probability, we would have $1-\deltaf < (1- 8\deltaf)\cdot 1 + 8\deltaf \cdot \frac{7}{8}=1-\deltaf$ which is a contradiction.

We need the following two assumptions regarding the choice of step size $h$ and $\deltaf$. They will be used to bound the $\condS$-conductance $\conductance_\condS$ from below. 

\begin{align}\label{eq:h_condition1}
    \step^2=\frac12h&\leq\frac{1}{6400\smoothparam\sqrt{\dims}\parenth{4+\log\parenth{{2\dims}/{\deltaf}}}^2}\\
    \label{eq:h_condition2}
    8\deltaf &\leq \min\braces{\frac{\condS}{4}, \frac{\sqrt{2 h}}{384}\isop(\target)\cdot \condS}
\end{align}
Denote the acceptance rate at $(\cq_0,\cp_0)$ by
\begin{align*}
    \accept(\cq_0,\cp_0)= \min\braces{\exp\parenth{-f(\hat q_\step)-\frac12\vecnorm{\hat p_\step}{2}^2+f(q_\step)+\frac12\vecnorm{p_\step}{2}^2},1}.
\end{align*}
The accept-reject step in MALA implies that
\begin{align*}
    \transitionbf_{q_0}(\{q_0\}) = \int_{\real^\dims}\parenth{1-\accept(q_0,p_0)}\frac{1}{\sqrt{2\pi}}e^{-\frac{\normp^2}{2}}dp_0.
\end{align*}
And we have
\begin{align}\label{eq:tv_dist}
    \tvdist{\transitionbf_{q_0}}{\proposal_{q_0}}&= \frac 12 \parenth{\transitionbf_{q_0}(\{q_0\}) + \int_{\real^\dims}\big(1-\accept(q_0,p_0)\big)\frac{1}{\sqrt{2\pi}}e^{-\frac{\normp^2}{2}}dp_0}\notag\\
    &=\frac 12\parenth{2-2\int_{\real^\dims}\accept(q_0,p_0)\frac{1}{\sqrt{2\pi}}e^{-\frac{\normp^2}{2}}dp_0}\notag\\
    &=1-\mathbb E_{p_0\sim \Normal(0,\mathbb{I}_\dims)}[\accept(q_0,p_0)]
\end{align}
Applying Lemma~\ref{lem:acceptance rate}, for all $(\cq_0, \cp_0) \in \goodqpset$ we have
\begin{align*}
    &\quad -\targetf(\hq_\step)-\frac12\vecnorm{\hp_\step}{2}^2 + \targetf(\cq_\step) + \frac12\vecnorm{\cp_\step}{2}^2 \\
    &\geq -100 \parenth{4 + \log\parenth{\frac{2\dims}{\deltaf}}}^2 \step^2 \smoothparam \dims^{\frac12}   - 8  \parenth{\sqrt{\dims}+\log\parenth{\frac{12}{\deltaf}}}^2\step^4 \smoothparam^2.\\
    &\overset{(i)}{\geq} -\frac{1}{32}.
  \end{align*}
where we use assumption~\eqref{eq:h_condition1} in step (i).  For any $\cq_0 \in \goodqsetcond$, by definition $\Prob_{\cp_0 \sim \Normal(0, \Ind_\dims)} ((\cq_0, \cp_0) \in \goodqpset) \geq 7/8$. Then Equation~\eqref{eq:tv_dist} implies 
\begin{align}\label{eq:TV_gap}
    \tvdist{\transitionbf_{q_0}}{\proposal_{q_0}}\leq 1-\frac{7}{8} e^{-\frac{1}{32}}\leq\frac 16, \ \ \forall \cp_0\in\goodqsetcond.
\end{align}

Next we show that Equation~\eqref{eq:TV_gap} is enough to bound the $\condS$-conductance. The following conductance argument follows from \citet{dwivedi2018log} and \citet{lee2020logsmooth}. Let $S$ be an arbitrary measurable set with probability $\targetdistri(S)\in(\condS, 1/2]$. Define the sets
\begin{align*}
S_1:=\Big\{x\in S \big| \transition_x(S^c)<\frac18\}, \ \ \ S_2:=\Big\{x\in S^c \big| \transition_x(S)<\frac18\}
\end{align*}
and $S_3 = (S_1\cup S_2)^c$.
Consider two distinct cases below.
\begin{enumerate}[label={(\arabic*)}]
    \item $\targetdistri(S_1)\leq \targetdistri(S)/2$ or $\targetdistri(S_2)\leq \targetdistri(S^c)/2$.
    \item $\targetdistri(S_1)> \targetdistri(S)/2$ and $\targetdistri(S_2)> \targetdistri(S^c)/2$.
\end{enumerate}

In the first case, we get
\begin{align}
\begin{split}\label{eq:case1}
    \targetdistri(S_1)\leq \targetdistri(S)/2&\Rightarrow \int_S\transition_x(S^c) \target(dx) \geq \frac{\targetdistri(S)}{2}\cdot\frac 18 = \frac{1}{16}\targetdistri(S)\\
    \text{ or } \targetdistri(S_2)\leq \targetdistri(S)/2&\Rightarrow \int_S\transition_x(S^c) \target(dx) =\int_{S^c}\transition_x(S) \target(dx)\\
    &\hspace{8.5em}\geq \frac{\targetdistri(S^c)}{2}\cdot\frac 18 \geq \frac{1}{16}\targetdistri(S).
\end{split}
\end{align}

In the second case, we first prove that $\dist(S_1\cap \goodqsetcond,S_2\cap \goodqsetcond)\geq \sqrt{2h}/6$. On the one hand, for any $x\in S_1\cap \goodqsetcond$ and $y\in S_2\cap\goodqsetcond$, we have
\begin{align*}\tvdist{\transition_x}{\transition_y}=\sup_{A}|\transition_x(A)-\transition_y(A)|&\geq |\transition_x(S)-\transition_y(S)|\\&\geq\frac 78-\frac18 = \frac34.
\end{align*}
On the other hand, for $h\leq 2/\smoothparam $, we obtain from Equation~\eqref{eq:lazy} that
\begin{align*}
    \tvdist{\transition_x}{\transition_y}&\leq \frac12 +\frac12 \tvdist{\transitionbf_x}{\transitionbf_x}\\
    &\leq \frac12 + \frac12 \parenth{\tvdist{\transitionbf_x}{\proposal_y}+\tvdist{\proposal_x}{\proposal_y}+\tvdist{\proposal_y}{\transitionbf_y}}\\
    &\overset{(i)}{\leq} \frac12 + \frac12\parenth{\frac 16 +\frac{\|x-y\|_2}{\sqrt{2h}}  + \frac 16}.
\end{align*}
In step $(i)$, the first and the third term are bounded by $1/6$ by Equation~\eqref{eq:TV_gap}, and the second term follows by a direct calculation of the total variation between two Gaussians (see Lemma~7 in \citet{dwivedi2018log}).
Thus $\dist(S_1\cap \goodqsetcond,S_2\cap \goodqsetcond)\geq \sqrt{2h}/6$.  By the isoperimetric inequality of $\target$, we get
\begin{align*}
    \targetdistri(S_3\cup \goodqsetcond^c) = \targetdistri((S_1\cap\goodqsetcond)^c\cap (S_2\cap\goodqsetcond)^c)&\geq  \frac{\sqrt{2h}}{6} \isop(\target)\cdot\targetdistri(S_1\cap \goodqsetcond)\targetdistri(S_2\cap \goodqsetcond).
\end{align*}
Since $\targetdistri(\goodqsetcond)=\Prob_{\cq_0\sim\target}(\cq_0\in\goodqsetcond)\geq 1-8\deltaf$, we have
\begin{align}
    \targetdistri(S_3) + 8\deltaf &= \targetdistri(S_3) + \targetdistri(\goodqsetcond^c)\notag\\
    &\geq \frac{\sqrt{2h}}{6} \isop(\target)\cdot\targetdistri(S_1\cap \goodqsetcond)\targetdistri(S_2\cap \goodqsetcond)\notag\\
    &\geq \frac{\sqrt{2h}}{6} \isop(\target)\cdot \parenth{\targetdistri(S_1)-8\deltaf} \parenth{\targetdistri(S_2)-8\deltaf}\notag\\
    &\overset{(i)}{\geq} \frac{\sqrt{2h}}{6} \isop(\target)\cdot \parenth{\frac{\targetdistri(S)}2-8\deltaf} \parenth{\frac{\targetdistri(S^c)}2-8\deltaf}\notag\\
    &\geq \frac{\sqrt{2h}}{6} \isop(\target)\cdot \parenth{\frac{\targetdistri(S)}2-8\deltaf} \parenth{\frac 14-8\deltaf}\label{eq:S3}.
\end{align}
In step $(i)$, we use the assumption of the second case. Applying assumption~\eqref{eq:h_condition2}, we have
\begin{align*}
    \targetdistri(S_3)&\geq \frac{\sqrt{2h}}{6} \isop(\target)\cdot \frac {\target(S)}4 \cdot \frac 18 -8\deltaf\\
    &\geq  \frac{\sqrt{2 h}}{384}\isop(\target)\cdot\targetdistri(S).
\end{align*}
Thus
\begin{align}
    \int_S\transition_x(S^c)\target(x)dx&=\frac12\bigg(\int_{S}\transition_x(S^c) \target(dx)+\int_{S^c}\transition_x(S) \target(dx)\bigg)\notag\\
    &\overset{(i)}{\geq} \frac12\bigg(\int_{S\cap S_3}\frac18  \target(dx)+\int_{S^c\cap S_3}\frac18  \target(dx)\bigg)\notag\\
    &=\frac1{16}\targetdistri(S_3)\geq \frac{\sqrt{2 h}}{6144}\isop(\target)\cdot \targetdistri(S),\label{eq:case2}
\end{align}
where step $(i)$ is from the definition of $S_3$.

From the analysis of these two cases in Equation~\eqref{eq:case1} and \eqref{eq:case2}, we obtain the following lower bound on the $\condS$-conductance
\begin{align}
    \conductance_\condS&\geq \min\braces{\frac 1{16}, \frac{\sqrt{2h}}{6144}\isop(\target)}.\label{eq:s_cond2}
\end{align}

 Now we solve $h$ from assumption~\eqref{eq:h_condition1} and \eqref{eq:h_condition2}. Let $\constc$ be a constant which may change from line to line. Using $\condS=\tole/2\warmparam$, we observe that $\deltaf$ needs to satisfy
\begin{align*}
    \constc\smoothparam \sqrt{\dims} \deltaf^2 \parenth{4+\log\parenth{\frac{2\dims}{\deltaf}}}^2 &\leq \isop(\target)^2\frac{\tole^2}{\warmparam^2}\\
    \text{ and }\hspace{2em} 64\deltaf &\leq \frac{\tole}{\warmparam}
\end{align*}
Take $\deltaf^{-1}=\constc\max\parenth{\rho\log\rho, \warmparam/\tole}$ where $\constc>0$ is a large enough constant and $\rho=\frac{\smoothparam^{1/2}\dims^{1/4}\warmparam}{\isop(\target)\tole}$ to meet conditions above. Equation~\eqref{eq:h_condition1} gives the final choice of $h$ by
\begin{align}\label{eq:h_final}
    h & = \frac{\constUniversal}{\smoothparam \sqrt{\dims} \cdot \log^2\parenth{\max\braces{\dims, \frac{\smoothparam}{\isop(\target)^2},  \frac{\warmparam}{\tole}, \constUniversalll}}}.
\end{align}
 for some universal constants $\constUniversal, \constUniversalll>0$. By Equation~\eqref{eq:s_cond}, \eqref{eq:s_cond2} and~\eqref{eq:h_final},  mixing time of MALA has an upper bound
\begin{align*}
    \tvmix{\tole}{\initial}\leq \constUniversall\cdot\max\braces{ \frac{\smoothparam\sqrt{\dims}}{\isop(\target)^2}\cdot\log^2\parenth{\max\braces{\dims, \frac{\smoothparam}{\isop(\target)^2},  \frac{\warmparam}{\tole}, \constUniversalll}},1}\cdot \log\parenth{\frac {2\warmparam }\tole}
\end{align*}
for some universal constants $\constUniversall, \constUniversalll>0$.

\subsection{Proof of Theorem \ref{thm:main}}\label{proof:main}

When $\target$ is $\scparam$-strongly log-concave, we have $\isop(\target)\geq \log2\cdot \sqrt{\scparam }$ by Theorem 4.4 in \citet{cousins2014cubic}). Now Theorem~\ref{thm:main} follows from Theorem~\ref{thm:general}.

\subsection{Proof of Theorem \ref{thm:var}}\label{proof:var}
We define the heat kernel with respect to the initial distribution $\initial$, as the ratio of the density of the Markov chain at the $n$-th iteration to the target density, via the following recursion
\begin{align*}
  h_0(x) = \frac{d\initial}{d\target}(x) \text{ and } h_{n}(x) =\frac{d\mu_{n}}{d\target}(x)= \frac{d\transition^n(\initial)}{d\target}(x).
\end{align*}
Note that $\Exs_\target[h_n] = \int_\statespace\mu_n(dx)=1$, we have
\begin{align}\label{eq:var_hn}
\begin{split}
    \Var_\target[h_n]&=\int_\statespace\parenth{h_n(x)-1}^2\target(dx)=\vecnorm{h_n-1}{2,\target}^2\\
    &=\int_\statespace\parenth{\frac{d\mu_n}{d\target}(x)-1}^2\target(dx)=\chidist{\mu_n}{\target}^2.
\end{split}
\end{align}

For $f\in L_2(\target)$, we define a function $\kernel f\in L_2(\target)$ via operating $\kernel$ on the left of $f$ by
\begin{align*}
    \transitionS f(x)=\int_{y\in\statespace}\kernel (x,dy)f(y).
\end{align*}
Since $\kernel $ is reversible, we have
\begin{align*}
    \mu_{n+1}(dx)=\transition(\mu_{n})(dx)&=\int_{y\in\statespace}\mu_n(dy)\kernel(y, dx)\\
    &=\int_{y\in\statespace}h_n(y)\target(dy)\kernel(y, dx)\\
    &=\int_{y\in\statespace}h_n(y)\target(dx)\kernel(x, dy)\\
    &=Kh_n(x)\target(dx).
\end{align*}
As a consequence of the above observation, the heat kernel at $n$-th iteration has a simplified expression
\begin{align*}
    h_{n+1}=\transitionS h_n=\transitionS^{n+1}h_0.
\end{align*}

To prove Theorem \ref{thm:var}, we start with two lemmas. The first lemma relates Dirichlet form $\dirichlet_{\kernel ^2}$ to $\vecnorm{f}{2,\target}^2$ and $\vecnorm{\transitionS^2 f}{2,\target}^2$.
\begin{lemma}\label{lem:EK2}
Let $\kernel$ be the kernel of a reversible Markov chain with invariant distribution $\target$. For any $f\in L_2(\target)$, we have
\begin{align}
  \label{eq:lem_EK2}
  \dirichlet_{\kernel ^2}(f,f)&=\vecnorm{f}{2,\target}^2-\vecnorm{\transitionS f }{2,\target}^2
\end{align}
\end{lemma}

The second lemma provides a control over $\vecnorm{\transitionS^n f }{2,\target}^2$, and is applied to bound $\Var_\target[h_n]$.
\begin{lemma}\label{lem:Sn}
Let $\kernel$ be the kernel of a reversible Markov chain with invariant distribution $\target$. For any $f\in L_2(\target)$, we have
\begin{align}
    \frac{\vecnorm{\transitionS^n f }{2,\target}^2}{\vecnorm{f}{2,\target}^2}\geq\parenth{\frac{\vecnorm{\transitionS f }{2,\target}^2}{\vecnorm{f}{2,\target}^2}}^n, \ \ \ \forall n\in\naturalnum.
\end{align}
\end{lemma}

See Section \ref{proof:EK2} and \ref{proof:Sn} for the proof of these two lemmas.

With the above two lemmas in hand, we are now equipped to prove the mixing time lower bound. By Equation~\eqref{eq:var_hn}, we have
\begin{align*}
    \chidist{\mu_n}{\target}^2&=\vecnorm{h_n-1}{2,\target}^2\\
    &\overset{(i)}{=}\vecnorm{\transitionS^n(h_0-1)}{2,\target}^2\\
    &\overset{(ii)}\geq \vecnorm{h_0-1}{2,\target}^2\parenth{\frac{\vecnorm{\transitionS(h_0-1)}{2,\target}^2}{\vecnorm{h_0-1}{2,\target}^2}} ^n\\
    &\overset{(iii)}{=}\chidist{\mu_0}{\target}^2\cdot\parenth{1-\frac{\dirichlet_{\kernel ^2}(h_0,h_0)}{\chidist{\mu_0}{\target}^2}}^n
\end{align*}
Here step $(i)$ uses the fact that $\transitionS$ is a linear operator and $\transitionS(1)=\int_{y\in\statespace}\kernel(x,dy)=1$. Step $(ii)$ follows by Lemma \ref{lem:Sn} and the step $(iii)$ makes use of Lemma \ref{lem:EK2}.

\subsubsection{Proof of Lemma \ref{lem:EK2}}\label{proof:EK2}
The left hand side of Equation~\eqref{eq:lem_EK2} can be expanded as follows
\begin{align*}
    \dirichlet_{\kernel^2}(f,f)=&\frac12 \int_{x,y\in\statespace^2}(f(x)-f(y))^2\kernel^2(x,dy) \target(dx)\\
    =&\int_{x\in\statespace}(f(x))^2 \target(dx)-\int_{x,y\in\statespace^2}f(x)f(y)\kernel^2(x,dy) \target(dx).
\end{align*}
The first term is $\vecnorm{f}{2,\target}^2$ by definition. The second term is equal to $\vecnorm{\transitionS f }{2,\target}^2$ in that
\begin{align}
    \vecnorm{\transitionS f }{2,\target}^2&=\int_{x\in\statespace}\big(\transitionS f (x)\big)^2 \target(dx)\notag\\
    &=\int_{x\in\statespace}\parenth{\int_{y\in\statespace}\kernel(x,dy)f(y)}^2 \target(dx)\notag\\
    &=\int_{x\in\statespace}\int_{y,z\in\statespace^2}\kernel(x,dy)f(y)\cdot\kernel(x,dz)f(z)\cdot \target(dx)\notag\\
    &\overset{(i)}=\int_{x\in\statespace}\int_{y,z\in\statespace^2}\kernel(x,dy)f(y)\cdot\kernel(z,dx)f(z)\cdot \target(dz)\notag\\
    &\overset{(ii)}{=}\int_{y,z\in\statespace^2}f(z)f(y)\kernel^2(z,dy)  \target(dz)\label{eq:Sf},
\end{align}
where step $(i)$ uses the reversible condition $\kernel(x,dz)\target(dx)=\kernel(z,dx)\target(dz)$, and step $(ii)$ uses the definition of $\kernel^2$ that $\kernel^2(z,dy)=\int_{x\in\statespace} \kernel(z,dx)\kernel(x,dy)$.

\subsubsection{Proof of Lemma \ref{lem:Sn}}\label{proof:Sn}
The proof goes by induction on $n$. First we show that the inequality holds for $n=2$, that is,
\begin{align}
    \frac{\vecnorm{\transitionS^2 f }{2,\target}^2}{\vecnorm{f}{2,\target}^2}\geq\parenth{\frac{\vecnorm{\transitionS f }{2,\target}^2}{\vecnorm{f}{2,\target}^2}}^2\label{eq:S2f}
\end{align}
By Cauchy-Schwartz inequality, we have
\begin{align*}
    \vecnorm{\transitionS^2 f }{2,\target}^2\cdot\vecnorm{f}{2,\target}^2&\geq\parenth{\int_{x\in\statespace}\transitionS^2 f (x)\cdot f(x) \target(dx)}^2\\
    &=\parenth{\int_{x\in\statespace}\int_{y\in\statespace}\kernel^2(x,dy)f(y)\cdot f(x) \target(dx)}^2\\
    &\overset{(i)}{=}\parenth{\vecnorm{\transitionS  f }{2,\target}^2}^2.
\end{align*}
where step $(i)$ is by Equation \eqref{eq:Sf}. The case $n=2$ in Equation~\eqref{eq:S2f} follows by rearranging the terms in the above equation.

Assuming Lemma~\ref{lem:Sn} holds for $n-1$, we obtain that
\begin{align*}
    \frac{\vecnorm{\transitionS^n f }{2,\target}^2}{\vecnorm{f}{2,\target}^2}&=\frac{\vecnorm{\transitionS^{n-1}\parenth{\transitionS f }}{2,\target}^2}{\vecnorm{f}{2,\target}^2}\\
    &=\frac{\vecnorm{\transitionS^{n-1}(\transitionS f )}{2,\target}^2}{\vecnorm{\transitionS f }{2,\target}^2}\cdot\frac{\vecnorm{\transitionS f }{2,\target}^2}{\vecnorm{f}{2,\target}^2}\\
    &\overset{(i)}{\geq}\parenth{\frac{\vecnorm{\transitionS\parenth{\transitionS f }}{2,\target}^2}{\vecnorm{\transitionS f }{2,\target}^2}}^{n-1}\frac{\vecnorm{\transitionS f }{2,\target}^2}{\vecnorm{f}{2,\target}^2}\\
    &\overset{(ii)}{\geq} \parenth{\frac{\vecnorm{\transitionS f }{2,\target}^2}{\vecnorm{f}{2,\target}^2}}^{n-1}\frac{\vecnorm{\transitionS f }{2,\target}^2}{\vecnorm{f}{2,\target}^2}\\
    &=\parenth{\frac{\vecnorm{\transitionS f }{2,\target}^2}{\vecnorm{f}{2,\target}^2}}^{n},
\end{align*}
where step $(i)$ uses the induction hypothesis, and step $(ii)$ uses Equation \eqref{eq:S2f} again.

\subsection{Proof of Corollary \ref{corollary:gap}}\label{proof:gap}

Corollary~\ref{corollary:gap} follows from Theorem~\ref{thm:var} after observing the following lemma which relates $\dirichlet_{\kernel^2}$ to $\dirichlet_\kernel$.
\begin{lemma}\label{lem:dirichlet} Let $\kernel$ be the kernel of a reversible Markov chain with invariant distribution $\target$. For any $f\in L_2(\target)$, we have
\begin{align*}
    \dirichlet_{\kernel^2}(f,f)\leq 2\dirichlet_{\kernel}(f,f).
\end{align*}
\end{lemma}

\begin{proof}
By the definition of Dirichlet form, we have
\begin{align*}
    \dirichlet_{\kernel^2}(f,f)&=\frac 12\int_{x,y\in\statespace^2}\parenth{f(x)-f(y)}^2\kernel^2(x,dy) \target(dx)\\
    &=\frac 12\int_{x,y,z\in\statespace^3}\parenth{f(x)-f(y)}^2\kernel(x,dz)\kernel(z,dy) \target(dx)\\
    &=\frac 12\int_{x,y,z\in\statespace^3}\big(f(x)-f(z)+f(z)-f(y)\big)^2\kernel(x,dz)\kernel(z,dy) \target(dx)\\
    &=\frac 12\int_{x,y,z\in\statespace^3}\big(f(x)-f(z)\big)^2\kernel(x,dz)\kernel(z,dy) \target(dx)\\
    &\quad \quad +\frac 12\int_{x,y,z\in\statespace^3}\big(f(z)-f(y)\big)^2\kernel(x,dz)\kernel(z,dy) \target(dx)\\
    &\quad \quad -\int_{x,y,z\in\statespace^3}\big(f(z)-f(x)\big)\big(f(z)-f(y)\big)\kernel(x,dz)\kernel(z,dy) \target(dx)\\
    &\overset{(i)}{=} \frac 12\int_{x,z\in\statespace^2}\big(f(x)-f(z)\big)^2\kernel(x,dz) \target(dx) \\
    &\quad \quad + \frac 12\int_{z,y\in\statespace^2}\big(f(z)-f(y)\big)^2\kernel(z,dy) \target(dz)\\
    &\quad \quad - \int_{z\in \statespace}\Big(f(z)-\int_{t\in\statespace}f(t)\kernel(z,dt)\Big)^2 \target(dz)\\
    &\leq 2\dirichlet_{\kernel}(f,f),
\end{align*}
where step $(i)$ uses the fact that $\kernel$ is reversible.
\end{proof}
 
Applying Lemma~\ref{lem:dirichlet} and Theorem~\ref{thm:var}, if ${2\dirichlet_{\kernel }(h_0,h_0)}/{\chidist{\mu_0}{\target}^2}\leq1$ we get
\begin{align*}
    \chidist{\mu_n}{\target}^2 &\geq \chidist{\mu_0}{\target}^2 \cdot \parenth{1-\frac{\dirichlet_{\kernel ^2}(h_0,h_0)}{\chidist{\mu_0}{\target}^2}}^n\\
    &\geq\chidist{\mu_0}{\target}^2 \cdot \parenth{1-\frac{2\dirichlet_{\kernel }(h_0,h_0)}{\chidist{\mu_0}{\target}^2}}^n.
\end{align*}
It follows that the mixing time in $\chi^2$-divergence has a lower bound
\begin{align*}
    \chimix{\tole}{\initial}\geq{2}\parenth{-\log\parenth{1-\frac{2\dirichlet_{\kernel }(h_0,h_0)}{\chidist{\mu_0}{\target}^2}}}^{-1}\log\frac{\chidist{\initial}{\target}}{\tole}.
\end{align*}
If the spectral gap satisfies $\dirichlet_{\kernel }(h_0,h_0)/{\chidist{\mu_0}{\target}^2}<1/4$, using $\log(1-x)\geq -x/(1-x)$ for $x\in(0,1)$, we get
\begin{align*}
    \chimix{\tole}{\initial} &\geq \parenth{1-\frac{2\dirichlet_{\kernel }(h_0,h_0)}{\chidist{\mu_0}{\target}^2}}\frac{\chidist{\mu_0}{\target}^2}{\dirichlet_{\kernel }(h_0,h_0)}\log\frac{\chidist{\initial}{\target}}{\tole} \\
    & \geq \frac{1}{2} \parenth{\frac{\dirichlet_{\kernel }(h_0,h_0)}{\chidist{\mu_0}{\target}^2}}^{-1} \log\frac{\chidist{\initial}{\target}}{\tole}.
\end{align*}

\subsection{Proof of Lemma \ref{lem:worst}}\label{proof:worst}

In this section we prove the two statements in Lemma~\ref{lem:worst}. We procced by first constructing difficult warm initialization and then upper bounding the spectral gap. To prove Lemma~\ref{lem:worst}-\ref{state:a}, it suffices to make the initialization $\initial$ different from the target $\target$ in only the last dimension. To prove Lemma~\ref{lem:worst}-\ref{state:b}, we construct a warm initialization supported on a set where the acceptance rate is exponentially small.

\subsubsection{Proof of the statement \ref{state:a} in Lemma \ref{lem:worst}}
Consider a initial distribution $\initial(x)=h_0(\coord{x}{\dims+1})\target(x)$ where $h_0: \real \rightarrow \real$ is the following piece-wise function
\begin{align*}
h_0(u)= \left\{\begin{aligned}
&\frac 1Z|u| & \sqrt{\scparam }|u|\leq 2\\
&\frac 1Z(\frac{4}{\sqrt{\scparam }}-|u|)\ \  & 2<\sqrt{\scparam }|u|\leq4\\
&0 &\sqrt{\scparam }|u|>4
\end{aligned}\right.
\end{align*}
and $Z$ is the normalizing constant that ensures $\int_{\reald}\initial(x)dx=1$,
\begin{align*}
    Z&=2\parenth{\int_0^{\frac{2}{\sqrt{\scparam}}}x\sqrt{\frac{\scparam}{2\pi}}e^{-\frac{\scparam}{2}x^2}dx+\int_{\frac{2}{\sqrt{\scparam}}}^{\frac{4}{\sqrt{\scparam}}}(4-x)\sqrt{\frac{\scparam}{2\pi}}e^{-\frac{\scparam}{2}x^2}dx}\\
    &=\frac{2}{\sqrt{\scparam}}\parenth{\int_0^2\frac{1}{\sqrt{2\pi}}te^{-\frac12t^2}dt+\int_2^4\frac{1}{\sqrt{2\pi}}(4-t)e^{-\frac12t^2}dt}.
\end{align*}
This construction guarantees that the warmness of $\initial(x)$ is
$\warmparam=2/(Z\sqrt{\scparam })$
and
\begin{align*}
|h_0(u)-h_0(v)|\leq \frac{1}{Z}|u-v|, \ \ \forall u,v\in\real.
\end{align*}
Numerical calculations show that $Z\sqrt{m}\in (0.7,0.8)$, the warmness $\warmparam\in(2.6, 2.7)$ and the initial $\chi^2$-divergence $\chidist{\initial}{\target}^2\in(0.4,0.5)$.
Then the spectral gap of this initialization is controlled by
\begin{align*}
    \frac{\dirichlet_{\kernel}(h_0,h_0)}{\chidist{\initial}{\target}^2}
    &= \frac{\frac12\mathbb{E}_{x\sim \target, y\sim \transition_x}\brackets{\parenth{h_0(\coord{x}{\dims+1})-h_0(\coord{y}{\dims+1})}^2}}{\chidist{\initial}{\target}^2}\\
    &\leq \frac 1{2\chidist{\initial}{\target}^2}\mathbb{E}_{x\sim \target, y\sim \transition_x}\brackets{\frac{1}{Z^2}(\coord{x}{\dims+1}-\coord{y}{\dims+1})^2}\\
    &\leq \frac 1{2\chidist{\initial}{\target}^2}\mathbb{E}_{x\sim \target, y\sim \proposal_x}\brackets{\frac{1}{Z^2}(\coord{x}{\dims+1}-\coord{y}{\dims+1})^2}\\
    &\overset{(i)}{\leq} 3\scparam\cdot  \mathbb{E}_{\coord{x}{\dims+1}\sim\Normal(0,1/\scparam ), \xi\sim \Normal(0,1)}\brackets{\parenth{h\cdot \scparam \coord{x}{\dims+1}-\sqrt{2h}\xi}^2}\\
    &\leq 6\scparam\cdot  \parenth{\mathbb{E}_{\coord{x}{\dims+1}\sim\Normal(0,1/\scparam )}\brackets{h^2\scparam ^2\coord{x}{\dims+1}^2}+\mathbb{E}_{\xi\sim\Normal(0,1)}\brackets{2h\xi^2}}\\
    &=6\parenth{\scparam ^2h^2+2\scparam h}\\
    &\leq 18mh,
\end{align*}
where step $(i)$ uses estimations of $Z$ and $\chidist{\initial}{\target}$, and the last step is from $mh\leq 1$.

\subsubsection{Proof of the statement \ref{state:b} in Lemma \ref{lem:worst}}

The target density $\target(x)$ can be written as a product $\target_1(\coord{x}{1:d})\target_2(\coord{x}{d+1})$, where $\target_1$ is the marginal density of the first $d$ dimensions, and $\target_2$ is the marginal density in the last dimension. We claim two lemmas which uppers bound the acceptance rate of MALA with the target distribution being $\target_1$ and $\target_2$ respectively.
\begin{lemma}\label{lem:pi1}
    Fix smoothness parameter $\smoothparam>0$, dimension $\dims$ and scalar $\deltas\in(0,1/20)$ satisfying $\dims^{\deltas}\geq \max\braces{\log d/2+6, 10}$. Consider the following target distribution
    \begin{align}\label{eq:perturbed_gaussian}
        \target_1(x)\propto \exp\parenth{\frac{\smoothparam }{2}\sum_{i=1}^{d}\coord{x}{i}^2-\frac{1}{2d^{\frac12-2\deltas}}\sum_{i=1}^{d}\cos\parenth{d^{\frac14-\deltas}\smoothparam ^{\frac12}\coord{x}{i}}}.
    \end{align}
    Let $\propkernel_1(x,\cdot)$ denote the density of the MALA proposal distribution at $x$. There exists a set $F_1\subset\real^\dims$ satisfying $\target_1(F_1)>1/6$, such that whenever $h\geq 1/\parenth{\smoothparam\dims^{1/2-3\deltas}}$, for any $x\in F_1$, there exists a set $G_x \subseteq \real^\dims$ satisfying
    \begin{align*}
        \int_{G_x}Q_1(x,y)dy&\geq 1-10\exp\parenth{-\frac{\dims^{4\deltas}}{16384}}, \ \text{ and }\\
        \frac{\target_1(y)\propkernel_1(y,x)}{\target_1(x)\propkernel_1(x,y)}&\leq\exp\parenth{ -\frac{\dims^{4\deltas}}{32}}, \ \forall y\in G_x.
    \end{align*}
\end{lemma}
\begin{lemma}\label{lem:pi2}
    Given $\scparam>0$, consider the target distribution $\target_2(x)\propto \exp\parenth{-\scparam x^2/2}$. Let $\propkernel_2(x,\cdot)$ denote the density of the MALA proposal distribution at $x$. There exists a set $F_2\subset\real$ satisfying $\target_2(F_2)\in(1/2, 3/4)$, such that
    \begin{align*}
        \int_{\real}\frac{\target_2(y)\propkernel_2(y,x)}{\target_2(x)}dy\leq 2, \ \forall x\in F_2.
    \end{align*}
\end{lemma}

See Appendix~\ref{sub:pi1} for the proof of Lemma~\ref{lem:pi1}, which is inspired by Theorem 8 in \citet{chewi2021optimal}. It is a stronger result as it includes the smoothness parameter $L$ and holds for a larger range of step size. See Appendix~\ref{sub:pi2} for the proof of Lemma~\ref{lem:pi2}. 

Given these two lemmas, we construct the following initial distribution $\initial$.
\begin{align*}
    \initial(x) = \frac{1}{\target_1(F_1)\target_2(F_2)} \cdot \target(x)\cdot \mathbbm{1}_{x\in F_1\times F_2}
\end{align*}
The warmness of this initial distribution $\warmparam=1/\parenth{\target_1(F_1)\target_2(F_2)}\in (4/3, 12)$ and its initial $\chi^2$-divergence is
\begin{align*}
    \chidist{\initial}{\target}=\bigg(\int_{\real^\dims}\frac{\initial(x)^2}{\target(x)}dx-1\bigg)^{\frac12}=(\warmparam-1)^{\frac12}.
\end{align*}
Then the spectral gap of this initialization is bounded by
\begin{align*}
    \frac{\dirichlet_{\kernel}(h_0,h_0)}{\chidist{\mu_0}{\target}^2}&=\frac{\warmparam^2\Exs_{x\sim \target, y\sim \transition_x}\brackets{(\mathbbm{1}_{x\in F_1\times F_2}-\mathbbm{1}_{y\in F_1\times F_2})^2}}{2(\warmparam-1)}\\
    &= \frac{\warmparam^2}{\warmparam-1}\int_{x\in F_1\times F_2}\int_{y\notin F_1\times F_2}\min\braces{1,\frac{\target(y)\propkernel(y,x)}{\target(x)\propkernel(x,y)}}\target(x)\propkernel(x,y)dydx\\
    &\leq \frac{\warmparam^2}{\warmparam-1}\sup_{x\in F_1\times F_2}\int_{y\in\real^\dims}\min\braces{1,\frac{\target(y)\propkernel(y,x)}{\target(x)\propkernel(x,y)}}\propkernel(x,y)dy\cdot \target(F_1\times F_2)\\
    &\leq \frac{\warmparam}{\warmparam-1}\sup_{x\in F_1\times F_2}\parenth{\int_{G_x\times\real}\frac{\target(y)\propkernel(y,x)}{\target(x)\propkernel(x,y)}\propkernel(x,y)dy+\int_{G_x^c\times\real}\propkernel(x,y)dy}\\
    &\overset{(i)}{\leq}\frac{M}{M-1}\parenth{2\exp\parenth{ -\frac{\dims^{4\deltas}}{32}}+10\exp\parenth{-\frac{\dims^{4\deltas}}{16384}}}\\
    &\overset{(ii)}{\leq} 4\parenth{2\exp\parenth{ -\frac{\dims^{4\deltas}}{32}}+10\exp\parenth{-\frac{\dims^{4\deltas}}{16384}}}\\
    &\leq 48\exp(-\frac{\dims^{4\deltas}}{16384}),
\end{align*}
where step $(i)$ uses Lemma~\ref{lem:pi1}, Lemma~\ref{lem:pi2}, and step $(ii)$ uses the lower bound of $\warmparam$.

\subsection{Proof of Theorem \ref{thm:main2}} \label{proof:main2}
When $\initial$ is $\warmparam$-warm, it is straightforward to show that $\mu_n$ is also $\warmparam$-warm. The $\chi^2$-divergence of $\mu_n$ with respect to $\target$ is bounded by the total variation distance via
\begin{align*}
\chidist{\mu_n}{\target}^2&=\int\parenth{\frac{d\mu_n}{d\target}(x)-1}^2 \target(dx)\leq \warmparam\int\abss{\frac{d\mu_n}{d\target}(x)-1} \target(dx)= 2\warmparam\cdot \tvdist{\mu_n}{\target}.
\end{align*}
and therefore
\begin{align*}
    \chimix{\sqrt{2\warmparam\tole}}{\initial}\leq \tvmix{\tole}{ \initial}
\end{align*}
By the definition of the minimax mixing time, $\minimaxmixingtime(d,\smoothparam ,\scparam ,\tole,\warmparam)\geq \minimaxmixingtime(d,\smoothparam ,\scparam  ,\tole, 12)$. From now on we fix $\warmparam=12$. The condition $\condi\geq3$ enables us to consider the perturbed Gaussian distribution in Lemma \ref{lem:worst}. 

In the definition of $f_\deltas$ in Equation~\eqref{eq:worst} , replace $\smoothparam$ with $2\smoothparam/3$ so that the smoothness of  $f_\deltas$ is exactly $L$. Replace the dimension $\dims$ with $\dims-1$ so that the dimension of $\target$ is exactly $\dims$.  Since we only consider the case when $\dims$ is large, we do not distinguish between $\dims-1$ and $\dims$ in the following proof for simplicity. We study two cases of step size $h$.

\begin{itemize}
    \item In Lemma~\ref{lem:worst}-\ref{state:b}, take
    \begin{align*}
        \deltas =\frac{\log\log(\condi d^{\frac12})+\log12}{\log d}, \text{ or equivalently }  \dims^{\deltas} = 12\log\parenth{\condi\dims^{\frac12}}.
    \end{align*}
     Since $3\leq\condi\leq\alpha\cdot \dims^\beta$, there exists $\integer_{1}>0$ such that when $\dims>\integer_1$ we have $\deltas<1/20$ and $\dims^{\deltas}\geq\max\braces{\log\dims/2+6,10}$. Given that $\exp( d^{4\deltas}/16384)\geq\exp\parenth{(\log\condi d^{1/2})^4}\geq\condi\dims^{1/2}$, there exists $\integer_2>0$ such that $\dirichlet_{\kernel}(h_0,h_0)/\chidist{\mu_0}{\target}^2\leq48 (\condi\dims^{1/2})^{-1}\leq 1/4$ for all $\dims>\integer_2$.  By Corollary~\ref{corollary:gap} and Lemma~\ref{lem:worst}-\ref{state:b}, there exists an $\warmparam$-warm $\initial$ such that for any $h\in(3/(2\smoothparam d^{1/2-3\deltas}),\infty)$ and $\dims>\max\braces{\integer_1, \integer_2}$ we have
    \begin{align}
        \chimix{\sqrt{2\warmparam\tole}}{\initial}&\geq \frac12\cdot\frac{\chidist{\initial}{\target}^2}{\dirichlet_{\kernel }(h_0,h_0)} \log\parenth{\frac{\chidist{\initial}{\target}}{\sqrt{2\warmparam\tole}}}\notag\\ 
        &\geq\frac{1}{96}\condi d^{\frac12}\log\parenth{\frac{1}{10\sqrt{\tole}}}.\label{eq:lower2}
    \end{align}

    \item In Lemma~\ref{lem:worst}-\ref{state:a}, consider $h\leq 3/\parenth{2\smoothparam d^{1/2-3\deltas}}\leq\constc\cdot \max\braces{\log\condi,\log d}^3/\parenth{\smoothparam d^{1/2}}$, where $\constc>0$ is some universal constant. There exists $\integer_3>0$ such that $\scparam h\leq1$ and $\dirichlet_{\kernel}(h_0,h_0)/\chidist{\mu_0}{\target}^2\leq18\scparam h\leq 1/4$ for all $\dims>\integer_3$. By Corollary~\ref{corollary:gap} and Lemma~\ref{lem:worst}-\ref{state:a}, there exists an $\warmparam$-warm $\initial$ such that for any $h\in\parenth{0,3/(2\smoothparam d^{1/2-3\deltas})}$ and $\dims>\integer_3$ we have
    \begin{align}
        \chimix{\sqrt{2\warmparam\tole}}{\initial}&\geq \frac12\cdot\frac{1}{18\scparam h}\log\parenth{\frac {\chidist{\initial}{\target}}{\sqrt{2\warmparam\tole}}}\notag\\
        &\geq\frac{\condi \dims^{\frac12}}{36\constc\cdot \max\braces{\log\condi, \log d}^3}\log\parenth{\frac{1}{10\sqrt{\tole}}},\label{eq:lower4}
    \end{align}
    \end{itemize}

Combining theses two cases by letting $\integer_{\alpha,\beta}=\max\parenth{\integer_1,\integer_2,\integer_3}$, we obtain Theorem~\ref{thm:main2}.


\section{Discussion} 
\label{sec:Discussion}
In this paper, we proved matching upper and lower bounds for the minimax mixing time of Metropolis-adjusted Langevin algorithm under a warm start. Specifically, our results show that for $\smoothparam$-log-smooth and $\scparam$-strongly log-concave target distributions, with step size chosen roughly $\TO\parenth{1/(\smoothparam d^{1/2})}$ , MALA has a mixing time of order $\TO(\condi d^{1/2})$. Furthermore, larger step size can lead to exponentially slow mixing for certain worst-case distributions.

Several open questions arise from our work. First, it is intriguing how to improve the warmness dependency and the error tolerance dependency of MALA. In our mixing time upper bounds, we have polynomial logarithmic dependencies on both warmness and inverse error tolerance. In~\citet{chen2020fast}, the warmness dependency was $\log\log(\warmparam)$ and the error dependency was simply $\log(1/\tole)$ albeit a worse dependency on dimension. It is not clear whether one has to suffer worse dependencies on both $\warmparam$ and $\tole^{-1}$ in order to obtain the tightest bound in terms of dimension and condition number dependency. 

Second, since a warm initialization is not always available in practice, one usually instead initializes the chain using a standard Gaussian distribution centered at $x^*$. For such an exponentially-warm start, one may consider running ULA or underdamped Langevin algorithm for a few steps first to obtain moderate accuracy, and then continue the chain using MALA. We find that this hybrid algorithm mixes much faster than directly running MALA under certain bad initializations in simulations. But whether we can obtain theoretical guarantees for the hybrid algorithm remains open.

Another future work is to apply and adapt our results to Hamiltonian Monte Carlo. In our proof, we showed that in a single-step leapfrog integration, the difference in Hamiltonian is of order roughly $\smoothparam ^2\step^4\dims$. Since Hamiltonian Monte Carlo are usually run with multiple steps of leapfrog integration at each iteration, it remains interesting how to generalize our proof techniques to the case where multiple steps of leapfrog integration are involved.

\subsection*{Acknowledgements}
Part of this work was done when Yuansi Chen was a post-doc researcher at ETH Z\"urich. Yuansi Chen was partially supported by the funding from the European Research Council under the Grant Agreement No 786461 (CausalStats - ERC-2017-AD).

\newpage
\appendix

\section{Lemmas related to the upper bound}\label{sec:ABCD}
In this section we prove Lemma~\ref{lem:accept_rate_kinetic_diff} and Lemma~\ref{lem:accept_rate_potential_diff}. For $t \in [0, \step]$, we define
\begin{align}
  \hq_t &\defn \cq_0 + t\cp_0 - \frac{t^2}{2}\gradf(\cq_0)\label{eq:def_hq}\\
  \tq_t &\defn \cq_0 + t\cp_0 - \frac{t^2}{2} \gradf(\cq_0 + t\cp_0\label{eq:def_tq} ).
\end{align}
The definition of $\hq_t$ is the same as definition of $\hq_\eta$ in Equation~\eqref{eq:leapfrog_q}. The following two lemmas regarding the distance distance between $\cq_0,\cq_t,\hq_t$ and $\tq_t$ will be frequently used in the proof.
\begin{lemma}
\label{lem:q_diff}
Assume the negative log density $\targetf$ is $\smoothparam$-smooth. For step size $\step>0$ satisfying $\smoothparam \step^2\leq 1$ and $t\in[0,\step]$, we have
\begin{align*}
  \vecnorm{q_t-q_0}{2}&\leq  2t \normp + t^2 \normgradq, \\
  \vecnorm{\hq_t-q_0}{2}&\leq t \normp + \frac{t^2}{2} \normgradq,\\
  \vecnorm{\hq_t-\hq_\step}{2}&\leq (\step-t)\normp + \step(\step-t)\normgradq,
\end{align*}
where $q_t,\hq_t,\hq_\step$ are defined in Equation~\eqref{eq:HMC_q}, \eqref{eq:def_hq}  and~\eqref{eq:leapfrog_q}.
\end{lemma}

\begin{lemma}
\label{lem:q_diff2}
Assume the negative log density $\targetf$ is $\smoothparam$-smooth. For step size $\step>0$ satisfying $\smoothparam \step^2\leq 1$ and $t\in[0,\step]$, we have
\begin{align*}
  \vecnorm{q_t-(q_0+tp_0)}{2}&\leq  t ^2 \sqrt{\smoothparam} \vecnorm{p_0}{2} +  t ^2 \normgradq,\\
  \vecnorm{q_t-\hq_t}{2}&\leq  t ^3 \smoothparam \parenth{\frac{1}{3} \normp + \frac{ t }{6} \normgradq}
\end{align*}
where $q_t, \hq_t$ are defined in Equation~\eqref{eq:HMC_q} and \eqref{eq:def_hq}.

\end{lemma}
See Appendix \ref{proof:q_diff} and \ref{proof:q_diff2} for the proof of Lemma~\ref{lem:q_diff} and Lemma~\ref{lem:q_diff2}, respectively.

\subsection{Proof of Lemma~\ref{lem:accept_rate_kinetic_diff}}
\label{proof:lem_accept_rate_kinetic_diff}
Observe that
\begin{align}
  \label{eq:proof_accept_rate_kinetic_diff}
  \targetf(\hq_\step) - \targetf(\cq_\step)
  \overset{(i)}&{=} \int_0^1 (\hq_\step - \cq_\step) \tp \gradf\parenth{r (\hq_\step - \cq_\step) + \cq_\step} dr \notag \\
  &= (\hq_\step - \cq_\step) \tp \gradf(\cq_0) + \int_0^1 (\hq_\step - \cq_\step) \tp \parenth{\gradf\parenth{r (\hq_\step - \cq_\step) + \cq_\step} - \gradf(\cq_0)} dr \notag \\
  \overset{(ii)}&{=} \underbrace{\int_0^\step \int_0^s \parenth{\gradf(\cq_\tau) - \gradf(\cq_0)} \tp \gradf(\cq_0) d\tau ds}_{A_1} \notag \\
  &\quad + \underbrace{\int_0^1 \int_0^\step \int_0^s \parenth{\gradf(\cq_\tau) - \gradf(\cq_0)} \tp \parenth{\gradf\parenth{r (\hq_\step - \cq_\step) + \cq_\step} - \gradf(\cq_0)} d\tau ds dr}_{A_2},
\end{align}
where step (i) follows from the fact that the function $r \mapsto \targetf\parenth{r (\hq_\step - \cq_\step) + \cq_\step} $ is differentiable, step (ii) plugs in the definition of $\hq_\step$~\eqref{eq:leapfrog_q} and $\cq_\step$~\eqref{eq:HMC_q}.

The term $A_2$ in Equation~\eqref{eq:proof_accept_rate_kinetic_diff} is relatively easy to bound. We have
\begin{align}
  \label{eq:accept_reject_A2_bound}
  \abss{A_2}
  &\leq \int_0^1 \int_0^\step \int_0^s \vecnorm{\gradf(\cq_\tau) - \gradf(\cq_0)}{2} \vecnorm{\gradf\parenth{r (\hq_\step - \cq_\step) + \cq_\step} - \gradf(\cq_0)}{2} d\tau ds dr \notag \\
  \overset{(i)}&{\leq} \smoothparam^2 \int_0^1 \int_0^\step \int_0^s \vecnorm{\cq_\tau - \cq_0}{2} \vecnorm{r(\hq_\step-\cq_0) + (1-r)(\cq_\step - \cq_0)}{2} d\tau ds dr \notag \\
  \overset{(ii)}&{\leq} \smoothparam^2 \int_0^1 \int_0^\step \int_0^s \parenth{2\tau \normp + \tau^2 \normgradq }  \parenth{2\step \normp + \step^2\normgradq } d\tau ds dr \notag\\
  &\leq \frac{2}{3} \step^4 \smoothparam^2 \parenth{\normp + \frac{\step}{2}\normgradq }^2 \notag \\
  \overset{(iii)}&\leq \frac{2}{3} \step^4 \smoothparam^2 \parenth{\normp + \frac{1}{\sqrt{\smoothparam}}\normgradq }^2,
\end{align}
where step (i) follows from the smoothness assumption, step (ii) applies Lemma~\ref{lem:q_diff} and step (iii) uses $\smoothparam \step^2 \leq 1$. We bound the term $A_1$ by making appear the term $\gradf(\cq_0 + \tau \cp_0)$
\begin{align}
  \label{eq:accept_reject_A1_bound}
  A_1
  &= \int_0^\step\int_0^s(\gradf(\cq_\tau) - \gradf(\cq_0+\tau \cp_0))^\top\gradf(\cq_0)d\tau ds\notag \\
  &\quad \quad + \int_0^\step\int_0^s(\gradf(\cq_0+\tau \cp_0)-\gradf(\cq_0))^\top\gradf(\cq_0)d\tau ds \notag\\
  \overset{(i)}&{=} \int_0^\step\int_0^s(\gradf(\cq_\tau) - \gradf(\cq_0+\tau \cp_0))^\top\gradf(\cq_0)d\tau ds \notag \\
  &\quad \quad + \frac{1}{2}\int_0^\step\int_0^s \vecnorm{\gradf(\cq_0+\tau \cp_0)}{2}^2-\normgradq^2 d\tau ds\notag\\
  &\quad \quad - \frac{1}{2} \int_0^\step \int_0^s \vecnorm{\gradf(\cq_0 + \tau \cp_0) - \gradf(\cq_0) }{2}^2 d\tau ds \notag \\
  \overset{(ii)}&{\leq} \underbrace{\int_0^\step\int_0^s(\gradf(\cq_\tau) - \gradf(\cq_0+\tau \cp_0))^\top\gradf(\cq_0) }_{A_{1,1}}d\tau ds+ \frac{1}{2}\int_0^\step\int_0^s (\interG_\tau - \interG_0) d\tau ds,
\end{align}
where step (i) follows from $(a-b)\tp b = \frac{1}{2} (a^2-b^2) - \frac{1}{2} (a-b)^2 $ for any $a,b \in \real^\dims$, step (ii) removes the last nonnegative term and uses the definition of $\interG_\tau$ in Equation~\eqref{eq:Gt}. For the term $A_{1, 1}$, we have
\begin{align}
  \label{eq:accept_reject_A11_bound}
  A_{1, 1}
  &\leq \int_0^\step\int_0^s \vecnorm{\gradf(\cq_\tau) - \gradf(\cq_0+\tau \cp_0)}{2}\normgradq d\tau ds \notag \\
  \overset{(i)}&{\leq} \smoothparam \int_0^\step\int_0^s \vecnorm{\cq_\tau - (\cq_0+\tau \cp_0)}{2}\normgradq d\tau ds \notag \\
  \overset{(ii)}&{\leq} \smoothparam \int_0^\step \int_0^s (\tau^2 \sqrt{\smoothparam}\normp + \tau^2 \normgradq ) \normgradq d\tau ds \notag \\
  &= \frac{\step^4\smoothparam}{12} \parenth{\sqrt{\smoothparam}\normp + \normgradq } \normgradq \notag \\
  \overset{(iii)}&{\leq} \frac{\step^4\smoothparam^2}{12} \parenth{\normp + \frac{1}{\sqrt{\smoothparam}}\normgradq }^2,
\end{align}
where step (i) follows from the smoothness assumption, step (ii) follows from Lemma~\ref{lem:q_diff2} and step (iii) follows from $(a+b)b \leq (\frac{1}{2}a+b)^2$ for $a, b \in \real$.

Combining Equation~\eqref{eq:accept_reject_A2_bound}~\eqref{eq:accept_reject_A1_bound}~\eqref{eq:accept_reject_A11_bound} with Equation~\eqref{eq:proof_accept_rate_kinetic_diff}, we obtain
\begin{align*}
  \targetf(\hq_\step) - \targetf(\cq_\step) \leq \frac{1}{2}\int_0^\step\int_0^s (\interG_\tau - \interG_0) d\tau ds + \frac{3}{4}\step^4\smoothparam^2 \parenth{\normp + \frac{1}{\sqrt{\smoothparam}}\normgradq }^2.
\end{align*}

\subsection{Proof of Lemma~\ref{lem:accept_rate_potential_diff}}

Observe that we can decompose $\frac{1}{2}\vecnorm{\hp_\step}{2}^2 - \frac{1}{2}\vecnorm{\cp_\step}{2}^2$ as follows
\begin{align}
  \label{eq:proof_accept_reject_potential_diff}
  & \quad \frac{1}{2}\vecnorm{\hp_\step}{2}^2 - \frac{1}{2}\vecnorm{\cp_\step}{2}^2 \notag \\
  &= \frac{1}{2} \parenth{\hp_\step - \cp_\step}\tp \parenth{\hp_\step + \cp_\step} \notag \\
  \overset{(i)}&= \frac{1}{2} \parenth{\int_0^\step \gradf(\cq_s) ds - \frac{\step}{2}\gradf(\cq_0) - \frac{\step}{2}\gradf(\hq_\step) }\tp \parenth{\hp_\step + \cp_\step} \notag \\
  \overset{(ii)}&= \underbrace{\frac12\parenth{\int_0^\step \gradf(\cq_s) - \gradf\parenth{\hq_s} ds}\tp \parenth{\hp_\step + \cp_\step}}_{A_3}\notag \\
  &\quad \quad + \underbrace{\frac12\parenth{\int_0^\step \gradf\parenth{\hq_s } ds - \frac{\step}{2} \gradf(\cq_0) - \frac{\step}{2} \gradf\parenth{\hq_\step }} \tp \parenth{\hp_\step + \cp_\step}}_{A_4}
\end{align}
where step (i) plugs the definition of $\hp_\step$ and $\cp_\step$ in Equation~\eqref{eq:leapfrog_p}~\eqref{eq:HMC_p} and step (ii) adds and substracts $\hq_s$ related terms. The terms $A_3$ is relatively easier to bound. We have
\begin{align}
  \label{eq:hp_plus_cp_naive_bound}
  \vecnorm{\hp_\step + \cp_\step}{2}
  \overset{(i)}&= \vecnorm{2\cp_0-\int_0^\step\parenth{\gradf(\cq_s)+\frac{1}{2}(\gradf(\cq_0)+\gradf(\hq_\step))}ds}{2} \notag \\
  &= \vecnorm{2p_0- 2\step \gradf(\cq_0) + \int_0^\step\parenth{\gradf(\cq_0) - \gradf(\cq_s) +\frac{1}{2}(\gradf(\cq_0)- \gradf(\hq_\step))}ds}{2} \notag \\
  \overset{(ii)}&\leq \parenth{2\vecnorm{\cp_0 - \step \gradf(\cq_0)}{2}  + \smoothparam \int_0^\step \vecnorm{\cq_0 - \cq_s}{2} ds + \frac{\smoothparam}{2} \int_0^\step \vecnorm{\cq_0 - \hq_\step}{2} ds} \notag \\
  \overset{(iii)}&\leq \parenth{2\normp + 2\step \normgradq + 2\smoothparam \step^2 (\normp + \frac{\step}{2} \normgradq) } \notag \\
  \overset{(iv)}&\leq 4\parenth{\normp + \frac{1}{\sqrt{\smoothparam}} \normgradq},
\end{align}
where step (i) plugs the definition of $\hp_\step$ and $\cp_\step$ in Equation~\eqref{eq:leapfrog_p}~\eqref{eq:HMC_p}, step (ii) uses the smoothness assumption, step (iii) uses Lemma~\ref{lem:q_diff} and step (iv) uses $\smoothparam \step^2 \leq 1$. Using the bound~\eqref{eq:hp_plus_cp_naive_bound}, we can bound the term $A_3$ as follows
\begin{align}
  \label{eq:accept_reject_A3_bound}
  \abss{A_3} &\leq \frac12\int_0^\step \vecnorm{\gradf(\cq_s) - \gradf(\hq_s) }{2} \vecnorm{\hp_\step + \cp_\step}{2} ds \notag \\
  \overset{(i)}&\leq \int_0^\step \vecnorm{\gradf(\cq_s) - \gradf(\hq_s) }{2} ds \cdot 2\parenth{\normp + \frac{1}{\sqrt{\smoothparam}} \normgradq} \notag \\
  \overset{(ii)}&\leq \int_0^\step \smoothparam \vecnorm{\cq_s - \hq_s}{2} ds \cdot 2\parenth{\normp + \frac{1}{\sqrt{\smoothparam}} \normgradq} \notag \\
  \overset{(iii)}&\leq \int_0^\step \smoothparam^2 s^3 \parenth{\frac{1}{3} \normp + \frac{\step}{6}\normgradq } ds \cdot 2\parenth{\normp + \frac{1}{\sqrt{\smoothparam}} \normgradq} \notag \\
  \overset{(iv)}&\leq \frac{\step^4 \smoothparam^2}{6} \parenth{\normp + \frac{1}{\sqrt{\smoothparam}} \normgradq}^2,
\end{align}
where step (i) follows from the bound~\eqref{eq:hp_plus_cp_naive_bound}, step (ii) uses the smoothness assumption, step (iii) uses Lemma~\ref{lem:q_diff2} and step (iv) uses $\smoothparam \step^2 \leq 1$. For term $A_4$, the bound~\eqref{eq:hp_plus_cp_naive_bound} is no longer tight enough. We can decompose $A_4$ into two terms
\begin{align}
  \label{eq:accept_reject_A4_bound}
  A_4 &= \parenth{\int_0^\step \gradf\parenth{\hq_s } ds - \frac{\step}{2} \gradf(\cq_0) - \frac{\step}{2} \gradf\parenth{\hq_\step }} \tp \notag\\
  &\quad \quad \parenth{\cp_0-\frac{1}{2}\int_0^\step\parenth{\gradf(\cq_s)+\frac{1}{2}(\gradf(\cq_0)+\gradf(\hq_\step))}ds} \notag \\
  &= A_{4,1} + A_{4,2},
  \end{align}
  where
  \begin{align*}
  A_{4,1}&=\parenth{\int_0^\step \gradf\parenth{\hq_s } ds - \frac{\step}{2} \gradf(\cq_0) - \frac{\step}{2} \gradf\parenth{\hq_\step }} \tp \parenth{\cp_0- \step \gradf(\cq_0)},  \notag \\
  A_{4,2} &= \parenth{\int_0^\step \gradf\parenth{\hq_s } ds - \frac{\step}{2} \gradf(\cq_0) - \frac{\step}{2} \gradf\parenth{\hq_\step }} \tp \\
  &\quad\quad \brackets{\frac{1}{2}\int_0^\step\parenth{\gradf(\cq_0) - \gradf(\cq_s) +\frac{1}{2}\parenth{\gradf(\cq_0) - \gradf(\hq_\step)}}ds}.
\end{align*}
For $A_{4,2}$, we have
\begin{align}
  \label{eq:accept_reject_A42_bound}
  &\quad \abss{A_{4,2}} \notag \\
  &\leq \frac{1}{2}\vecnorm{\int_0^\step \gradf\parenth{\hq_s } ds - \frac{\step}{2} \gradf(\cq_0) - \frac{\step}{2} \gradf\parenth{\hq_\step } ds}{2} \notag \\
  &\quad \quad \cdot \vecnorm{\int_0^\step\parenth{\gradf(\cq_0) - \gradf(\cq_s) +\frac{1}{2}\parenth{\gradf(\cq_0) - \gradf(\hq_\step)}}ds}{2} \notag \\
  \overset{(i)}&\leq \frac{\smoothparam^2}{2} \parenth{ \int_0^{\step/2} \vecnorm{\hq_s - \cq_0}{2} ds +  \int_{\step/2}^{\step} \vecnorm{\hq_s - \hq_\step}{2} ds } \parenth{\int_0^\step \vecnorm{\cq_s - \cq_0}{2} ds + \frac{\step}{2} \vecnorm{\hq_\step - \cq_0}{2} } \notag \\
  \overset{(ii)}&\leq \frac{\smoothparam^2}{2} \parenth{\frac{\step^2}{4}\normp + \frac{\step^3}{4} \normgradq} \parenth{\frac{3\step^2}{2}\normp + \frac{3 \step^3}{2} \normgradq } \notag \\
  \overset{(iii)}&\leq \frac{\step^4 \smoothparam^2}{4} \parenth{\normp + \frac{1}{\sqrt{\smoothparam}} \normgradq}^2,
\end{align}
where step (i) uses the smoothness assumption, step (ii) uses Lemma~\ref{lem:q_diff} and step (iii) uses $\step^2 \smoothparam \leq 1$.

\begin{align}
  \label{eq:accept_reject_A41_bound}
  A_{4,1} &= \parenth{\int_0^\step \gradf\parenth{\hq_s } ds - \frac{\step}{2} \gradf(\cq_0) - \frac{\step}{2} \gradf\parenth{\hq_\step }} \tp \parenth{\cp_0- \step \gradf(\cq_0)} \notag \\
  &= \underbrace{\int_0^\step \gradf\parenth{\hq_s } \tp \parenth{\cp_0- s \gradf(\cq_0)} ds - \frac{\step}{2} \gradf(\cq_0)\tp\cp_0 - \frac{\step}{2}\gradf(\hq_\step)\tp\parenth{\cp_0 - \step \gradf(\cq_0)}}_{A_{4,1,1}}  \notag \\
  &\quad + \underbrace{\int_0^\step \gradf\parenth{\hq_s } \tp \parenth{(s - \step) \gradf(\cq_0)}ds + \frac{\step^2}{2} \gradf(\cq_0)\tp \gradf(\cq_0)}_{A_{4,1,2}}
\end{align}
For term $A_{4,1,2}$, we have
\begin{align}
  \label{eq:accept_reject_A412_bound}
  A_{4,1,2}
  &= \int_0^\step (s - \step) \gradf\parenth{\hq_s } \tp \gradf(\cq_0)ds + \frac{\step^2}{2} \gradf(\cq_0)\tp \gradf(\cq_0) \notag \\
  &= \int_0^\step (s - \step) \gradf\parenth{\cq_0 + s \cp_0 } \tp \gradf(\cq_0)ds + \frac{\step^2}{2} \gradf(\cq_0)\tp \gradf(\cq_0)  \notag \\
  &\quad + \int_0^\step (s - \step) \parenth{\gradf\parenth{\hq_s } - \gradf\parenth{\cq_0 + s \cp_0 }} \tp \gradf(\cq_0)ds \notag \\
  \overset{(i)}&= \int_0^\step (s - \step) \brackets{\frac{1}{2} \parenth{\vecnorm{\gradf\parenth{\cq_0 + s \cp_0 }}{2}^2 - \vecnorm{\gradf(\cq_0)}{2}^2} - \frac{1}{2} \vecnorm{\gradf(\cq_0 + s\cp_0) - \gradf(\cq_0)}{2}^2}ds  \notag \\
  &\quad + \int_0^\step (s - \step) \parenth{\gradf\parenth{\hq_s } - \gradf\parenth{\cq_0 + s \cp_0 }} \tp \gradf(\cq_0)ds \notag \\
  \overset{(ii)}&\leq \frac{1}{2} \int_0^\step (s - \step) \parenth{\vecnorm{\gradf\parenth{\cq_0 + s \cp_0 }}{2}^2 - \vecnorm{\gradf(\cq_0)}{2}^2} ds + \frac{\step^4 \smoothparam^2}{12} \parenth{\normp + \frac{1}{\smoothparam} \normgradq}^2 \notag \\
  &= \frac{1}{2} \int_0^\step (s-\step) \parenth{\interG_s - \interG_0} ds + \frac{\step^4 \smoothparam^2}{12} \parenth{\normp + \frac{1}{\smoothparam} \normgradq}^2,
\end{align}
where step (i) follows from the fact that $a\tp b = \frac{1}{2}\parenth{\vecnorm{a}{2}^2 - \vecnorm{b}{2}^2} - \frac{1}{2}\vecnorm{a-b}{2}^2  + \vecnorm{b}{2}^2$ for any $a, b\in \real^\dims$, step (ii) uses
\begin{align*}
  \int_0^{\step} (\step - s) \vecnorm{\gradf(\cq_0 + s\cp_0) - \gradf(\cq_0)}{2}^2 ds \leq \smoothparam^2 \normp^2 \int_0^\step (\step - s) s^2 ds = \frac{\step^4\smoothparam^2}{12} \normp^2,
\end{align*}
and also
\begin{align*}
  &\quad \abss{\int_0^\step (s - \step) \parenth{\gradf\parenth{\hq_s } - \gradf\parenth{\cq_0 + s \cp_0 }} \tp \gradf(\cq_0)ds} \\
  &\leq \int_0^\step (\step - s) \smoothparam \vecnorm{\frac{s^2}{2}\gradf(\cq_0) }{2} \normgradq ds \\
  &\leq \smoothparam \int_0^\step \frac{(\step - s) s^2}{2} \normgradq^2 ds \\
  &\leq \frac{\step^4 \smoothparam}{24} \normgradq^2.
\end{align*}

The term $A_{4,1,1}$ is the most difficult term in this lemma. We  replace $\hq_s$ in $A_{4,1,1}$ with  $\tq_s$ defined in Equation~\eqref{eq:def_tq},
and denote the replaced quantity by
\begin{align}
  \label{eq:def_accept_reject_key_term}
  \arkey_\step(\cq_0, \cp_0) \defn \int_{0}^\step \gradf(\tq_s)\tp \parenth{\cp_0 - s \gradf(\cq_0)} ds - \frac{\step}{2} \gradf(\cq_0)\tp \cp_0 - \frac{\step}{2} \gradf(\tq_\step)\tp (\cp_0 - \step \gradf(\cq_0)).
\end{align}
To bound $A_{4,1,1}$, we need the following lemma which provides a high probability bound for $\arkey_\step(\cq_0, \cp_0)$ when $\cq_0$ is randomly drawn from $\target$ and $\cp_0$ is independently drawn from $\Normal(0, \Ind_\dims)$.
\begin{lemma}
  \label{lem:accept_reject_key_term_bound}
  Assume the negative log density $\targetf$ is $\smoothparam$-smooth and convex. There exists a set $\goodqpset \subset \real^\dims \times \real^\dims$ with $\Prob_{\cq_0 \sim \target, \cp_0 \sim \Normal(0, \Ind_\dims)} \parenth{ (\cq_0, \cp_0) \in \goodqpset} \geq {1-\deltaf}$, such that for $(\cq_0, \cp_0)\in \goodqpset$, for $\arkey_\step$ defined in Equation~\eqref{eq:def_accept_reject_key_term},  $\smoothparam \step^2 \leq 1$, we have
  \begin{align*}
    \arkey_\step(\cq_0, \cp_0) &\leq 100 \parenth{4 + \log\parenth{\frac{2\dims}{\deltaf}}}^2 \step^2 \smoothparam \dims^{\frac12}, \text{ and } \\
    \vecnorm{\gradf(\cq_0)}{2} &\leq \sqrt{\smoothparam} \parenth{\sqrt{\dims} + \log\parenth{\frac{12}{\deltaf}  }}, \text{ and }\\
    \vecnorm{\cp_0}{2} &\leq \sqrt{\dims} + \log\parenth{\frac{12}{\deltaf}  }.
  \end{align*}
\end{lemma}
The proof of Lemma~\ref{lem:accept_reject_key_term_bound} is deferred to Appendix~\ref{sub:proof_lem_accept_reject_key_term_bound}. Since $\arkey_\step(\cq_0, \cp_0)$ is obtained by replacing $\hq_s$ with $\tq_s$ in $A_{4,1,1}$, we have
\begin{align}\label{eq:A411_B_eta}
    \vecnorm{A_{4,1,1}-\arkey_\step(\cq_0, \cp_0)}{2}&\leq \frac{3}{2} \step \max_{s \in [0,\step]} \vecnorm{\gradf(\hq_s)-\gradf(\tq_s)}{2} \vecnorm{\cp_0 - s\gradf(\cq_0)}{2} \notag\\
    &\overset{(i)}{\leq} \frac{3\step^4 \smoothparam^2}{4} \vecnorm{\cp_0}{2} \parenth{\vecnorm{\cp_0}{2} + \step \vecnorm{\gradf(\cq_0)}{2}} \notag\\
    &\overset{(ii)}{\leq} \frac{3\step^4 \smoothparam^2}{4} \parenth{\vecnorm{\cp_0}{2} + \frac{1}{\sqrt{\smoothparam}}\vecnorm{\gradf(\cq_0}{2}}^2.
\end{align}
Step $(i)$ uses $ \vecnorm{\gradf(\hq_s)-\gradf(\tq_s)}{2}\leq \smoothparam \vecnorm{\hq_s-\tq_s}{2}\leq \smoothparam^2 s^3\vecnorm{\cp_0}{2}/2$, and step $(ii)$ uses $\step^2\smoothparam\leq 1$.

Combining Lemma~\ref{lem:accept_reject_key_term_bound}, Equation~\eqref{eq:A411_B_eta},~\eqref{eq:accept_reject_A412_bound},~\eqref{eq:accept_reject_A41_bound},~\eqref{eq:accept_reject_A42_bound} into Equation~\eqref{eq:accept_reject_A4_bound}, and then combining it with~\eqref{eq:accept_reject_A3_bound} into Equation~\eqref{eq:proof_accept_reject_potential_diff}, we obtain that, there exists a set $\goodqpset \subset \real^\dims \times \real^\dims$ with $\Prob_{\cq_0 \sim \target, \cp_0 \sim \Normal(0, \Ind_\dims)}((\cq_0, \cp_0) \in \goodqpset) \geq 1-\deltaf$, such that for $(\cq_0, \cp_0) \in \goodqpset$ and the step-size choice $\step^2\smoothparam \leq 1$, we have
\begin{align*}
  \frac{1}{2} \vecnorm{\hp_\step}{2}^2 - \frac{1}{2} \vecnorm{\cp_\step}{2}^2 &\leq 
   \frac{1}{2} \int_0^\step (s-\step) \parenth{\interG_s - \interG_0} ds + 100 \parenth{4 + \log\parenth{\frac{2\dims}{\deltaf}}}^2 \step^2 \smoothparam \dims^{\frac12}  \\
   &\quad + \frac{5\step^4 \smoothparam^2}{4} \parenth{\normp + \frac{1}{\sqrt{\smoothparam}} \normgradq}^2.
\end{align*}
Finally, bounds on $\vecnorm{\cp_0}{2}$ and $\vecnorm{\gradf(\cq_0)}{2}$ in Lemma~\ref{lem:accept_rate_potential_diff} come from Lemma~\ref{lem:accept_reject_key_term_bound}.


\subsection{Proof of Lemma \ref{lem:q_diff}}\label{proof:q_diff}
Recall from the definition~\eqref{eq:HMC_q} and~\eqref{eq:def_hq} that
\begin{align*}
    \cq_t &= \cq_0+t \cp_0-\int_0^t\int_0^s\gradf(\cq_\tau)d\tau ds\\
    \hq_t &= \cq_0+t \cp_0 -\frac{t^2}{2}\gradf(\cq_0).
\end{align*}
Directly from the above definition, we obtain the second part of the lemma via triangular inequality
\begin{align*}
  \vecnorm{\hq_t - \cq_0}{2} \leq t \normp + \frac{t^2}{2} \normgradq.
\end{align*}
Similarly,
\begin{align*}
    \vecnorm{\hq_t - \hq_\step}{2} &\leq (\step-t) \normp + \frac{\step^2-t^2}{2} \normgradq\\
    &\leq(\step-t) \normp +\step(\step-t)\normgradq.
\end{align*}
For the first part, we have
\begin{align*}
    \vecnorm{\cq_t-\cq_0}{2}
    &\leq t \vecnorm{\cp_0}{2}+\int_0^t\int_0^s  \vecnorm{\gradf(\cq_\tau)}{2} d\tau ds \\
    &\leq t\normp + \int_0^t\int_0^s\parenth{\vecnorm{\gradf(\cq_\tau)-\gradf(\cq_0)}{2}+\normgradq }d\tau ds \\
    \overset{(i)}&{\leq} t \normp+\int_0^t\int_0^s\parenth{\smoothparam \vecnorm{\cq_\tau - \cq_0}{2} + \normgradq }d\tau ds \\
    &\leq t \normp + \frac12\smoothparam t^2 \sup_{\tau\in[0,t]}\vecnorm{\cq_\tau-\cq_0}{2} + \frac12t^2 \normgradq,
\end{align*}
where step (i) follows from the smoothness assumption.
Taking supreme of the left-hand side over $\tau \in [0, t]$ and rearranging terms, we obtain
\begin{align*}
    (1-\frac12 \smoothparam t^2)\sup_{\tau\in[0,t]}\vecnorm{\cq_\tau-\cq_0}{2} \leq t \normp +\frac{t^2}2 \normgradq.
\end{align*}
Because $\smoothparam t^2\leq\smoothparam\step^2\leq1$, we have
\begin{align*}
\vecnorm{\cq_t-\cq_0}{2}\leq 2\parenth{t \normp+\frac{t^2}2 \normgradq}.
\end{align*}

\subsection{Proof of Lemma \ref{lem:q_diff2}}\label{proof:q_diff2}
Compared with the bound of $\vecnorm{\cq_t-\cq_0}{2}$ in Lemma~\ref{lem:q_diff}, the bound of $\vecnorm{\cq_t-(\cq_0+t \cp_0)}{2}$ requires an extra factor of $ t $. The main proof strategy remains the same. We have
\begin{align}
  \label{eq:lemma4_proof_1}
  \vecnorm{\cq_t-(\cq_0+t\cp_0)}{2}&\leq\int_0^t\int_0^s\vecnorm{\gradf(\cq_\tau)}{2} d\tau ds\notag\\
  & \leq\int_0^t\int_0^s\parenth{\vecnorm{\gradf(\cq_\tau)-\gradf(\cq_0)}{2}+\normgradq} d\tau ds\notag\\
  \overset{(i)}&{\leq}\int_0^ t \int_0^s\parenth {\smoothparam \vecnorm{\cq_\tau - \cq_0}{2} + \normgradq} d\tau ds\notag\\
  \overset{(ii)}&{\leq}\int_0^ t \int_0^s\parenth{\smoothparam \parenth{2 t  \normp +  t ^2 \normgradq}+ \normgradq } d\tau ds\notag\\
  \overset{(iii)}&{\leq} \int_0^ t \int_0^s\parenth{2  \sqrt{\smoothparam} \normp + 2 \normgradq } d\tau ds\notag\\
  &= t ^2 \sqrt{\smoothparam} \normp +  t ^2 \normgradq,
\end{align}
where step (i) follows from the smoothness assumption, step (ii) uses Lemma \ref{lem:q_diff} and step (iii) uses $\smoothparam t^2\leq1$.
Compared to the above term, the bound of $\vecnorm{\cq_t-\hq_t}{2}$ requires another factor of $t$. We have
\begin{align*}
    \vecnorm{\cq_t-\hq_t}{2}
    &=\vecnorm{\int_0^t\int_0^s\parenth{\gradf(\cq_\tau)-\gradf(\cq_0)} d\tau ds}{2}\\
    \overset{(i)}&{\leq} \int_0^t\int_0^s\smoothparam \vecnorm{\cq_\tau - \cq_0}{2}d\tau ds\\
    \overset{(ii)}&{\leq} \int_0^t\int_0^s\smoothparam \parenth{2 \tau \normp + 2\tau^2 \normgradq } d\tau ds\\
    \overset{(iii)}&{\leq} t^3 \smoothparam \parenth{\frac{1}{3} \normp + \frac{t}{6} \normgradq}.
\end{align*}
where step (i) follows from the smoothness assumption, step (ii) uses Lemma~\ref{lem:q_diff}, step (iii) just completes the integration.

\subsection{Proof of Lemma~\ref{lem:accept_reject_key_term_bound}}
\label{sub:proof_lem_accept_reject_key_term_bound}
Given $(\cq_0, \cp_0) \in \real^\dims \times \real^\dims$ and $t > 0$, define
\begin{align}
  \label{eq:def_arkpower}
  \arkpower_t(\cq_0, \cp_0) \defn \gradf\parenth{\tq_t} \tp \parenth{\cp_0- t \gradf(\cq_0)} - \gradf\parenth{\cq_0}\tp \cp_0.
\end{align}
For $\arkey_\step$ defined in Equation~\eqref{eq:def_accept_reject_key_term}, we have $\arkey_\step(\cq_0, \cp_0) = \int_0^\step \arkpower_s(\cq_0, \cp_0) ds - \frac{\step}{2} \arkpower_\step(\cq_0, \cp_0) $. To prove a high probability bound for $\arkey_\step(\cq_0, \cp_0)$, we first bound $\arkpower_s(\cq_0, \cp_0)$ in high probability via Markov's inequality. We need the following two lemmas.
\begin{lemma}
  \label{lem:arkpower_expectation}
  For $s \geq 0$, for $\arkpower_s$ defined in~\eqref{eq:def_arkpower}, we have
  \begin{align*}
    \Exsqp \arkpower_s(\cq_0, \cp_0) = 0.
  \end{align*}
\end{lemma}
\begin{lemma}
  \label{lem:arkpower_kth_expectation}
  There exists a universal constant $c > 0$ such that for $s > 0$, $s^2 \smoothparam < 1$ and $k\geq 4$ an even positive integer, for $\arkpower_s$ defined in~\eqref{eq:def_arkpower},
  \begin{align*}
    &\quad \Exsqp \smoothparam^{-\frac k 2}\arkpower_s^k(\cq_0, \cp_0) \\
    &\leq \max\braces{1, \ 60 (3k)^k s^k \max\braces{\smoothparam^{\frac k2} \dims^{\frac k 2},\  \smoothparam^{\frac k2} \Exsp \vecnorm{\cp_0}{2}^k,\  \Exsq \vecnorm{\gradf(\cq_0)}{2}^k}}.
  \end{align*}
\end{lemma}

The proofs of Lemma~\ref{lem:arkpower_expectation} and Lemma~\ref{lem:arkpower_kth_expectation} are deferred to  Appendix~\ref{proof:arkpower_expectation} and \ref{proof:arkpower_kth_expectation}. Assuming these two lemmas for now, we complete the proof of Lemma~\ref{lem:accept_reject_key_term_bound}.

First, to provide upper bounds for the moments $\Exsp \vecnorm{\cp_0}{2}^k$ and $\Exsq \vecnorm{\gradf(\cq_0)}{2}^k$, we evoke the following bound established in Theorem 5 of~\cite{lee2020logsmooth}. For $f$ twice-differentiable, $\smoothparam$-smooth and convex, $\target \propto e^{-f}$, we have for $\lambda \in \parenth{0, \frac{2}{\sqrt{\smoothparam}}}$,
\begin{align}
  \label{eq:lee_etal_expectation_gradient_norm}
  \Exsq \brackets{\exp \parenth{\lambda \normgradq}} \leq \frac{1 + \frac12 \sqrt{\smoothparam} \lambda }{ 1- \frac12 \sqrt{\smoothparam} \lambda } \exp\parenth{\lambda \sqrt{\smoothparam\dims} }.
\end{align}
Applying the above result, for $\tau^2 \smoothparam \leq 1$, we have
\begin{align}
  \label{eq:k_th_gradient_power_explicit_bound}
  \frac{1}{k!} \Exsq \tau^{2k} \smoothparam^{\frac k2} \vecnorm{\gradf(\cq_0)}{2}^k
  &\overset{(i)}{\leq} \Exsq \brackets{\exp \parenth{\tau^2 \smoothparam^{\frac12} \normgradq}} \notag \\
  &\overset{(ii)}{\leq} \frac{1 + \frac{1}{2} \tau^2 \smoothparam }{1 - \frac{1}{2} \tau^2 \smoothparam } \exp \parenth{\tau^2 \smoothparam \dims^{\frac12} } \notag \\
  &\leq 3 \exp \parenth{\tau^2 \smoothparam \dims^{\frac12} }
\end{align}
(i) follows from the power series expansion of $\exp$. (ii) makes use of Equation~\eqref{eq:lee_etal_expectation_gradient_norm} with $\tau^2 \smoothparam \leq 1$. Similarly, we have
\begin{align}
  \label{eq:k_th_p_power_explicit_bound}
  \frac{1}{k!} \Exsq \tau^{2k} \smoothparam^k \vecnorm{\cp_0}{2}^k
  \leq 3 \exp \parenth{\tau^2 \smoothparam \dims^{\frac12} }
\end{align}
Second, plugging Equation~\eqref{eq:k_th_gradient_power_explicit_bound} and Equation~\eqref{eq:k_th_p_power_explicit_bound} into the bound in Lemma~\ref{lem:arkpower_kth_expectation}, we have for $s \in (0, \step]$,
\begin{align*}
  &\quad \Exsqp \step^k \arkpower_s^k(\cq_0, \cp_0) \\
  &\leq \max\braces{\step^k \smoothparam^{\frac k2}, 60 (3k)^k \parenth{\step^2 \smoothparam \dims^{\frac12} }^k, 180 (3k)^k k! \parenth{\frac{\step}{\tau}}^{2k}\exp\parenth{\tau^2 \smoothparam \dims^{\frac12}}} \\
  &\overset{(i)}{\leq} \max\braces{\step^k \smoothparam^{\frac k2}, 180 e^2 k^{2k+\frac12} \parenth{\frac 3 e}^k \parenth{\step\smoothparam^{\frac12} \dims^{\frac14}}^{2k} }.
\end{align*}
The last step (i) takes $\tau^2 = 1/\parenth{\smoothparam \dims^{1/2}}$ and uses the upper bound of the Stirling's approximation (see for example, upper bound of the Gamma function in~\cite{mortici2011improved}).

Third, fix $\deltaf >0$ and take $k$ to be the smallest even integer larger than $\max\braces{4, \log\parenth{2\dims/{\deltaf}}}$. By Markov's inequality and the expectation calculation in Lemma~\ref{lem:arkpower_expectation}, we obtain
\begin{align*}
  \Prob_{\cq_0 \sim \target, \cp_0 \sim \Normal\parenth{0, \Ind_\dims}}\parenth{\step \abss{\arkpower_s(\cq_0, \cp_0)} \geq \alpha} &\leq \frac{\Exsqp \step^k \arkpower_s^k(\cq_0, \cp_0)}{\alpha^k}\\
  &\leq \frac {\max\braces{\step^k \smoothparam^{\frac k2}, 180 e^2 k^{2k+\frac12} \parenth{\frac 3 e}^k \parenth{\step^2\smoothparam \dims^{\frac12}}^{k} }} {\alpha^k}.
\end{align*}
Taking $\alpha = 50 k^2 \step^2\smoothparam \dims^{\frac12}$, we have
\begin{align*}
  \Prob_{\cq_0 \sim \target, \cp_0 \sim \Normal\parenth{0, \Ind_\dims}}\parenth{\step \abss{\arkpower_s(\cq_0, \cp_0)} \geq \alpha} \leq e^{-k} \leq \frac{\deltaf}{2\dims}.
\end{align*}
Because $k \leq 4 + \log\parenth{\frac{2\dims}{\deltaf}}$, for $s \in (0, \step]$, we have
\begin{align}
  \label{eq:arkpower_high_prob_bound}
  \Prob_{\cq_0 \sim \target, \cp_0 \sim \Normal\parenth{0, \Ind_\dims}}\parenth{\step \abss{\arkpower_s(\cq_0, \cp_0)} \geq 50 \parenth{4 + \log\parenth{\frac{2\dims}{\deltaf}}}^2 \step^2 \smoothparam \dims^{\frac12} } \leq \frac{\deltaf}{2\dims}.
\end{align}
To obtain a high probability bound for $\arkey_\step$ defined in Equation~\eqref{eq:def_accept_reject_key_term} from the high probability bound for $\arkpower_s$, we build a covering the segment $[0, \step]$ and apply union bound. We have, for any integer $\dims \geq 1$,
\begin{align*}
  \arkey_\step(\cq_0, \cp_0) &= \int_0^\step \arkpower_s(\cq_0, \cp_0) ds - \frac{\step}{2} \arkpower_\step(\cq_0, \cp_0) \\
  &\leq \frac{3}{2} \sup_{s \in [0, \step]} \step \abss{\arkpower_s(\cq_0, \cp_0)} \\
  &\leq \frac{3}{2} \sup_{s \in \braces{0, \frac{\step}{\dims}, 2\frac{\step}{\dims}, \ldots, (\dims-1)\frac{\step}{\dims}}} \step \abss{\arkpower_s(\cq_0, \cp_0)} + \frac{\step}{\dims}  \sup_{s \in [0, \step]} \step \abss{\frac{\partial \arkpower_s\parenth{\cq_0, \cp_0}}{\partial s}}
\end{align*}

For the derivative of $\arkpower_s$ with respect to $s$, we have
\begin{align*}
  \abss{\frac{\partial \arkpower_s\parenth{\cq_0, \cp_0}}{\partial s}}
  &= \parenth{\cp_0 - s \gradf(\cq_0 + s\cp_0) - \frac{s^2}{2} \hessf_{\cq_0 + \cp_0} \cp_0 }\tp \hessf_{\tq_s} \parenth{\cp_0 - s\gradf(\cq_0)} \\
  &\quad \quad - \gradf(\tq_s) \tp \gradf(\cq_0) \\
  &\leq 8 \smoothparam \vecnorm{\cp_0}{2}^2 + 8 \vecnorm{\gradf(\cq_0)}{2}^2
\end{align*}
Using the Gradient norm concentration (see Corollary 6 in ~\cite{lee2020logsmooth}), we have
\begin{align*}
  \Prob_{\cq_0 \sim \target} \brackets{\vecnorm{\gradf(\cq_0)}{2} \geq \sqrt{\smoothparam \dims} + \gamma \sqrt{\smoothparam}} \leq 3 e^{-\gamma}.
\end{align*}
Thus for $\gamma = \log\parenth{12/{\deltaf}}$, we have
\begin{align*}
  \Prob_{\cq_0 \sim \target} \brackets{\vecnorm{\gradf(\cq_0)}{2} \geq \sqrt{\smoothparam \dims} + \log\parenth{\frac{12}{\deltaf}} \sqrt{\smoothparam}} \leq \frac{\deltaf}{4}.
\end{align*}
Similarly, we have
\begin{align*}
  \Prob_{\cp_0 \sim \Normal(0, \Ind_\dims)} \brackets{\vecnorm{\cp_0}{2} \geq \sqrt{\dims} + \log\parenth{\frac{12}{\deltaf}} } \leq \frac{\deltaf}{4}.
\end{align*}
Denote the following three events,
\begin{small}
\begin{align*}
  E_1 &= \braces{(\cq_0, \cp_0) \in \real^\dims \times \real^\dims \mid \sup_{s \in \braces{0, \frac{\step}{\dims}, 2\frac{\step}{\dims}, \ldots, (\dims-1)\frac{\step}{\dims}}} \step \abss{\arkpower_s(\cq_0, \cp_0)} \leq 50 \parenth{4 + \log\parenth{\frac{2\dims}{\deltaf}}}^2 \step^2 \smoothparam \dims^{\frac12} }, \\
  E_2 &= \braces{(\cq_0, \cp_0) \in \real^\dims \times \real^\dims \mid \vecnorm{\gradf(\cq_0)}{2} \leq \sqrt{\smoothparam \dims} + \log\parenth{\frac{12}{\deltaf}  }\sqrt{\smoothparam}}, \\
  E_3 &= \braces{(\cq_0, \cp_0) \in \real^\dims \times \real^\dims \mid \vecnorm{\cp_0}{2} \leq \sqrt{\dims} + \log\parenth{\frac{12}{\deltaf}}}.
\end{align*}
\end{small}
For $(\cq_0, \cp_0) \in E_1 \cap E_2 \cap E_3$, we have
\begin{align*}
  \arkey_\step(\cq_0, \cp_0) &\leq 75 \parenth{4 + \log\parenth{\frac{2\dims}{\deltaf}}}^2 \step^2 \smoothparam \dims^{\frac12} + \frac{16 \step^2}{\dims} \smoothparam \parenth{\sqrt{\dims} + \log\parenth{\frac{12}{\deltaf}}}^2 \\
  &\leq 100 \parenth{4 + \log\parenth{\frac{2\dims}{\deltaf}}}^2 \step^2 \smoothparam \dims^{\frac12}.
\end{align*}
Furthermore, we have
\begin{align*}
  \Prob\parenth{(\cq_0, \cp_0) \in \parenth{E_1 \cap E_2 \cap E_3}^c} &\leq \Prob(E_1^c) + \Prob(E_2^c) + \Prob(E_3^c) \\
  &\leq \dims \cdot \frac{\deltaf}{2\dims} + \frac{\deltaf}{4} + \frac{\deltaf}{4}\\
  &\leq \deltaf.
\end{align*}

\subsubsection{Proof of Lemma~\ref{lem:arkpower_expectation}}\label{proof:arkpower_expectation}
For $s > 0$, we have
\begin{align*}
  &\quad \Exs_{\cq_0 \sim \target, \cp_0 \sim \Normal(0, \Ind_\dims)} \brackets{\gradf(\tq_s)\tp \cp_0} \\
  &= \int_{\real^\dims} \int_{\real^\dims} \brackets{\gradf\parenth{x + s \cp_0 - \frac{s^2}{2} \gradf(x + s \cp_0) }}\tp \cp_0 e^{-\targetf(x)} \frac{1}{\sqrt{(2\Pi)^\dims}} e^{-\frac{\vecnorm{\cp_0}{2}^2}{2}} dp_0 dx \\
  &\overset{(i)}{=} \frac{1}{s} \int_{\real^\dims} \int_{\real^\dims } \brackets{ \gradf\parenth{y - \frac{s}{2} \gradf(y)} \tp (y-x)   } e^{-\targetf(x)} \frac{1}{\sqrt{(2\Pi s^2)^\dims}} e^{-\frac{\vecnorm{y-x}{2}^2}{2s^2}} d y d x \\
  &\overset{(ii)}{=} s \int_{\real^\dims} \int_{\real^\dims } \brackets{ \gradf\parenth{y - \frac{s}{2} \gradf(y)} \tp \gradf(x)  } e^{-\targetf(x)} \frac{1}{\sqrt{(2\Pi s^2)^\dims}} e^{-\frac{\vecnorm{y-x}{2}^2}{2s^2}} d y d x \\
  &\overset{(iii)}{=} s \Exs_{p_0 \sim \Normal(0, \Ind_\dims), q_0 \sim \target} \brackets{ \gradf(\tq_s) \tp \gradf(q_0)  } \\
\end{align*}
(i) applies change of variable $p_0 \leftarrow (y - x)/s$. (ii) applies integration by parts with respect to $x$ (with $u = \gradf\parenth{y - \frac{s}{2} \gradf(y)} e^{-\targetf(x)}$ and $v = e^{-\frac{\vecnorm{y-x}{2}^2}{2s^2}}$), and the boundary term is zero. (iii) changes the variable back $y \leftarrow x + sp_0$. Note that the above derivation requires $s > 0$. But for the case $s = 0$, we trivially have $\Exsqp \gradf(\cq_0) \tp \cp_0 = 0$ since $\Exs_{\cp_0\sim \Normal(0, \Ind_\dims)} [p_0] = 0$. Overall, we have proved that for any $s \geq 0$,
\begin{align*}
  \Exsqp \brackets{\gradf(\tq_s)\tp \parenth{\cp_0 - s\gradf(\cq_0)}} = 0.
\end{align*}
Consequently, we have
\begin{align*}
  \Exsqp \arkpower_s(\cq_0, \cp_0) = 0.
\end{align*}

\subsubsection{Proof of Lemma~\ref{lem:arkpower_kth_expectation}}\label{proof:arkpower_kth_expectation}
The main idea to upper bound the expectation of the $k$-th power of $\arkpower$ is to use integration by parts and H\"older's inequality to establish a recursive relationship for it. Before we dive into integration by parts, we first establish a few derivatives of $\arkpower^k$ that become handy in the rest part of the proof. Using the change of variables $x \leftarrow \cq_0, y \leftarrow \cq_0 + s\cp_0$, we have
\begin{align}
  \begin{split}\label{eq:arkpower_partial_derivatives}
  \arkpower_s\parenth{x, \frac{y-x}{s}} &= \gradf\parenth{y-\frac{s^2}{2}\gradf(y)}^\top\parenth{\frac{y-x}{s}-s\gradf(x)}-\gradf(x)^\top\parenth{\frac{y-x}{s}} \\
  \frac{\partial \arkpower_s\parenth{x, \frac{y-x}{s}}}{\partial x} &= - \frac{1}{s} \gradf\parenth{y - \frac{s^2}{2}\gradf(y)} - s \hessf_x \gradf\parenth{y - \frac{s^2}2\gradf(y)} \\
  &\quad - \hessf_x \parenth{\frac{y-x}{s}} + \frac{1}{s}\gradf(x) \\
  \frac{\partial \arkpower_s\parenth{x, \frac{y-x}{s}}}{\partial y} &= \parenth{\Ind_\dims - \frac{s^2}{2} \hessf_y } \hessf_{y - \frac{s^2}{2} \gradf(y)} \parenth{\frac{y-x}{s} - s\gradf(x)} \\
  &\quad + \frac{1}{s} \gradf\parenth{y - \frac{s^2}{2} \gradf(y)} - \frac{1}{s} \gradf(x).
  \end{split}
\end{align}
Using the $\smoothparam$-smoothness assumption and $s^2 \smoothparam \leq 1$, it is not hard to obtain the following derivative bounds
\begin{align}
  \label{eq:arkpower_partial_derivative_bounds}
  \vecnorm{\frac{\partial \arkpower_s\parenth{x, \frac{y-x}{s}}}{\partial x}}{2} &\leq 5 \smoothparam \vecnorm{\frac{y-x}{s}}{2} + 3 s \smoothparam \vecnorm{\gradf(x)}{2} \notag \\
  \vecnorm{\frac{\partial \arkpower_s\parenth{x, \frac{y-x}{s}}}{\partial y}}{2} &\leq 3 \smoothparam \vecnorm{\frac{y-x}{s}}{2} + 2 s \smoothparam \vecnorm{\gradf(x)}{2} \notag \\
  \abss{\frac{\partial \arkpower_s\parenth{\cq_0, \cp_0}}{\partial s}} &\leq 8 \smoothparam \vecnorm{\cp_0}{2}^2 + 8 \vecnorm{\gradf(\cq_0)}{2}^2
\end{align}

We have
\begin{small}
\begin{align}
  \label{eq:arkpower_kth_expectation_first_steps}
  &\quad \Exsq \Exsp \arkpower^k_s(\cq_0, \cp_0) \notag \\
  &\overset{(i)}{=} \iint \parenth{\gradf\parenth{y - \frac{s^2}{2}\gradf(y) }\tp \parenth{\frac{y-x}{s} - s \gradf(x)}  - \gradf(x)\tp \parenth{\frac{y-x}{s}}}^k  \cdot e^{-\targetf(x)} \frac{e^{-\frac{\vecnorm{y-x}{2}^2}{2s^2}}}{\sqrt{(2\Pi s^2)^\dims}}  dy dx \notag\\
  &= \iint \gradf\parenth{y - \frac{s^2}{2}\gradf(y) }\tp \parenth{\frac{y-x}{s} - s \gradf(x)} \arkpower^{k-1}_s\parenth{x, \frac{y-x}{s}} \cdot  e^{-\targetf(x)} \frac{e^{-\frac{\vecnorm{y-x}{2}^2}{2s^2}}}{\sqrt{(2\Pi s^2)^\dims}}  dy dx \notag \\
  &\quad - \iint \gradf(x)\tp \parenth{\frac{y-x}{s}} \arkpower^{k-1}_s\parenth{x, \frac{y-x}{s}} \cdot  e^{-\targetf(x)} \frac{e^{-\frac{\vecnorm{y-x}{2}^2}{2s^2}}}{\sqrt{(2\Pi s^2)^\dims}}  dy dx \notag \\
  &\overset{(ii)}{=} - s (k-1) \iint \gradf\parenth{y - \frac{s^2}{2}\gradf(y) }\tp \frac{\partial \arkpower_s\parenth{x, \frac{y-x}{s}}}{\partial x} \arkpower^{k-2}_s\parenth{x, \frac{y-x}{s}} \cdot  e^{-\targetf(x)} \frac{e^{-\frac{\vecnorm{y-x}{2}^2}{2s^2}}}{\sqrt{(2\Pi s^2)^\dims}}  dy dx \notag \\
  &\quad - s (k-1) \iint \gradf(x) \tp \frac{\partial \arkpower_s\parenth{x, \frac{y-x}{s}}}{\partial y} \arkpower^{k-2}_s\parenth{x, \frac{y-x}{s}} \cdot  e^{-\targetf(x)} \frac{e^{-\frac{\vecnorm{y-x}{2}^2}{2s^2}}}{\sqrt{(2\Pi s^2)^\dims}}  dy dx
\end{align}
\end{small}
(i) applies change of variable $\cq_0 \leftarrow x$ and $\cp_0 \leftarrow (y-x)/s$. (ii) applies integration by parts with respect to $x$ in the first term and applies integration by parts with respect to $y$ in the second term, the boundary terms are zero.

Observe that parts of the two integrals can be combined together, we have
\begin{small}
\begin{align}
  \label{eq:arkpower_kth_expectation_second_steps}
  &\quad - s \gradf\parenth{y - \frac{s^2}{2}\gradf(y)} \tp \frac{\partial \arkpower_s\parenth{x, \frac{y-x}{s}}}{\partial x}  - s \gradf(x) \tp \frac{\partial \arkpower_s\parenth{x, \frac{y-x}{s}}}{\partial y} \notag \\
  &= \underbrace{\vecnorm{\gradf\parenth{y - \frac{s^2}{2}\gradf(y)} - \gradf(x) }{2}^2}_{T_1} \notag\\
  &\quad + \underbrace{s^2 \gradf\parenth{y - \frac{s^2}{2} \gradf(y)} \tp \hessf_x \gradf(y-\frac{s^2}{2}\gradf(y)) - s^2 \gradf(x)\tp\parenth{\Ind_\dims - \frac{s^2}{2} \hessf_y } \hessf_{y - \frac{s^2}{2} \gradf(y)} \gradf(x)}_{T_2} \notag \\
  &\quad + \underbrace{\gradf\parenth{y - \frac{s^2}{2}\gradf(y) } \tp \hessf_x (y-x)}_{T_3} \underbrace{-\gradf(x) \tp \parenth{\Ind_\dims - \frac{s^2}{2}\hessf_y }\hessf_{y - \frac{s^2}{2} \gradf(y)} (y-x)}_{T_4}
\end{align}
\end{small}
where we used the derivative formula in Equation~\eqref{eq:arkpower_partial_derivatives} and reorganized the terms.
$T_1$ and $T_2$ can be bounded by linear combinations of $\vecnorm{\frac{y-x}{s}}{2}$ and $\vecnorm{\gradf(x)}{2}$ via triangle inequalities. We have
\begin{align*}
  \abss{T_1} &\leq 4 \smoothparam^2 \vecnorm{y-x}{2}^2 + s^4 \smoothparam^2 \vecnorm{\gradf(x)}{2}^2 \\
  \abss{T_2} &\leq 2 s^2 \smoothparam \abss{T_1} + 3 s^2 \smoothparam \vecnorm{\gradf(x)}{2}^2.
\end{align*}
$T_3$ and $T_4$ require another treatment of integration by parts. For $T_3$, using integration by parts with respect to $y$, we have
\begin{small}
\begin{align*}
  &\quad \iint \gradf\parenth{y - \frac{s^2}{2}\gradf(y) } \tp \hessf_x (y-x) \arkpower^{k-2}_s \parenth{x, \frac{y-x}{s}} \cdot e^{-\targetf(x)} \frac{e^{-\frac{\vecnorm{y-x}{2}^2}{2s^2}}}{\sqrt{(2\Pi s^2)^\dims}}  dy dx \\
  &= s^2 \iint \trace\parenth{\parenth{\Ind_\dims - \frac{s^2}{2} \hessf_y}\hessf_{y - \frac{s^2}{2} \gradf(y) } \hessf_x} \arkpower^{k-2}_s \parenth{x, \frac{y-x}{s}} \cdot e^{-\targetf(x)} \frac{e^{-\frac{\vecnorm{y-x}{2}^2}{2s^2}}}{\sqrt{(2\Pi s^2)^\dims}}  dy dx \\
  &\quad - s^2 (k-2)\iint \gradf\parenth{y - \frac{s^2}{2}\gradf(y)} \tp \hessf_x   \frac{\partial \arkpower_s\parenth{x, \frac{y-x}{s}}}{\partial y} \arkpower^{k-3}_s\parenth{x, \frac{y-x}{s}} \cdot  e^{-\targetf(x)} \frac{e^{-\frac{\vecnorm{y-x}{2}^2}{2s^2}}}{\sqrt{(2\Pi s^2)^\dims}}  dy dx.
\end{align*}
\end{small}
Hence, using the derivative bound~\eqref{eq:arkpower_partial_derivative_bounds}, we have
\begin{small}
\begin{align*}
  &\quad \abss{\iint T_3 \arkpower^{k-2}_s \parenth{x, \frac{y-x}{s}} \cdot e^{-\targetf(x)} \frac{e^{-\frac{\vecnorm{y-x}{2}^2}{2s^2}}}{\sqrt{(2\Pi s^2)^\dims}}  dy dx} \\
  &\leq 2 s^2 \smoothparam^2 \dims \iint \abss{\arkpower^{k-2}_s \parenth{x, \frac{y-x}{s}}} \cdot e^{-\targetf(x)} \frac{e^{-\frac{\vecnorm{y-x}{2}^2}{2s^2}}}{\sqrt{(2\Pi s^2)^\dims}}  dy dx \\
  &\quad + s^2(k-2)\iint \parenth{ 9 \smoothparam^{\frac32} \vecnorm{\gradf(x)}{2}^2 + 11 \smoothparam^{\frac52} \vecnorm{\frac{y-x}{s}}{2}^2} \abss{\arkpower^{k-3}_s \parenth{x, \frac{y-x}{s}}} \cdot e^{-\targetf(x)} \frac{e^{-\frac{\vecnorm{y-x}{2}^2}{2s^2}}}{\sqrt{(2\Pi s^2)^\dims}}dydx.
\end{align*}
\end{small}
Similarly, for $T_4$, using integration by parts with respect to $x$, we have
\begin{small}
\begin{align*}
    &\quad \iint \gradf(x) \tp \parenth{\Ind_\dims - \frac{s^2}{2} \hessf_y } \hessf_{y - \frac{s^2}{2} \gradf(y)} (y-x) \arkpower^{k-2}_s \parenth{x, \frac{y-x}{s}} \cdot e^{-\targetf(x)} \frac{e^{-\frac{\vecnorm{y-x}{2}^2}{2s^2}}}{\sqrt{(2\Pi s^2)^\dims}}  dy dx \\
    &\overset{(ii)}{=} s^2 \iint \trace\parenth{\parenth{\hessf_x - \gradf(x)\gradf(x)\tp} \parenth{\Ind_d - \frac{s^2}{2}\hessf_y}\hessf_{y - \frac{s^2}{2} \gradf(y)}}\\
    &\hspace{20em}\cdot \arkpower^{k-2}_s \parenth{x, \frac{y-x}{s}} \cdot e^{-\targetf(x)} \frac{e^{-\frac{\vecnorm{y-x}{2}^2}{2s^2}}}{\sqrt{(2\Pi s^2)^\dims}}  dy dx\\
    &\quad + s^2(k-2) \iint \gradf(x)\tp \parenth{\Ind_d - \frac{s^2}{2}\hessf_y} \hessf_{y-\frac{s^2}{2}\gradf(y)} \frac{\partial \arkpower_s(x, \frac{y-x}{s})}{\partial x}\\
    &\hspace{20em}\cdot \arkpower^{k-3}_s\parenth{x, \frac{y-x}{s}} \cdot  e^{-\targetf(x)} \frac{e^{-\frac{\vecnorm{y-x}{2}^2}{2s^2}}}{\sqrt{(2\Pi s^2)^\dims}}  dy dx
\end{align*}
\end{small}
Hence, using the derivative bound~\eqref{eq:arkpower_partial_derivative_bounds}, we have
\begin{small}
\begin{align*}
  &\quad \abss{\iint T_4 \arkpower^{k-2}_s \parenth{x, \frac{y-x}{s}} \cdot e^{-\targetf(x)} \frac{e^{-\frac{\vecnorm{y-x}{2}^2}{2s^2}}}{\sqrt{(2\Pi s^2)^\dims}}  dy dx} \\
  &\leq 2 s^2 \smoothparam^2 \dims \iint \abss{\arkpower^{k-2}_s \parenth{x, \frac{y-x}{s}}} \cdot e^{-\targetf(x)} \frac{e^{-\frac{\vecnorm{y-x}{2}^2}{2s^2}}}{\sqrt{(2\Pi s^2)^\dims}}  dy dx \\
  &\quad + s^2 \iint \smoothparam \vecnorm{\gradf(x)}{2}^2 \abss{\arkpower^{k-2}_s \parenth{x, \frac{y-x}{s}}} \cdot e^{-\targetf(x)} \frac{e^{-\frac{\vecnorm{y-x}{2}^2}{2s^2}}}{\sqrt{(2\Pi s^2)^\dims}}  dy dx \\
  &\quad + s^2(k-2)\iint \parenth{ 6 \smoothparam^{\frac32} \vecnorm{\gradf(x)}{2}^2 + 3 \smoothparam^{\frac52} \vecnorm{\frac{y-x}{s}}{2}^2} \abss{\arkpower^{k-3}_s \parenth{x, \frac{y-x}{s}}} \cdot e^{-\targetf(x)} \frac{e^{-\frac{\vecnorm{y-x}{2}^2}{2s^2}}}{\sqrt{(2\Pi s^2)^\dims}}dydx.
\end{align*}
\end{small}

Finally, combining the $T_1, T_2, T_3, T_4$ bounds into Equation~\eqref{eq:arkpower_kth_expectation_second_steps} and then Equation~\eqref{eq:arkpower_kth_expectation_first_steps}, we obtain
\begin{align*}
  &\quad \abss{\Exsq \Exsp \arkpower^k_s(\cq_0, \cp_0)} \\
  &\leq (k-1) \Exsqp \brackets{\parenth{4 s^2 \smoothparam^2 \dims + 12 s^2 \smoothparam^2 \vecnorm{\cp_0}{2}^2 + 7 s^2 \smoothparam \vecnorm{\gradf(q_0)}{2}^2 } \abss{\arkpower_s^{k-2}(\cq_0, \cp_0) }} \\
  &\quad + (k-1)(k-2) \Exsqp \brackets{\parenth{14 s^2 \smoothparam^{\frac52} \vecnorm{\cp_0}{2}^2 + 15 s^2 \smoothparam^{\frac32} \vecnorm{\gradf(q_0)}{2}^2 } \abss{\arkpower_s^{k-3}(\cq_0, \cp_0) }}.
\end{align*}
For $k\geq 4$ an even positive integer, applying H\"older's inequality, we relate $\arkpower_s^{k-2}$ and $\arkpower_s^{k-3}$ with $\arkpower_s^k$. We have
\begin{align*}
  &\quad \Exsqp \arkpower^k_s(\cq_0, \cp_0) \\
  &\leq k s^2 \smoothparam \parenth{4 \smoothparam \dims + 12 \smoothparam \parenth{\Exsp \vecnorm{\cp_0}{2}^k}^{2/k} + 7 \parenth{\Exsq \vecnorm{\gradf(\cq_0)}{2}^k}^{2/k} }\\
  &\hspace{17em}\cdot \parenth{\Exsqp \arkpower^k_s(\cq_0, \cp_0)}^{\frac{k-2}{k}} \\
  &\quad + k^2 s^2 \smoothparam^{\frac32} \parenth{14 \smoothparam \parenth{\Exsp \vecnorm{\cp_0}{2}^k}^{2/k} + 15 \parenth{\Exsq \vecnorm{\gradf(\cq_0)}{2}^k}^{2/k}  }\\
  &\hspace{17em}\cdot \parenth{\Exsqp \arkpower^k_s(\cq_0, \cp_0)}^{\frac{k-3}{k}}.
\end{align*}
Hence
\begin{align*}
  &\quad \Exsqp \smoothparam^{-\frac k2}\arkpower^k_s(\cq_0, \cp_0) \\
  &\leq \max\braces{1, 60 (3k)^k s^k \max\braces{\smoothparam^{\frac k2} \dims^{\frac k 2}, \smoothparam^{\frac k2} \Exsp \vecnorm{\cp_0}{2}^k,  \Exsq \vecnorm{\gradf(\cq_0)}{2}^k}}.
\end{align*}

\section{Lemmas related to the lower bound}
We provide the proof of Lemma~\ref{lem:pi1} and Lemma~\ref{lem:pi2} in Appendix~\ref{sub:pi1} and \ref{sub:pi2}.
\subsection{Proof of Lemma~\ref{lem:pi1}}
\label{sub:pi1}
Define the event
\begin{align}
  \label{eq:lower_bound_def_x_F1}
  F_1 \defn \Bigg\{x \in \real^\dims \bigg{|}& \max_{i\in[\dims]} \sqrt{\smoothparam}\abss{\coord{x}{i}}< 4\sqrt{\log(8\dims)}, \notag \\
     &\smoothparam\vecnorm{x}{2}^2 < \dims +  \dims^{1-4\loExpo} + 5 \sqrt{\dims},\notag \\
     &\sum_{i=1}^\dims - \cos(\dims^\loExpo\smoothparam^{\frac12} \coord{x}{i}) < -\frac{1}{4}\dims^{1-2\loExpo} + \frac12 \dims^{1-4\loExpo} + 2 \dims^{\frac12} ,\notag \\ &\abss{\sum_{i=1}^\dims - \cos(2\dims^\loExpo\smoothparam^{\frac12} \coord{x}{i}) +\frac{1}{16}\dims^{1-2\loExpo}} \leq \frac{1}{8} \dims^{1-4\loExpo} + 2 \dims^{\frac12} ,\notag \\  &\abss{\sum_{i=1}^\dims \smoothparam^{\frac12}\coord{x}{i} \sin(\dims^\loExpo\smoothparam^{\frac12} \coord{x}{i})} < \frac12\dims^{1-4\loExpo} + 2 \dims^{\frac12} \Bigg\}
\end{align}
Bounding the measure of $F_1$ under $\target_1$ requires concentration inequalities for $\target_1$. Its proof is deferred to Lemma~\ref{lem:33}, where we proved $\target_1(F_1) > 1/6$.

Let $\loExpo = 1/4 - \deltas$. For any $x \in \real^\dims$, denote the negative log density by $\targetf(x) = \frac{\smoothparam}{2}\sum_{i=1}^{\dims} x_{[i]}^2 - \frac{1}{2\dims^{2\loExpo}} \sum_{i=1}^d \cos(d^{\loExpo}\smoothparam^{\frac12} x_{[i]}) $ and its cosine part by $\targetf_P(x)=- \frac{1}{2\dims^{2\loExpo}} \sum_{i=1}^d \cos(d^{\loExpo}\smoothparam^{\frac12} x_{[i]})$. The gradient of $\targetf$ at $x$ is
\begin{align*}
  \gradf(x) &= \smoothparam x +\gradf_P(x)=\smoothparam x + \frac{\smoothparam^{\frac12}}{2\dims^{\loExpo}} \bmat{\sin(d^{\loExpo}\smoothparam^{\frac12} x_{[1]}) \\ \vdots \\ \sin(d^{\loExpo}\smoothparam^{\frac12} x_{[\dims]})}.
\end{align*}
For any $x \in F_1$ and $y \in \real^\dims$, we can express the quantity of interest as follows
\begin{align*}
  &\quad \frac{\target_1(y)\propkernel_1(y, x)}{\target_1(x) \propkernel_1(x, y)} \\
  &= \exp\brackets{\targetf(x) - \targetf(y) - \frac{1}{4h} \vecnorm{x-y+h\gradf(y)}{2}^2 + \frac{1}{4h} \vecnorm{y-x+h\gradf(x)}{2}^2 }.
\end{align*}
Define $g\defn y - x+h\gradf(x)$. We further decompose the quantity of interest by isolating the linear, quadratic terms on $g$ and the cosine part $\targetf_P$. We have
\begin{align*}
  &\quad \targetf(x) - \targetf(y) \\
  &= \frac{\smoothparam}{2} \parenth{\vecnorm{x}{2}^2 - \vecnorm{y}{2}^2} + \targetf_P(x) - \targetf_P(y) \\
  &= \frac{\smoothparam}{2} \parenth{\vecnorm{x}{2}^2 - \vecnorm{x - h\gradf(x) + g}{2}^2} + \targetf_P(x) - \targetf_P(y) \\
  &= -\frac{\smoothparam}{2} \parenth{\parenth{- 2 \smoothparam h + \smoothparam^2 h^2}\vecnorm{x}{2}^2 + 2(1-\smoothparam h) \angles{x, g} + \vecnorm{g}{2}^2 }\\
  &\quad + \smoothparam \angles{h\gradf_P(x), (1-\smoothparam h)x+g} - \frac{\smoothparam h^2}{2} \vecnorm{\gradf_p(x)}{2}^2 + \targetf_P(x) - \targetf_P(y).
\end{align*}
And we have
\begin{align*}
  &\quad - \frac{1}{4h} \vecnorm{x-y+h\gradf(y)}{2}^2 + \frac{1}{4h} \vecnorm{y-x+h\gradf(x)}{2}^2 \\
  &= - \frac{1}{4h} \vecnorm{ h \gradf(x) + h\gradf(y) - g}{2}^2 + \frac{1}{4h} \vecnorm{g}{2}^2 \\
  &=  \frac{1}{2} \angles{g, \gradf(x) + \gradf(y)} - \frac{h}{4} \vecnorm{\gradf(x) + \gradf(y)}{2}^2 \\
  &= \frac{\smoothparam}{2} \vecnorm{g}{2}^2 + \frac{1}{2} \angles{g, \smoothparam(2-\smoothparam h ) x + (1-\smoothparam h)\gradf_P(x)  + \gradf_P(y) } - \frac{h}{4} \vecnorm{\gradf(x) + \gradf(y)}{2}^2 \\
  &= \frac{\smoothparam}{2} \vecnorm{g}{2}^2 + \frac{1}{2} \angles{g, \smoothparam(2-\smoothparam h ) x + (1-\smoothparam h) \gradf_P(x)  + \gradf_P(y)  } \\
  &\quad - \frac{h\smoothparam^2}{4} \vecnorm{(2-\smoothparam h) x + g}{2}^2 - \frac{h\smoothparam}{2} \angles{(2-\smoothparam h) x + g, (1-\smoothparam h) \gradf_P(x)+\gradf_P(y)} \\
  &\quad - \frac{h}{4} \vecnorm{(1-\smoothparam h) \gradf_P(x)+\gradf_P(y)}{2}^2.
\end{align*}
Rearranging the terms in the above two equations, we have
\begin{align}\label{eq:7terms}
  &\quad \targetf(x) - \targetf(y) - \frac{1}{4h} \vecnorm{x-y+h\gradf(y)}{2}^2 + \frac{1}{4h} \vecnorm{y-x+h\gradf(x)}{2}^2 \notag\\
  &= \underbrace{\targetf_P(x) - \targetf_P(y)}_{\Deltapart_1} + \underbrace{\parenth{\frac{\smoothparam^3 h^2}{2} - \frac{\smoothparam^4 h^3}{4}} \vecnorm{x}{2}^2- \frac{\smoothparam^2 h}{4} \vecnorm{g}{2}^2 }_{\Deltapart_2} +  \underbrace{\parenth{\frac{\smoothparam^3 h^2}{2} - \frac{\smoothparam^2 h}{2} } \angles{x, g}}_{\Deltapart_3} \notag\\
  &\quad +\underbrace{ \angles{\gradf_P(x), \parenth{\frac{1+\smoothparam^2 h^2}{2} }g }  - \angles{\gradf_P(y), \frac{1 + \smoothparam h }{2} g}}_{\Deltapart_4}\notag\\
  &\quad +\underbrace{\angles{\gradf_P(x), \parenth{\frac{\smoothparam^2 h^2}{2} -  \frac{\smoothparam^3 h^3}{2} }x } - \angles{\gradf_P(y), \frac{\smoothparam h }{2} \parenth{2 - \smoothparam h}  x}}_{\Deltapart_5} \notag\\
  &\quad \underbrace{- \smoothparam h^2 \vecnorm{\gradf_P(x)}{2}^2}_{\Deltapart_6} \underbrace{- \frac{h}{4} \vecnorm{(1-\smoothparam h) \gradf_P(x)+\gradf_P(y)}{2}^2}_{\Deltapart_7}
\end{align}

We bound each of these seven terms in Lemma~\ref{lem:7terms} under the condition $x \in F_1$, $y = x - h\gradf(x) + g$ and $g\sim \Normal(0, 2h\Ind_\dims)$. According to Lemma~\ref{lem:7terms}, for $h\geq 1/\parenth{\smoothparam\dims^{1/2-3\theta}}$, any fixed $x\in F_1$ and $ y \sim \Normal(x - h\gradf(x), 2h\Ind_\dims)$ we have
\begin{align*}
  \sum_{i=1}^7\Deltapart_i (x, y)\leq -\dims^{4\deltas}/32
\end{align*} with probability at least $1-10\exp\parenth{-\dims^{4\deltas}/16384}$. Finally, Lemma~\ref{lem:pi1} follows by setting $G_x$ to be $\braces{ y \in \real^\dims \mid \sum_{i=1}^7\Deltapart_i(x, y) \leq -\dims^{4\deltas}/32}$.

\subsubsection{Lemmas on concentration properties of the perturbed Gaussian distribution}\label{sec:perturbed_properties}

In this section we present three lemmas (Lemma~\ref{lem:31},~\ref{lem:32} and~\ref{lem:33}) to characterize several properties of the perturbed Gaussian distribution $\target_1$ in Equation~\eqref{eq:perturbed_gaussian}. Lemma~\ref{lem:33} directly implies that the set $F_1$ in Equation~\eqref{eq:lower_bound_def_x_F1} satisfies $\target_1(F_1)>1/6$. Additionally, these lemmas are useful to complete the proof of Lemma~\ref{lem:7terms} in Appendix~\ref{sec:7terms}.

The following lemma establishes bounds on the expectations of cosine terms. It is adapted from Lemma 31 in \cite{chewi2021optimal}.

\begin{lemma}\label{lem:31}
Fix $\xi\sim\Normal(0,1)$ and constants $a,b\in \real$. Then we have
\begin{enumerate}[label=(\alph*)]
    \item $\abss{\Exs[\cos(a+b\xi)]}\leq\exp\parenth{-\frac{b^2}{2}}$.
    \item $\abss{\Exs[\xi \cos(a+b\xi)]}\leq\abss{b}\exp\parenth{-\frac{b^2}{2}}$
    \item $\abss{\Exs[\xi^2\cos(a+b\xi)]}\leq\abss{b^2-1}\exp\parenth{-\frac{b^2}{2}}$
\end{enumerate}
\end{lemma}

\begin{proof}
Let $\text{Re}(\cdot)$ denote the real part of a complex number. Since $\Exs[e^{it\xi}]=e^{-\frac12t^2}$, for any integer $\ell\geq0$, we have
\begin{align*}
    \Exs[\xi^{\ell}\cos(a+b\xi)]&=\Exs[\text{Re}(\xi^{\ell}e^{i(a+b\xi})]\\
    &=\text{Re}\parenth{e^{ia}\Exs\brackets{\xi^{\ell}e^{ib{\xi}}}}\\
    &=\text{Re}\parenth{e^{ia}i^{-\ell}\Exs\brackets{\dfrac{\text{d}^\ell}{\text{dt}^\ell}e^{it\xi}\Big|_{t=b}}}\\
    &=\text{Re}\parenth{e^{ia}i^{-\ell}\dfrac{\text{d}^\ell}{\text{dt}^\ell}e^{-\frac{t^2}{2}}\Big|_{t=b}}
\end{align*}
Three results now follow from taking $\ell=0,1,2$.
\end{proof}

Next we analyze several expectations under the perturbed Gaussian distribution using Lemma~\ref{lem:31}. The first three statements in Lemma~\ref{lem:31} are adapted from Lemma 32 in \cite{chewi2021optimal}.
\begin{lemma}\label{lem:32}
Let $\loExpo\in(1/5, 1/4), \dims\geq2048$ and $\smoothparam > 0$. Consider the one-dimensional distribution $\target(x) =\frac{1}{Z} \exp\parenth{ -\frac{\smoothparam}{2}x^2+\frac{1}{2\dims^{2\loExpo}}\cos(d^{\loExpo}\smoothparam^{\frac12}x)}$, where the normalization constant $Z = \int_\real  \exp\parenth{ -\frac{\smoothparam}{2}x^2+\frac{1}{2\dims^{2\loExpo}}\cos(d^{\loExpo}\smoothparam^{\frac12}x)}dx$. We have
\begin{enumerate}[label=(\alph*)]
    \item  $\abss{\frac{1}{Z}\sqrt{\frac{2\pi}{\smoothparam}}-1}\leq \frac12\dims^{-4\loExpo}+ \dims^{-2\loExpo} \exp\parenth{-\frac12\dims^{2\loExpo}}$.
    \item $\Exs_{x\sim\pi}[\smoothparam x^2]-1\leq \dims^{-4\loExpo}$.
    \item  $
    \abss{\Exs_{x\sim \target}[\cos(\dims^\loExpo\smoothparam^{\frac12} x)]-\frac{1}{4}\dims^{-2\loExpo}}\leq \frac12 \dims^{-4\loExpo}$.
    \item $\abss{\Exs_{x \sim \target} \brackets{\smoothparam^{\frac12}x \sin\parenth{\dims^\loExpo \smoothparam^{\frac12} x}}}\leq\frac12\dims^{-4\loExpo}$.
\end{enumerate}
\end{lemma}

\begin{proof}
\begin{enumerate}[label=(\alph*)]
    \item The normalizing constant $Z$ is 
    \begin{align*}
        Z&=\int_{\real}\exp\parenth{-\frac{\smoothparam}{2}x^2+\frac{1}{2\dims^{2\loExpo}}\cos(d^{\loExpo}\smoothparam^{1/2}x)}dx\\
        &\overset{(i)}=\frac{1}{\sqrt{\smoothparam}}\int_{\real}\exp\parenth{-\frac{1}{2}\xi^2+\frac{1}{2\dims^{2\loExpo}}\cos(d^{\loExpo}\xi)}d\xi\\
        &=\sqrt{\frac{2\pi}{\smoothparam}}\Exs_{\xi\sim\Normal(0,1)}\brackets{\exp\parenth{\frac{1}{2\dims^{2\loExpo}}\cos(\dims^{\loExpo}\xi)}}\\
        &\overset{(ii)}{=}\sqrt{\frac{2\pi}{\smoothparam}}\Exs_{\xi\sim\Normal(0,1)}\brackets{1+\frac{1}{2\dims^{2\loExpo}}\cos(\dims^{\loExpo}\xi)+\Remain_1}\\
        &=\sqrt{\frac{2\pi}{\smoothparam}}\parenth{1+\frac{1}{2\dims^{2\loExpo}}\Exs_{\xi\sim\Normal(0,1)}\brackets{\cos(\dims^{\loExpo}\xi)}+\Exs_{\xi\sim\Normal(0,1)}[\Remain_1]}
    \end{align*}
    where step $(i)$ uses the transformation $\xi=\sqrt{\smoothparam}x$ and step $(ii)$ introduces the remainder term \mbox{$\Remain_1 = \exp\parenth{\frac{1}{2\dims^{2\loExpo}}\cos(\dims^{\loExpo}\xi)} -1-\frac{1}{2\dims^{2\loExpo}}\cos(\dims^{\loExpo}\xi)$}. By Lemma~\ref{lem:31}, the second term in the last line satisfies
    \begin{align*}
        \abss{\frac{1}{2\dims^{2\loExpo}}\Exs_{\xi\sim\Normal(0,1)}\brackets{\cos(\dims^{\loExpo}\xi)}}\leq \frac{\dims^{-2\loExpo} \exp\parenth{-\frac{\dims^{2\loExpo}}{2}}}{2}
    \end{align*}
    Since $1+x\leq \exp(x)\leq 1+x+x^2$ for $x\in[-1,1]$, we have $0\leq R_1\leq 1/(4d^{4\loExpo})$. We obtain \begin{align*}
        \abss{\sqrt{\frac{\smoothparam}{2\pi}}Z-1}\leq \frac14\dims^{-4\loExpo}+\frac12 \dims^{-2\loExpo} \exp\parenth{-\frac12\dims^{2\loExpo}}
    \end{align*}
    The result now follows by $1/(1-x)\leq1+2x$ and $1/(1+x)\geq1-2x$ when $x\in(0,1/2)$.
    
    \item We write the expectation as
    \begin{align*}
        \Exs_{x\sim\pi}[\smoothparam x^2]&=\frac{1}{Z}\int_\real \smoothparam x^2 \exp\parenth{ -\frac{\smoothparam}{2}x^2+\frac{1}{2\dims^{2\loExpo}}\cos(d^{\loExpo}\smoothparam^{1/2}x)}dx\\
        &\overset{(i)}=\frac{1}{Z}\sqrt{\frac{2\pi}{\smoothparam}}\Exs_{\xi\sim\Normal(0,1)}\brackets{\xi^2\exp\parenth{\frac{1}{2\dims^{2\loExpo}}\cos(\dims^{\loExpo}\xi)}}\\
        &\overset{(ii)}=\frac{1}{Z}\sqrt{\frac{2\pi}{\smoothparam}}\Exs_{\xi\sim\Normal(0,1)}\brackets{\xi^2\parenth{1+\frac{1}{2\dims^{2\loExpo}}\cos(\dims^\loExpo\xi)+\Remain_2}}\\
        &=\frac{1}{Z}\sqrt{\frac{2\pi}{\smoothparam}}\parenth{1+\frac{1}{2\dims^{2\loExpo}}\Exs_{\xi\sim\Normal(0,1)}\brackets{\xi^2\cos(\dims^\loExpo\xi)}+\Exs_{\xi\sim\Normal(0,1)}[R_2]},
    \end{align*}
    where step $(i)$ uses the transformation $\xi=\sqrt{\smoothparam}x$ and step $(ii)$ introduces $\Remain_2$ as the remainder term. Lemma~\ref{lem:31} guarantees that the second term satisfies 
    
    $\Exs_{\xi\sim\Normal(0,1)}[\xi^2\cos(\dims^\loExpo \xi)]/(2\dims^{2\loExpo})\leq  \exp\parenth{- \dims^{2\loExpo}/2}/2$. The remainder term satisfies 
    
    $ \Exs_{\xi\sim\Normal(0,1)}[\Remain_2]\leq \dims^{-4\loExpo}/4$ because $\exp(x)\leq 1+x+x^2$ for $x\in[-1,1]$. Using the above estimates and part (a), we obtain
    \begin{align*}
        \Exs_{x\sim\pi}\brackets{\smoothparam x^2}
        &\leq\parenth{1+\frac12\dims^{-4\loExpo}+ \dims^{-2\loExpo} \exp\parenth{-\frac12\dims^{2\loExpo}} }\parenth{1+ \frac12 \exp\parenth{-\frac12 \dims^{2\loExpo}} + \frac1{4}d^{-4\loExpo} }\\
        &\leq\parenth{1+\frac12\dims^{-4\loExpo}+ \frac1{16}\dims^{-6\loExpo}}\parenth{1+ \frac1{32} \dims^{-4\loExpo} + \frac14d^{-4\loExpo} } \\
        &\leq 1 + \dims^{-4\loExpo},
    \end{align*}
    where we use $\exp(x/2) \geq 16x^2$ for $x \geq 20$ with $x=\dims^{2\loExpo}$.
    
    \item 
Using a similar strategy as above, we obtain
\begin{align*}
    &\quad\Exs_{x\sim \target}[\cos(\dims^\loExpo\smoothparam^{1/2} x)]\\
    &=\frac{1}{Z}\int_{\real}\cos(\dims^\loExpo\smoothparam^{1/2} x) \exp\parenth{ -\frac{\smoothparam}{2}x^2+\frac{1}{2\dims^{2\loExpo}}\cos(d^{\loExpo}\smoothparam^{1/2}x)}dx\\
    &\overset{(i)}{=}\frac{1}{Z}\sqrt{\frac{2\pi}{\smoothparam}}\Exs_{\xi\sim\Normal(0,1)}\brackets{\cos(\dims^\loExpo\xi)\exp\parenth{\frac{1}{2\dims^{2\loExpo}}\cos(\dims^\loExpo \xi)}}\\
    &\overset{(ii)}{=}\frac{1}{Z}\sqrt{\frac{2\pi}{\smoothparam}}\Exs_{\xi\sim\Normal(0,1)}\brackets{\cos(\dims^{\loExpo}\xi)+\frac{1}{2\dims^{2\loExpo}}\cos^2(\dims^{\loExpo}\xi)+\Remain_3}\\
    &=\frac{1}{Z}\sqrt{\frac{2\pi}{\smoothparam}}\Big(\Exs_{\xi\sim\Normal(0,1)}\brackets{\cos(\dims^{\loExpo}\xi)}+\frac{1}{4\dims^{2\loExpo}}+\frac{1}{4\dims^{2\loExpo}}\Exs_{\xi\sim\Normal(0,1)}\brackets{\cos(2\dims^\loExpo\xi)}\\&\hspace{4em}+\Exs_{\xi\sim\Normal(0,1)} [\Remain_3]\Big),
\end{align*}
where step $(i)$ uses the transformation $\xi=\sqrt{\smoothparam}x$ and step $(ii)$ introduces $\Remain_3$ as the remainder term. We have $\abss{\Exs_{\xi\sim\Normal(0,1)}[\cos(\dims^\loExpo\xi)}\leq \exp(- \dims^{2\loExpo}/2)$ and $\abss{\Exs_{\xi\sim\Normal(0,1)}[\cos(2\dims^\loExpo\xi)]}$ $\leq \exp(-2\dims^{2\loExpo})$ by Lemma~\ref{lem:31}. The remainder term satisfies $\abss{\Exs_{\xi\sim\Normal(0,1)}[\Remain_3]}\leq \dims^{-4\loExpo}/4$ again because $1+x\leq \exp(x)\leq 1+x+x^2$ for $x\in[-1,1]$. Plugging in these estimates and using part (a), we obtain
\begin{align*}
    &\quad \abss{\Exs_{x\sim \target}[\cos(\dims^\loExpo \smoothparam^{1/2} x)]  - \frac14 \dims^{-2\loExpo}} \\
    &\leq \parenth{1+\frac12\dims^{-4\loExpo}+ \dims^{-2\loExpo} \exp\parenth{-\frac12\dims^{2\loExpo}} }\parenth{\frac14 \dims^{-2\loExpo} + \frac14 \dims^{-4\loExpo} + 2 \exp(-\frac12 \dims^{2\loExpo})} - \frac{1}{4}\dims^{-2\loExpo}\\
    &\leq \parenth{1+\frac12\dims^{-4\loExpo}+ \frac1{16}\dims^{-6\loExpo} }\parenth{\frac14 \dims^{-2\loExpo} + \frac14 \dims^{-4\loExpo} + \frac1{8}\dims^{-4\loExpo}} - \frac{1}{4}\dims^{-2\loExpo}\\
    &\leq \frac12\dims^{-4\loExpo}.
\end{align*}

  \item We have
  \begin{align*}
    &\quad \Exs_{x \sim \target} \brackets{\smoothparam^{\frac12}x \sin\parenth{\dims^\loExpo \smoothparam^{\frac12} x}} \\
    &= \frac{1}{Z} \int_\real \smoothparam^{\frac12}x \sin\parenth{\dims^\loExpo \smoothparam^{\frac12}x } \exp\parenth{-\frac{\smoothparam}{2}x^2 + \frac{1}{2\dims^{2\loExpo}} \cos\parenth{\dims^\loExpo \smoothparam^{\frac12} x} } dx\\
    &= \frac{1}{Z} \sqrt{\frac{2\pi}{\smoothparam}}  \Exs_{\xi \sim \Normal(0, 1)} \brackets{\xi \cos(\dims^\loExpo \xi) \exp\parenth{\frac{1}{2 \dims^{2\loExpo}} \cos(\dims^\loExpo \xi) } } \\
    &= \frac{1}{Z} \sqrt{\frac{2\pi}{\smoothparam}}  \Exs_{\xi \sim \Normal(0, 1)} \brackets{\xi \cos(\dims^\loExpo \xi) + \frac{\xi}{2\dims^{2\loExpo}} \cos(\dims^\loExpo \xi)^2 + \xi R_3  } \\
    &= \frac{1}{Z} \sqrt{\frac{2\pi}{\smoothparam}}  \parenth{\Exs_{\xi\sim\Normal(0,1)}\brackets{\xi \cos(\dims^\loExpo \xi)} + \Exs_{\xi\sim\Normal(0,1)}\brackets{\frac{\xi}{4\dims^{2\loExpo}} \cos(2\dims^\loExpo \xi)} + \Exs_{\xi\sim\Normal(0,1)}\brackets{\xi R_3}  }.
  \end{align*}
  Applying Lemma~\ref{lem:31} again, we obtain
  \begin{align*}
    \abss{\Exs_{\xi\sim\Normal(0,1)}\brackets{\xi \cos(\dims^\loExpo \xi)}} &\leq \dims^\loExpo \exp\parenth{-\frac{\dims^{2\loExpo}}{2}} \\
    \abss{\Exs_{\xi\sim\Normal(0,1)}\brackets{\xi \cos(2\dims^\loExpo \xi)}} &\leq 2\dims^\loExpo \exp\parenth{-2\dims^{2\loExpo}} \\
    \abss{\Exs_{\xi\sim\Normal(0,1)}[\xi R_3]} &\leq \Exs_{\xi\sim\Normal(0,1)}[\xi^2 R_3^2]^{1/2} \leq \frac{\dims^{-4\loExpo}}{4}
  \end{align*}
  Plugging these estimates and using part (a), we obtain
  \begin{small}
  \begin{align*}
      &\quad\abss{\Exs_{x \sim \target} \brackets{\smoothparam^{\frac12}x \sin\parenth{\dims^\loExpo \smoothparam^{\frac12} x}}}\\
      &\leq\parenth{1+\frac12\dims^{-4\loExpo}+ \dims^{-2\loExpo} \exp\parenth{-\frac12\dims^{2\loExpo}} }\parenth{\dims^{\loExpo}\exp\parenth{-\frac{1}{2}\dims^{2\loExpo}}+2\dims^\loExpo\exp\parenth{-2\dims^{2\loExpo}}+\frac14\dims^{-4\loExpo}}\\
      &\leq \parenth{1+\frac12 \dims^{-4\loExpo}+\frac{1}{6}\dims^{-7\loExpo}}\parenth{\frac{1}{6}\dims^{-4\loExpo}+\frac{1}{96}\dims^{-4\loExpo}+\frac14\dims^{-4\loExpo}}\\
      &\leq\frac12\dims^{-4\loExpo},
  \end{align*}
  \end{small}
  where the last two steps use $\exp(x^2/2)\geq 6x^5$ for $x\geq 5$ with $x=\dims^{\loExpo}$.
\end{enumerate}
\end{proof}

Finally we provide constant probability bounds for each term in the definition of $F_1$ in Equation~\eqref{eq:lower_bound_def_x_F1} using Lemma~\ref{lem:32}.
\begin{lemma}\label{lem:33}
    Fix $\loExpo\in(\frac15,\frac14), d\geq2048$ and $\smoothparam>0$. Assume  that the $\dims$-dimensional random variable $x$ follows the distribution $\target(x) \propto \exp\parenth{ \frac{\smoothparam}{2}\sum_{i=1}^\dims\coord{x}{i}^2-\frac{1}{2\dims^{2\loExpo}}\sum_{i=1}^\dims\cos(\dims^{\loExpo}\smoothparam^{\frac12}\coord{x}{i})}$, we have
    \begin{enumerate}[label=(\alph*)]
        \item
        \begin{align*}
            \Prob_{x \sim \target} \brackets{\max_{i} \sqrt{\smoothparam}\abss{\coord{x}{i}}\geq 4\sqrt{\log(8\dims)}} \leq \frac{1}{4\dims}
        \end{align*} 
        \item
        \begin{align*}
          \Prob_{x \sim \target}\brackets{\smoothparam\vecnorm{x}{2}^2 \geq \dims +  \dims^{1-4\loExpo} + 5 \sqrt{\dims}} \leq 0.14
        \end{align*}
        \item
        \begin{align*}
          \Prob_{x\sim \target}\brackets{ \sum_{i=1}^\dims - \cos(\dims^\loExpo\smoothparam^{1/2} \coord{x}{i}) \geq -\frac{1}{4}\dims^{1-2\loExpo} + \frac12 \dims^{1-4\loExpo} + 2 \dims^{\frac12} } \leq 0.14
        \end{align*}
        \item
        \begin{align*}
          \Prob_{x\sim \target}\brackets{ \abss{\sum_{i=1}^\dims - \cos(2\dims^\loExpo\smoothparam^{\frac12} \coord{x}{i})  +\frac{1}{16}\dims^{1-2\loExpo}}\geq  \frac{1}{8} \dims^{1-4\loExpo} + 2 \dims^{\frac12} } \leq 0.28
        \end{align*}
        \item
        \begin{align*}
            \Prob_{x\sim \target}\brackets{ \abss{\sum_{i=1}^\dims \smoothparam^{\frac12}\coord{x}{i} \sin(\dims^\loExpo\smoothparam^{\frac12} \coord{x}{i})} \geq \frac12\dims^{1-4\loExpo} + 2 \dims^{\frac12} } \leq 0.26
        \end{align*}
    \end{enumerate}
\end{lemma}

\begin{proof}
\begin{enumerate}[label=(\alph*)]
    \item See Lemma 33 in \cite{chewi2021optimal}.
    \item By Lemma~\ref{lem:32}-(b), $\Exs_{x \sim \target}[L\vecnorm{x}{2}^2]\leq \dims+ \dims^{1-4\loExpo}$. Note that $\target$ is $\smoothparam/2$-strongly log-concave, we deduce that for a random variable $x$ drawn from $\target$, $x - \Exs[x]$ is a sub-Gaussian random vector with parameter $\sqrt{2/\smoothparam}$ (see Proof of Lemma 1 in~\cite{dwivedi2018log}). Since $x \mapsto \sqrt{\smoothparam} \vecnorm{x}{2}$ is $\sqrt{\smoothparam}$ Lipschitz, $\sqrt{\smoothparam}\vecnorm{x}{2} - \Exs[\sqrt{\smoothparam}\vecnorm{x}{2}]$ is a sub-Gaussian with parameter $\sqrt{2}$. Applying the Chernoff bound (see for example, Equation 2.9 in~\cite{wainwright2019high}), we have
    \begin{align*}
      \Prob_{x \sim \target}\brackets{\sqrt{\smoothparam}\vecnorm{x}{2} \geq \Exs[\sqrt{\smoothparam}\vecnorm{x}{2}] + t} \leq \exp\parenth{-\frac{t^2}{4}}.
    \end{align*}
    Take $t = 2$ and use $\Exs[\sqrt{\smoothparam}\vecnorm{x}{2}] \leq \Exs[\smoothparam\vecnorm{x}{2}^2]^{\frac12}$, we obtain
    \begin{align*}
      \Prob_{x \sim \target}\brackets{\smoothparam\vecnorm{x}{2}^2 \geq \dims + \dims^{1-4\loExpo} + 5 \sqrt{\dims}} \leq \exp(-2)<0.14
    \end{align*}
    \item By Lemma~\ref{lem:32}-(c), $\Exs_{x \sim \target} \brackets{-\sum_{i=1}^\dims \cos(\dims^\loExpo\smoothparam^{1/2} \coord{x}{i})} \leq -\dims^{1-2\loExpo}/4 + \dims^{1-4\loExpo}/2$. Each cosine term is bounded in $[-1, 1]$. Applying Hoeffding bound (see for example, Proposition 2.1 in~\cite{wainwright2019high}), we have
    \begin{align*}
      \Prob_{x \sim \target}\brackets{-\sum_{i=1}^\dims \cos(\dims^\loExpo\smoothparam^{1/2} \coord{x}{i}) \geq \Exs\brackets{-\sum_{i=1}^\dims \cos(\dims^\loExpo\smoothparam^{1/2} \coord{x}{i})} + t\dims } \leq \exp\parenth{-\frac{ t^2\dims}{2}}.
    \end{align*}
    Take $t = 2 \dims^{-1/2}$, we obtain
    \begin{align*}
      \Prob_{x \sim \target}\brackets{-\sum_{i=1}^\dims \cos(\dims^\loExpo\smoothparam^{1/2} \coord{x}{i}) \geq -\frac{1}{4}\dims^{1-2\loExpo} + \frac12 \dims^{1-4\loExpo} + 2 \dims^{\frac12}} \leq \exp(-2)<0.14
    \end{align*}
    \item 
    The proof is the same as (c) using a two-sided Hoeffding bound.
    \item By Lemma~\ref{lem:32}-(d), $\abss{\Exs_{x \sim \target} \brackets{\sum_{i=1}^\dims\smoothparam^{\frac12}\coord{x}{i} \sin\parenth{\dims^\loExpo \smoothparam^{\frac12} \coord{x}{i}}}}\leq\dims^{1-4\loExpo}/2$.
    We also have
    \begin{align*}
        \Var\brackets{\sum_{i=1}^\dims\smoothparam^{\frac12}\coord{x}{i} \sin\parenth{\dims^\loExpo \smoothparam^{\frac12} \coord{x}{i}}}&\leq\smoothparam\sum_{i=1}^\dims\Exs\coord{x}{i}^2\leq\dims+\dims^{1-4\loExpo}
    \end{align*}
    Applying Chebyshev's inequality, we obtain
    \begin{align*}
        \Prob_{x \sim \target}\brackets{\sum_{i=1}^\dims\smoothparam^{\frac12}\coord{x}{i} \sin\parenth{\dims^\loExpo \smoothparam^{\frac12} \coord{x}{i}}\geq\frac12\dims^{1-4\loExpo}+2\dims^{\frac12}}\leq \frac{\dims+
        \dims^{1-4\loExpo}}{4\dims}&<0.26
    \end{align*}
\end{enumerate}
\end{proof}

\subsubsection{High probability upper bound on the acceptance rate}\label{sec:7terms}
In this section we bound each of the seven terms in Equation~\eqref{eq:7terms} with high probability, which implies an high probability upper bound on the acceptance rate for MALA applied $\target_1$. Before we do so, we first state the following two forms of the Bernstein's inequality (see Propostion 2.14 in~\cite{wainwright2019high} and Lemma 14.9 in~\cite{buhlmann2011statistics}) which we frequently use in the proof of the following lemma.

Let $X_1, \ldots, X_d$ be i.i.d. random variables satisfying $\abss{X_i-\Exs[X_i]}\leq K$. Then for any $\epsilon \geq 0$,
\begin{align}
  \label{eq:Bernstein_ineq}
  \Prob\brackets{\sum_{i=1}^d \parenth{X_i - \Exs[X]} \geq \epsilon} \leq \exp\parenth{-\frac{\epsilon^2}{2\parenth{\sum_{i=1}^d\Var[X_i]+\frac13K\epsilon}}}
\end{align}

Let $X_1, \ldots, X_n $ be i.i.d. random variables satisfying $\Exs[\abss{X-\Exs[X]}^\ell] \leq \ell! K^{\ell-2} /2, \ \forall \ell\geq2$. Then for any $\epsilon \geq 0$,
\begin{align}
  \label{eq:Bernstein_ineq0}
  \Prob\brackets{\sum_{i=1}^d \parenth{X_i - \Exs[X]} \geq d(K\epsilon+\sqrt{2\epsilon})} \leq \exp\parenth{-d\epsilon}.
\end{align}

\begin{lemma}\label{lem:7terms}
    Assume 
    \begin{align}\label{eq:condition_for_d}
    \deltas\in(0,\frac1{20}), \smoothparam>0, \dims^{\deltas}\geq \max\braces{\frac{1}{2}\log d+6, 10},  h\geq \frac{1}{\smoothparam \dims^{\frac12-3\deltas}}.
    \end{align}
    Given any fixed $x\in F_1$ defined in Equation~\eqref{eq:lower_bound_def_x_F1}. Let $\Deltapart_i(x, y)\ (1\leq i\leq 7)$ be the decomposed terms from the exponent of the acceptance rate in Equation~\eqref{eq:7terms}, in which $y\sim\Normal(x-h\gradf(x), 2h\Ind_\dims)$ and $g=y-x+h\gradf(x)$. We have
    \begin{align*}
        \sum_{i=1}^7 \Deltapart_i(x, y) \leq -\frac{1}{32}\dims^{4\deltas}
    \end{align*}
     with probability at least $1-10\exp\parenth{-\dims^{4\deltas}/16384}$. 
\end{lemma}

\begin{proof}
Recall that $\loExpo=1/4-\deltas$. All high probability bounds in the proof are stated with respect to $\xi\defn\parenth{y-x+h\gradf(x)}/\sqrt{2h}\sim\Normal(0,\Ind_d)$. Now we bound each $\Deltapart_i$ separately.

\begin{enumerate}[label={(\arabic*)}]
  \item We write $\Deltapart_1$ as
  \begin{align*}
    \Deltapart_1=\targetf_P(x) - \targetf_P(y) = \sum_{i=1}^\dims \underbrace{-\frac{1}{2 \dims^{2\loExpo}}\cos(\dims^{\loExpo}\smoothparam^{\frac12}\coord{x}{i})}_{\Deltapart_{1,1,i}} + \underbrace{\frac{1}{2 \dims^{2\loExpo}}\cos(\dims^{\loExpo}\smoothparam^{\frac12}\coord{y}{i})}_{\Deltapart_{1,2,i}}.
  \end{align*}
  By the definition of $F_1$ in Equation~\eqref{eq:lower_bound_def_x_F1}, we have
  \begin{align}\label{eq:delta_11i}
    \sum_{i=1}^\dims \Deltapart_{1,1,i} &\leq -\frac{1}{8}\dims^{1-4\loExpo} + \frac14 \dims^{1-6\loExpo} +  \dims^{\frac12 - 2\loExpo}\notag\\
    &\leq-\frac{1}{8}\dims^{1-4\loExpo}+\frac{1}{64}\dims^{1-4\loExpo}
  \end{align}
  where we use the assumption \eqref{eq:condition_for_d} in the last line. Since $\coord{y}{i}=\coord{x}{i}-h\coord{\gradf(x)}{i}+\sqrt{2h}\coord{\xi}{i}$, we have
  \begin{align*}
    \Exs\brackets{\Deltapart_{1,2,i}}
    &\overset{(i)}{\leq} \frac1 {2\dims^{2\loExpo}}\exp\parenth{-Lh\dims^{2\loExpo}}\overset{(ii)}{\leq}\frac{1}{128}\dims^{-4\loExpo}\\
    \Var\brackets{\Deltapart_{1,2,i}}&\leq\Exs\brackets{\Deltapart_{1,2,i}^2} \overset{(iii)}{\leq} \frac{1}{4\dims^{4\loExpo}}\\
    \abss{\Deltapart_{1,2,i}}&\overset{(iv)}{\leq}\frac{1}{2\dims^{2\loExpo}}
  \end{align*}
 Inequality (i) follows from Lemma~\ref{lem:31}-(a). Inequality (ii) follows from the assumption~\eqref{eq:condition_for_d}. (iii) and (iv) simply bounds cousine by $1$.

With the above bounds for individual terms, applying Bernstein's inequality~\eqref{eq:Bernstein_ineq} with $\epsilon = \dims^{1-4\loExpo}/128$, we obtain
\begin{align}\label{eq:delta_12i}
    \Prob\brackets{\sum_{i=1}^\dims \Deltapart_{1,2,i}  \geq \frac{1}{64}\dims^{1 - 4\loExpo}} \leq \exp\parenth{-\frac{\dims^{1-4\loExpo}}{16384} }
\end{align}
Combine Equation~\eqref{eq:delta_11i} and \eqref{eq:delta_12i}, we have that for fixed $x \in F_1$, the following bound
\begin{align}\label{eq:bound_delta1}
    \Deltapart_1\leq-\frac{3}{32}\dims^{1-4\loExpo}.
\end{align}
holds with probability at least $1-\exp(-\dims^{1-4\loExpo}/16384)$.

  \item From the definition of $F_1$ in Equation~\eqref{eq:lower_bound_def_x_F1}, we have
  \begin{align*}
  \smoothparam\vecnorm{x}{2}^2\leq\dims +  \dims^{1-4\loExpo} + 5 \dims^{\frac12}.
  \end{align*}
  
  Since $\vecnorm{\xi}{2}^2$ is Chi-square with $\dims$-degree of freedom,  standard Chi-square tail bound (Lemma 1 in~\cite{laurent2000adaptive}) shows that
  \begin{align*}
  \Prob\brackets{\vecnorm{\xi}{2}^2 \leq \dims - \frac{1}{64}\dims^{1-2\loExpo } }\leq \exp\parenth{-\frac{\dims^{1-4\loExpo}}{16384}}.
  \end{align*}
  Given that $g=\sqrt{2h}\xi$, we obtian
  \begin{align}\label{eq:bound_delta2}
      \Deltapart_2&=\parenth{\frac{\smoothparam^3 h^2}{2} - \frac{\smoothparam^4 h^3}{4}} \vecnorm{x}{2}^2- \frac{\smoothparam^2 h}{4} \vecnorm{g}{2}^2 \notag\\
      &\leq\frac{\smoothparam^2 h^2}{2} \cdot\parenth{\dims +  \dims^{1-4\loExpo} + 5 \dims^{\frac12}}-\frac{\smoothparam^2 h}{4}\parenth{2h\dims-\frac{1}{32}h\dims^{1-2\loExpo}}\notag\\
      &\leq \parenth{\frac12\smoothparam^2h^2\dims^{1-4\loExpo}+\frac52\smoothparam^2h^2\dims^{\frac12}}+\frac{1}{128}\smoothparam^2h^2\dims^{1-2\loExpo}\notag\\
      &\leq \frac1{32}\smoothparam^2 h^2\dims^{1-2\loExpo}+\frac{1}{128}\smoothparam^2 h^2\dims^{1-2\loExpo} \notag \\
      &=\frac{5}{128}\smoothparam^2 h^2\dims^{1-2\loExpo}
  \end{align}
  holds with probability at least $1-\exp\parenth{-\dims^{1-4\loExpo}/16384}$. The last line makes use of the assumption~\eqref{eq:condition_for_d}.

  \item Since $\angles{\smoothparam^{1/2} x, \xi}$ is $\vecnorm{\smoothparam^{1/2}x}{2}$-Lipschitz with respect to $\xi$, it is sub-Gaussian with the same Lipschitz constant. Using the tail bound for Lipschitz function of sub-Gaussian random variables, we have
  \begin{align*}
    \Prob\brackets{\abss{\angles{\smoothparam^{1/2} x, \xi}} \geq \frac{1}{64}\dims^{1-2\loExpo}} \leq 2 \exp\parenth{-\frac{\dims^{2-4\loExpo}}{8192\smoothparam\vecnorm{x}{2}^2}}\leq 2\exp\parenth{-\frac{\dims^{1-4\loExpo}}{16384}}.
  \end{align*}
  where we use the definition of $F_1$ and the assumption~\eqref{eq:condition_for_d} to get $\smoothparam\vecnorm{x}{2}^2\leq 2\dims$ in the last step. Now given that $\abss{\Deltapart_3}\leq\frac12\parenth{L^{\frac32}h^{\frac32}+L^{\frac52}h^{\frac52}}\abss{\angles{\smoothparam^{\frac12}x,\xi}}$, we have
  \begin{align}\label{eq:bound_delta3}
      \abss{\Deltapart_3}&\leq\frac1{128}\parenth{L^{\frac32}h^{\frac32}+L^{\frac52}h^{\frac52}}\dims^{1-2\loExpo}\notag\\
      &\leq \parenth{\frac{1}{128}\smoothparam h+\frac{1}{256}\smoothparam^2h^2+\frac{1}{128}\smoothparam^3 h^3}\dims^{1-2\loExpo}
  \end{align}
  with probability at least $1-2\exp\parenth{-\dims^{1-4\loExpo}/16384}$.

  \item For $\angles{\gradf_P(x),g}$, from the definition of $F_1$ in Equation~\eqref{eq:lower_bound_def_x_F1}, we have
  \begin{align*}
    \vecnorm{\gradf_P(x)}{2}^2&=\sum_{i=1}^d\frac{\smoothparam}{4\dims^{2\loExpo}}\sin^2(\dims^\loExpo\smoothparam^{\frac12}\coord{x}{i})\\
    &=\sum_{i=1}^d\frac{\smoothparam}{4\dims^{2\loExpo}}\parenth{1-\cos(2\dims^\loExpo\smoothparam^{\frac12}\coord{x}{i})}\\
      &\leq \frac 14{\smoothparam \dims^{1-2\loExpo}}+\frac14{\smoothparam\dims^{-2\loExpo}}\parenth{-\frac{1}{16}\dims^{1-2\loExpo}+\frac18\dims^{1-4\loExpo}+2\dims^{\frac12}}\\
      &\leq \frac{1}{2}\smoothparam \dims^{1-2\loExpo}-\frac{1}{64}\smoothparam\dims^{1-4\loExpo},
  \end{align*}
  where the last step uses the assumption~\eqref{eq:condition_for_d}. Since $\angles{\gradf_P(x),g}$ is $\sqrt{2h}\vecnorm{\gradf_P(x)}{2}$-Lipschitz with respect to $\xi$, it is sub-Gaussian with the same Lipschitz constant. We have
    \begin{align}\label{eq:fpxg}
        \Prob\brackets{\abss{\angles{\gradf_P(x),g}}\geq \frac{1}{64}\sqrt{\smoothparam h}\dims^{1-3\loExpo}}\leq 2\exp\parenth{-\frac{\smoothparam \dims^{2-6\loExpo}/64^2}{4\vecnorm{\gradf_P(x)}{2}^2}}\leq 2\exp\parenth{-\frac{\dims^{1-4\loExpo}}{8192}}
    \end{align}

    For $\angles{\gradf_P(y),g}$, we have
    \begin{align*}
    \angles{\gradf_P(y), g} &=  \sqrt{\smoothparam h}\sum_{i=1}^\dims   \underbrace{\frac{\coord{\xi}{i}}{\sqrt{2}{\dims^{\loExpo}}} \sin\parenth{\dims^\loExpo \smoothparam^{\frac12} \coord{y}{i}}}_{\Deltapart_{4,2,i}} 
    \end{align*}
    Applying Lemma~\ref{lem:31}, we can bound the moments of $\Deltapart_{4,2,i}$ as follows
    \begin{align*}
    \abss{\Exs\brackets{\Deltapart_{4,2,i}}}&\leq\sqrt{\smoothparam h}\exp\parenth{-\smoothparam h d^{2\loExpo}}\overset{(i)}{\leq}  \frac{1}{512}\sqrt{\smoothparam h}\dims^{-2\loExpo}\\
    \Exs\brackets{\abss{\Deltapart_{4,2,i}-\Exs[\Deltapart_{4,2,i}]}^\ell}&\leq \Exs\brackets{2^{\ell-1}\parenth{\abss{\Deltapart_{4,2,i}}^\ell+\abss{\Exs\brackets{\Deltapart_{4,2,i}}}^\ell}}\\
    &\leq2^{\ell-1}\parenth{\Exs\abss{\frac{\coord{\xi}{i}}{\sqrt{2}}}^\ell+\abss{\Exs\brackets{\Deltapart_{4,2,i}}}^\ell}\\
    &\overset{(ii)}{\leq}2^{\ell-1}\parenth{\frac{(\ell-1)!!}{2^{\frac\ell2}}+\parenth{\frac{1}{2e}}^{\frac\ell2}}\\
    &\leq \frac{\ell!}{2}2^{\ell-2}.
  \end{align*}
Step $(i)$ uses the assumption~\eqref{eq:condition_for_d}. Step (ii) uses the expected moments of the normal distribution for the first term, and applies $x\exp(-x^2)\leq(2e)^{-1/2}$ for $x=\sqrt{\smoothparam h}$ to bound the second term. Applying Bernstein's inequality~\eqref{eq:Bernstein_ineq0} with $K=2$ and $\epsilon=\dims^{-4\loExpo}/12800$, we have
  \begin{align}\label{eq:fpyg}
      \Prob\brackets{\angles{\gradf_P(y),g}\geq \frac{1}{512}\smoothparam h\dims^{1-2\loExpo}+\sqrt{\smoothparam h}\parenth{\frac{1}{6400}\dims^{1-5\loExpo}+\frac1{80}\dims^{1-3\loExpo}}}\leq \exp\parenth{-\frac{\dims^{1-4\loExpo}}{12800}}
  \end{align}
Combining \eqref{eq:fpxg} and \eqref{eq:fpyg}, we obtain
\begin{align}\label{eq:bound_delta4}
    \abss{\Deltapart_4}&=\abss{\angles{\gradf_P(x), \parenth{\frac{1+\smoothparam^2 h^2}{2} }g }  - \angles{\gradf_P(y), \frac{1 + \smoothparam h }{2} g}}\notag\\
    &\leq \frac{1}{64}\parenth{1+\smoothparam^2h^2}\sqrt{\smoothparam h}\dims^{1-3\loExpo} + \frac{1}{512}(1+\smoothparam h)\smoothparam h \dims^{1-2\loExpo}\notag\\
    &\quad  +\parenth{1+\smoothparam h}\sqrt{\smoothparam h}\parenth{\frac{1}{12800}\dims^{1-5\loExpo}+\frac1{160}\dims^{1-3\loExpo}}\notag\\
    &\overset{(i)}{\leq}\frac{1}{128}(1+\smoothparam^2h^2)\parenth{4\dims^{1-4\loExpo}+\frac14\smoothparam h\dims^{1-2\loExpo}}+ \frac{1}{512}(1+\smoothparam h)\smoothparam h\dims^{1-2\loExpo}\notag\\
    &\quad + \frac{1}{128}\parenth{1+\smoothparam h}\parenth{4\dims^{1-4\loExpo}+\frac14\smoothparam h\dims^{1-2\loExpo}}\notag\\
    &= \parenth{\frac1{16}+\frac{1}{32}\smoothparam h +\frac{1}{32}\smoothparam^2 h^2}\dims^{1-4\loExpo} + \parenth{\frac{3}{512}\smoothparam h +\frac{1}{256}\smoothparam^2 h^2+\frac{1}{512}\smoothparam^3 h^3}\dims^{1-2\loExpo}
\end{align}
with probability at least $1-3\exp\parenth{-\dims^{1-4\loExpo}/12800}$. Step $(i)$ uses the assumption~\eqref{eq:condition_for_d}.

  \item By the definition of $F_1$ in  \eqref{eq:lower_bound_def_x_F1},
  \begin{align}\label{eq:gradfPx_x}
      \abss{\angles{\gradf_P(x),x}}&=\abss{\sum_{i=1}^d\frac{\smoothparam^{\frac12}}{2\dims^{\loExpo}}\coord{x}{i}\sin(\dims^{\loExpo}\smoothparam^{\frac12}\coord{x}{i})}\notag\\
      &\leq \frac12\dims^{1-4\loExpo} + 2\dims^{\frac12}
 \end{align}
    We write $\angles{\gradf_P(y),x}$ as 
    \begin{align*}
         \angles{\gradf_P(y),x}&=\sum_{i=1}^d\underbrace{\frac{\smoothparam^{\frac12}}{2\dims^{\loExpo}}\coord{x}{i}\sin(\dims^{\loExpo}\smoothparam^{\frac12}\coord{y}{i})}_{\Deltapart_{5,i}}
    \end{align*}
    By the definition of $F_1$ and the assumption~\eqref{eq:condition_for_d},
    \begin{align*}
        \Exs\abss{\Deltapart_{5,i}}&\leq\frac{\smoothparam^{\frac12}\abss{\coord{x}{i}}}{2\dims^{\loExpo}}\exp\parenth{-\smoothparam h\dims^{2\loExpo}}\leq 2 \parenth{\log(8\dims)}^{\frac12}\dims^{-\loExpo}\exp\parenth{-\smoothparam h\dims^{2\loExpo}}\leq\frac{1}{512}\dims^{-2\loExpo}\\
        \Var\brackets{\Deltapart_{5,i}}&\leq\frac{\smoothparam}{4\dims^{2\loExpo}}\coord{x}{i}^2\leq 4\parenth{\log(8\dims)}\dims^{-2\loExpo}\\
        \abss{\Deltapart_{5,i}}&\leq2 \parenth{\log(8\dims)}^{\frac12}\dims^{-\loExpo}
    \end{align*}
    Applying Bernstein's inequality~\eqref{eq:Bernstein_ineq} with $\epsilon =\dims^{1-2\loExpo}/512$, we have 
    \begin{align}\label{eq:gradfPy_x}
        \Prob\brackets{\abss{\angles{\gradf_P(y),x}}\geq\frac{1}{256}\dims^{1-2\loExpo} }&\leq \exp\parenth{-\frac{\dims^{1-4\loExpo}}{4096}}
    \end{align}
    
    Combining Equation~\eqref{eq:gradfPx_x} and \eqref{eq:gradfPy_x}, we get
    \begin{align}\label{eq:bound_delta5}
        \Deltapart_5&=\angles{\gradf_P(x), \parenth{\frac{\smoothparam^2 h^2}{2} -  \frac{\smoothparam^3 h^3}{2} }x } - \angles{\gradf_P(y), \frac{\smoothparam h }{2} \parenth{2 - \smoothparam h}  x}\notag\\
        &\leq\parenth{\frac14\smoothparam^2h^2\dims^{1-4\loExpo}+\smoothparam^2h^2\dims^{\frac12}}+\parenth{\frac14\smoothparam^3h^3\dims^{1-4\loExpo}+\smoothparam^3h^3\dims^{\frac12}}\notag\\
        &\quad +\frac{1}{256}\smoothparam h\dims^{1-2\loExpo} + \frac1{512} \smoothparam^2h^2\dims^{1-2\loExpo}\notag\\
        &\overset{(i)}{\leq} \frac{1}{64}\smoothparam^2h^2\dims^{1-2\loExpo}+\frac{1}{64}\smoothparam^3h^3\dims^{1-2\loExpo}+\frac{1}{256}\smoothparam h\dims^{1-2\loExpo} + \frac1{512} \smoothparam^2h^2\dims^{1-2\loExpo}\notag\\
        &= \parenth{\frac{1}{256}\smoothparam h+\frac{9}{512}\smoothparam^2h^2+\frac{1}{64}\smoothparam^3h^3}\dims^{1-2\loExpo}
    \end{align}
    with probability at least $1-\exp\parenth{-\dims^{1-4\loExpo}/4096}$. Step $(i)$ uses the assumption~\eqref{eq:condition_for_d}.

  \item By Definition of $F_1$ \eqref{eq:lower_bound_def_x_F1}, we have
  \begin{align}\label{eq:bound_delta6}
    - \smoothparam h^2 \vecnorm{\gradf_P(x)}{2}^2 &=- \frac{\smoothparam^2 h^2}{4 \dims^{2\loExpo}}\sum_{i=1}^\dims  \sin^2\parenth{\dims^\loExpo \smoothparam^{\frac12} \coord{x}{i}}\notag\\
    &= - \frac{\smoothparam^2 h^2}{8 \dims^{2\loExpo}} \sum_{i=1}^\dims \parenth{1 - \cos\parenth{2\dims^\loExpo \smoothparam^{\frac12} \coord{x}{i}}}\notag\\
    &\leq -\frac18\smoothparam^2h^2\dims^{1-2\loExpo} -\frac{1}{128}\smoothparam^2h^2\dims^{1-4\loExpo}+\frac{1}{64}\smoothparam^2h^2\dims^{1-6\loExpo}+\frac{1}{4}\smoothparam^2h^2\dims^{\frac12-2\loExpo}\notag\\
    &\leq -\frac{1}{8}\smoothparam^2h^2\dims^{1-2\loExpo} -\frac{1}{128}\smoothparam^2h^2\dims^{1-4\loExpo}+\frac{1}{512}\smoothparam^2h^2\dims^{1-2\loExpo}
  \end{align}
  where the last step follows from the assumption~\eqref{eq:condition_for_d}.

  \item We decompose $\Deltapart_7$ into three parts.
  \begin{align*}
      \vecnorm{(1-\smoothparam h) \gradf_P(x)+\gradf_P(y)}{2}^2&=(1-\smoothparam h)^2\vecnorm{\gradf_P(x)}{2}^2+2(1-\smoothparam h)\angles{\gradf_P(x),\gradf_P(y)}\\
      &\quad +\vecnorm{\gradf_P(y)}{2}^2
  \end{align*}
  Similar to previous analysis of $\vecnorm{\gradf_P(x)}{2}^2$, we have
  \begin{align}\label{eq:delta_70i}
      -\vecnorm{\gradf_P(x)}{2}^2&=-\sum_{i=1}^d\frac{\smoothparam}{4\dims^{2\loExpo}}\sin^2(\dims^\loExpo\smoothparam^{\frac12}\coord{x}{i})=-\sum_{i=1}^d\frac{\smoothparam}{4\dims^{2\loExpo}}\parenth{1-\cos(2\dims^\loExpo\smoothparam^{\frac12}\coord{x}{i})}\notag\\
      &\leq -\frac 14{\smoothparam \dims^{1-2\loExpo}}-\frac{1}{64}\smoothparam \dims^{1-4\loExpo}+\frac1{32}\smoothparam\dims^{1-6\loExpo}+\frac12\smoothparam\dims^{\frac12-2\loExpo}\notag\\
      &\leq -\frac 14{\smoothparam \dims^{1-2\loExpo}}-\frac{1}{64}\smoothparam \dims^{1-4\loExpo} + \frac{1}{512}\smoothparam \dims^{1-2\loExpo}.
  \end{align}
  By Lemma \ref{lem:31} and the definition of $F_1$, we have
  \begin{align*}
      \angles{\gradf_P(x),\gradf_P(y)}&=\sum_{i=1}^d\underbrace{\frac{\smoothparam}{4\dims^{2\loExpo}}\sin(\dims^\loExpo\smoothparam^{\frac12}\coord{x}{i})\sin(\dims^\loExpo\smoothparam^{\frac12}\coord{y}{i})}_{\Deltapart_{7,1,i}}
  \end{align*}
  \begin{align*}
      \Exs\abss{\Deltapart_{7,1,i}}&\leq \frac{\smoothparam^{\frac32}}{4\dims^{\loExpo}}\abss{\coord{x}{i}}\exp\parenth{-\smoothparam \dims^{2\loExpo}h}\leq\smoothparam\parenth{\log(8\dims)}^{\frac12}{\dims^{-\loExpo}}\exp\parenth{-\smoothparam \dims^{2\loExpo}h}\leq\frac{1}{1024}\smoothparam\dims^{-2\loExpo}\\
      \Var[\Deltapart_{7,1,i}]&\leq \frac{\smoothparam^3}{16\dims^{2\loExpo}}\coord{x}{i}^2\leq \smoothparam^2\dims^{-2\loExpo}\log(8\dims)\\
      \abss{\Deltapart_{7,1,i}}&\leq\smoothparam\parenth{\log(8\dims)}^{\frac12}{\dims^{-\loExpo}}
  \end{align*}
  Applying Bernstein's inequality \eqref{eq:Bernstein_ineq} with $\epsilon = \smoothparam\dims^{1-2\loExpo}/1024$, we have
  \begin{align}\label{eq:delta_71i}
      \Prob\brackets{\abss{\angles{\gradf_P(x),\gradf_P(y)}}\geq \frac{1}{512}\smoothparam\dims^{1-2\loExpo}}\leq\exp\parenth{-\frac{\dims^{1-4\loExpo}}{16384}}
  \end{align}
  
   Similarly, we have
    \begin{align*}
        \vecnorm{\gradf_P(y)}{2}^2&=\sum_{i=1}^d\frac{\smoothparam}{4\dims^{2\loExpo}}\sin^2(\dims^\loExpo\smoothparam^{\frac12}\coord{y}{i})=\sum_{i=1}^d\underbrace{\frac{\smoothparam}{4\dims^{2\loExpo}}\parenth{1-\cos(2\dims^\loExpo\smoothparam^{\frac12}\coord{y}{i})}}_{\Deltapart_{7,2,i}}\\
        \abss{\Exs[\Deltapart_{7,2,i}]-\frac{\smoothparam}{4\dims^{2\loExpo}}}&\leq \frac{\smoothparam}{4\dims^{2\loExpo}}\exp\parenth{-2\smoothparam h\dims^{2\loExpo}}\leq\frac{1}{128}\smoothparam\dims^{-4\loExpo}\\
        \Var[\Deltapart_{7,2,i}^2]&\leq \frac{\smoothparam^2}{16\dims^{4\loExpo}}\\
        \abss{\Deltapart_{7,2,i}}&\leq \frac{\smoothparam}{4\dims^{2\loExpo}}
    \end{align*}
    Applying Bernstein's inequality \eqref{eq:Bernstein_ineq} with $\epsilon = {\smoothparam}\dims^{1-4\loExpo}/128$, we have
  \begin{align}\label{eq:delta_72i}
      \Prob\brackets{\vecnorm{\gradf_P(y)}{2}^2\geq \frac{1}{4}\smoothparam\dims^{1-2\loExpo}+\frac{1}{64}\smoothparam\dims^{1-4\loExpo}}\leq\exp\parenth{-\frac{\dims^{1-4\loExpo}}{16384}}
  \end{align}
    From these three estimates \eqref{eq:delta_70i}, \eqref{eq:delta_71i} and \eqref{eq:delta_72i}, we have
    \begin{align}\label{eq:bound_delta7}
        \Deltapart_7&=-\frac{h}{4}\vecnorm{(1-\smoothparam h)\gradf_P(x)+\gradf_P(y)}{2}^2\notag\\
        &\leq -\frac1{16} \smoothparam h(1-\smoothparam h)^2\dims^{1-2\loExpo} - \frac{1}{256}\smoothparam h(1-\smoothparam h)^2\dims^{1-4\loExpo}+\frac{1}{2048}\smoothparam h(1-\smoothparam h)^2\dims^{1-2\loExpo}\notag\\
        &\quad +\frac{1}{1024}\smoothparam h\parenth{1+\smoothparam h}\dims^{1-2\loExpo} -\frac{1}{16}\smoothparam h\dims^{1-2\loExpo}  +\frac{1}{256}\smoothparam h\dims^{1-4\loExpo}\notag\\
        &=\parenth{-\frac{253}{2048}\smoothparam h+\frac{1}{8}\smoothparam^2 h^2-\frac{127}{2048}\smoothparam^3 h^3}\dims^{1-2\loExpo} +\parenth{\frac{1}{128}\smoothparam^2 h^2-\frac{1}{256}\smoothparam^3h^3}\dims^{1-4\loExpo}
    \end{align}
    with probability at least $1-2\exp\parenth{-\dims^{1-4\loExpo}/16384}$.
\end{enumerate}

Finally, combining seven bounds in Equation \eqref{eq:bound_delta1}, \eqref{eq:bound_delta2}, \eqref{eq:bound_delta3}, \eqref{eq:bound_delta4}, \eqref{eq:bound_delta5}, \eqref{eq:bound_delta6} and \eqref{eq:bound_delta7}, we obtain
\begin{align}
    \sum_{i=1}^7\Delta_i&\leq-\frac{1}{32}\dims^{1-4\loExpo}+\frac{1}{2048}\parenth{-217\smoothparam h+136\smoothparam^2h^2-75\smoothparam^3h^3}\dims^{1-2\loExpo}\notag\\
    &\hspace{2em}+\parenth{\frac1{32}\smoothparam h+\frac{1}{32}\smoothparam^2h^2-\frac{1}{256}\smoothparam^3h^3}\dims^{1-4\loExpo}\notag\\
    &\overset{(i)}{\leq}-\frac{1}{32}\dims^{1-4\loExpo}-\frac{17}{512}(1-\smoothparam h)^2\dims^{1-2\loExpo}\notag\\
    &\leq -\frac{1}{32}\dims^{1-4\loExpo}
\end{align}
with probability at least $1-10\exp\parenth{-\dims^{1-4\loExpo}/16384}$ for $y\sim\Normal(x-h\gradf(x),2h\Ind_\dims)$. Step $(i)$ uses $\smoothparam h\dims^{1-4\loExpo}\leq\dims^{1-2\loExpo}/32$ and $\smoothparam^2h^2\dims^{1-4\loExpo}\leq \parenth{\smoothparam h\dims^{1-2\loExpo}+\smoothparam^3h^3\dims^{1-2\loExpo}}/64$ by the assumption~\eqref{eq:condition_for_d}, and throws all the other negative terms.
\end{proof}

\subsection{Proof of Lemma~\ref{lem:pi2}}\label{sub:pi2}
We calculate the integral explicitly as follows.

\begin{align*}
\int_\real\frac{\target_2(y)\propkernel_2(y,x)}{\target_2(x)}dy
&=\int_\real\frac{1}{\sqrt{4\pi h}}\exp\parenth{\frac{\scparam}{2}(x^2-y^2)}\exp\parenth{-\frac{1}{4h}(x-(1-\scparam h)y)^2}dy\\
&=\frac{1}{\sqrt{1+\scparam ^2h^2}}\exp\parenth{\frac{\scparam ^3h^2x^2}{2(1+\scparam ^2h^2)}}\\
&\quad \cdot\int_\real\sqrt{\frac{1+\scparam^2 h^2}{4\pi h }}\exp\parenth{-\frac{1+\scparam^2h^2}{4h}\bigparenth{y-\frac{1-\scparam h}{1+\scparam^2 h^2}x}^2}dy\\
&=\frac{1}{\sqrt{1+\scparam ^2h^2}}\exp\parenth{\frac{\scparam ^3h^2x^2}{2(1+\scparam ^2h^2)}}.
\end{align*}
Take $F_2=\braces{x:x\in(-1/{\sqrt{\scparam }},1/{\sqrt{\scparam }})}$, then $\target_2(F_2)\in (1/2,3/4)$. We have
\begin{align}
    \int_\real\frac{\target_2(y)\propkernel_2(y,x)}{\target_2(x)}dy\leq \exp\Big(\frac {\scparam}{2}x^2 \Big)\leq 2, \ \ \forall x\in F_2.\label{eq:pi2}
\end{align}


\vskip 0.2in
\bibliography{ref}

\begin{thebibliography}{59}
\providecommand{\natexlab}[1]{#1}
\providecommand{\url}[1]{\texttt{#1}}
\expandafter\ifx\csname urlstyle\endcsname\relax
  \providecommand{\doi}[1]{doi: #1}\else
  \providecommand{\doi}{doi: \begingroup \urlstyle{rm}\Url}\fi

\bibitem[Andrieu et~al.(2003)Andrieu, De~Freitas, Doucet, and
  Jordan]{andrieu2003introduction}
C.~Andrieu, N.~De~Freitas, A.~Doucet, and M.~I. Jordan.
\newblock An introduction to {MCMC} for machine learning.
\newblock \emph{Machine learning}, 50\penalty0 (1):\penalty0 5--43, 2003.

\bibitem[Bakry et~al.(2014)Bakry, Gentil, and Ledoux]{bakry2014analysis}
D.~Bakry, I.~Gentil, and M.~Ledoux.
\newblock \emph{Analysis and geometry of Markov diffusion operators}, volume
  103.
\newblock Springer, 2014.

\bibitem[Besag(1994)]{besag1994comments}
J.~Besag.
\newblock Comments on ``{Representations} of knowledge in complex systems'' by
  {U.} {Grenander} and {MI} {Miller}.
\newblock \emph{J. Roy. Statist. Soc. Ser. B}, 56:\penalty0 591--592, 1994.

\bibitem[Borgs(2003)]{borgs2003statistical}
C.~Borgs.
\newblock Statistical physics expansion methods in combinatorics and computer
  science.
\newblock \emph{CBMS lecture notes (in preparation)}, 2003.

\bibitem[Bou-Rabee and Hairer(2013)]{bou2013nonasymptotic}
N.~Bou-Rabee and M.~Hairer.
\newblock Nonasymptotic mixing of the {MALA} algorithm.
\newblock \emph{IMA Journal of Numerical Analysis}, 33\penalty0 (1):\penalty0
  80--110, 2013.

\bibitem[B{\"u}hlmann and Van De~Geer(2011)]{buhlmann2011statistics}
P.~B{\"u}hlmann and S.~Van De~Geer.
\newblock \emph{Statistics for high-dimensional data: methods, theory and
  applications}.
\newblock Springer Science \& Business Media, 2011.

\bibitem[Carpenter et~al.(2017)Carpenter, Gelman, Hoffman, Lee, Goodrich,
  Betancourt, Brubaker, Guo, Li, and Riddell]{carpenter2017stan}
B.~Carpenter, A.~Gelman, M.~D. Hoffman, D.~Lee, B.~Goodrich, M.~Betancourt,
  M.~A. Brubaker, J.~Guo, P.~Li, and A.~Riddell.
\newblock Stan: a probabilistic programming language.
\newblock \emph{Grantee Submission}, 76\penalty0 (1):\penalty0 1--32, 2017.

\bibitem[Chen et~al.(2020)Chen, Dwivedi, Wainwright, and Yu]{chen2020fast}
Y.~Chen, R.~Dwivedi, M.~J. Wainwright, and B.~Yu.
\newblock Fast mixing of {Metropolized} {Hamiltonian} {Monte} {Carlo}: Benefits
  of multi-step gradients.
\newblock \emph{Journal of Machine Learning Research}, 21\penalty0
  (92):\penalty0 1--71, 2020.

\bibitem[Cheng and Bartlett(2018)]{cheng2018convergence}
X.~Cheng and P.~Bartlett.
\newblock Convergence of {Langevin} {MCMC} in {KL}-divergence.
\newblock \emph{Proceedings of Machine Learning Research, Volume 83:
  Algorithmic Learning Theory}, pages 186--211, 2018.

\bibitem[Cheng et~al.(2018{\natexlab{a}})Cheng, Chatterji, Abbasi-Yadkori,
  Bartlett, and Jordan]{cheng2018sharp}
X.~Cheng, N.~S. Chatterji, Y.~Abbasi-Yadkori, P.~L. Bartlett, and M.~I. Jordan.
\newblock Convergence rates for {Langevin} dynamics in the nonconvex setting.
\newblock \emph{arXiv preprint arXiv:1805.01648}, 2018{\natexlab{a}}.

\bibitem[Cheng et~al.(2018{\natexlab{b}})Cheng, Chatterji, Bartlett, and
  Jordan]{cheng2018underdamped}
X.~Cheng, N.~S. Chatterji, P.~L. Bartlett, and M.~I. Jordan.
\newblock Underdamped {Langevin} {MCMC}: A non-asymptotic analysis.
\newblock In \emph{Conference on Learning Theory}, pages 300--323. PMLR,
  2018{\natexlab{b}}.

\bibitem[Chewi et~al.(2021)Chewi, Lu, Ahn, Cheng, Le~Gouic, and
  Rigollet]{chewi2021optimal}
S.~Chewi, C.~Lu, K.~Ahn, X.~Cheng, T.~Le~Gouic, and P.~Rigollet.
\newblock Optimal dimension dependence of the {Metropolis}-{Adjusted}
  {Langevin} {Algorithm}.
\newblock In \emph{Conference on Learning Theory}, pages 1260--1300. PMLR,
  2021.

\bibitem[Coulhon(1996)]{coulhon1996ultracontractivity}
T.~Coulhon.
\newblock Ultracontractivity and nash type inequalities.
\newblock \emph{Journal of functional analysis}, 141\penalty0 (2):\penalty0
  510--539, 1996.

\bibitem[Coulhon and Grigor’yan(1997)]{coulhon1997diagonal}
T.~Coulhon and A.~Grigor’yan.
\newblock On-diagonal lower bounds for heat kernels and markov chains.
\newblock \emph{Duke Mathematical Journal}, 89\penalty0 (1):\penalty0 133--199,
  1997.

\bibitem[Coulhon et~al.(2001)Coulhon, Grigor'yan, and
  Pittet]{coulhon2001geometric}
T.~Coulhon, A.~Grigor'yan, and C.~Pittet.
\newblock A geometric approach to on-diagonal heat kernel lower bounds on
  groups.
\newblock In \emph{Annales de l'institut Fourier}, volume~51, pages 1763--1827,
  2001.

\bibitem[Cousins and Vempala(2014)]{cousins2014cubic}
B.~Cousins and S.~Vempala.
\newblock A cubic algorithm for computing {Gaussian} volume.
\newblock In \emph{Proceedings of the twenty-fifth annual ACM-SIAM symposium on
  discrete algorithms}, pages 1215--1228. SIAM, 2014.

\bibitem[Dalalyan(2017)]{dalalyan2017theoretical}
A.~S. Dalalyan.
\newblock Theoretical guarantees for approximate sampling from smooth and
  log-concave densities.
\newblock \emph{Journal of the Royal Statistical Society: Series B (Statistical
  Methodology)}, 79\penalty0 (3):\penalty0 651--676, 2017.

\bibitem[Dalalyan and Riou-Durand(2020)]{dalalyan2020sampling}
A.~S. Dalalyan and L.~Riou-Durand.
\newblock On sampling from a log-concave density using kinetic {Langevin}
  diffusions.
\newblock \emph{Bernoulli}, 26\penalty0 (3):\penalty0 1956--1988, 2020.

\bibitem[Diaconis and Saloff-Coste(1996)]{diaconis1996logarithmic}
P.~Diaconis and L.~Saloff-Coste.
\newblock {Logarithmic} {Sobolev} inequalities for finite {Markov} chains.
\newblock \emph{The Annals of Applied Probability}, 6\penalty0 (3):\penalty0
  695--750, 1996.

\bibitem[Durmus and Moulines(2017)]{durmus2017nonasymptotic}
A.~Durmus and E.~Moulines.
\newblock Nonasymptotic convergence analysis for the unadjusted {Langevin}
  algorithm.
\newblock \emph{The Annals of Applied Probability}, 27\penalty0 (3):\penalty0
  1551--1587, 2017.

\bibitem[Durmus and Moulines(2019)]{durmus2019high}
A.~Durmus and E.~Moulines.
\newblock High-dimensional {Bayesian} inference via the unadjusted {Langevin}
  algorithm.
\newblock \emph{Bernoulli}, 25\penalty0 (4A):\penalty0 2854--2882, 2019.

\bibitem[Dwivedi et~al.(2019)Dwivedi, Chen, Wainwright, and Yu]{dwivedi2018log}
R.~Dwivedi, Y.~Chen, M.~J. Wainwright, and B.~Yu.
\newblock Log-concave sampling: Metropolis-hastings algorithms are fast.
\newblock \emph{Journal of Machine Learning Research}, 20\penalty0
  (183):\penalty0 1--42, 2019.

\bibitem[Eberle et~al.(2019)Eberle, Guillin, and Zimmer]{eberle2019couplings}
A.~Eberle, A.~Guillin, and R.~Zimmer.
\newblock Couplings and quantitative contraction rates for {Langevin} dynamics.
\newblock \emph{The Annals of Probability}, 47\penalty0 (4):\penalty0
  1982--2010, 2019.

\bibitem[Erdogdu et~al.(2021)Erdogdu, Hosseinzadeh, and
  Zhang]{erdogdu2021convergence}
M.~A. Erdogdu, R.~Hosseinzadeh, and M.~S. Zhang.
\newblock Convergence of {Langevin} {Monte} {Carlo} in {Chi-Squared} and
  {Renyi} divergence, 2021.

\bibitem[Goel et~al.(2006)Goel, Montenegro, Tetali, et~al.]{goel2006mixing}
S.~Goel, R.~Montenegro, P.~Tetali, et~al.
\newblock Mixing time bounds via the spectral profile.
\newblock \emph{Electronic Journal of Probability}, 11:\penalty0 1--26, 2006.

\bibitem[Grenander and Miller(1994)]{grenander1994representations}
U.~Grenander and M.~I. Miller.
\newblock Representations of knowledge in complex systems.
\newblock \emph{Journal of the Royal Statistical Society: Series B
  (Methodological)}, 56\penalty0 (4):\penalty0 549--581, 1994.

\bibitem[Hastings(1970)]{hastings1970monte}
W.~K. Hastings.
\newblock Monte carlo sampling methods using markov chains and their
  applications.
\newblock \emph{Biometrika}, 57\penalty0 (1):\penalty0 97--109, 1970.

\bibitem[Laurent and Massart(2000)]{laurent2000adaptive}
B.~Laurent and P.~Massart.
\newblock Adaptive estimation of a quadratic functional by model selection.
\newblock \emph{Annals of Statistics}, pages 1302--1338, 2000.

\bibitem[Ledoux(2001)]{ledoux2001concentration}
M.~Ledoux.
\newblock \emph{The concentration of measure phenomenon}, volume~89.
\newblock American Mathematical Soc., 2001.

\bibitem[Lee et~al.(2020)Lee, Shen, and Tian]{lee2020logsmooth}
Y.~T. Lee, R.~Shen, and K.~Tian.
\newblock Logsmooth gradient concentration and tighter runtimes for
  {Metropolized} {Hamiltonian} {Monte} {Carlo}.
\newblock In \emph{Conference on Learning Theory}, pages 2565--2597. PMLR,
  2020.

\bibitem[Lee et~al.(2021{\natexlab{a}})Lee, Shen, and Tian]{lee2021lower}
Y.~T. Lee, R.~Shen, and K.~Tian.
\newblock Lower bounds on {Metropolized} sampling methods for well-conditioned
  distributions.
\newblock \emph{arXiv preprint arXiv:2106.05480}, 2021{\natexlab{a}}.

\bibitem[Lee et~al.(2021{\natexlab{b}})Lee, Shen, and Tian]{lee2021structured}
Y.~T. Lee, R.~Shen, and K.~Tian.
\newblock Structured logconcave sampling with a restricted gaussian oracle.
\newblock In \emph{Conference on Learning Theory}, pages 2993--3050. PMLR,
  2021{\natexlab{b}}.

\bibitem[Levin and Peres(2017)]{levin2017markov}
D.~A. Levin and Y.~Peres.
\newblock \emph{Markov chains and mixing times}, volume 107.
\newblock American Mathematical Soc., 2017.

\bibitem[Li et~al.(2020)Li, Zha, and Tao]{li2020hessian}
R.~Li, H.~Zha, and M.~Tao.
\newblock Hessian-free high-resolution {Nesterov} acceleration for sampling.
\newblock \emph{arXiv e-prints}, pages arXiv--2006, 2020.

\bibitem[Lov{\'a}sz and Simonovits(1993)]{lovasz1993random}
L.~Lov{\'a}sz and M.~Simonovits.
\newblock Random walks in a convex body and an improved volume algorithm.
\newblock \emph{Random structures \& algorithms}, 4\penalty0 (4):\penalty0
  359--412, 1993.

\bibitem[Ma et~al.(2021)Ma, Chatterji, Cheng, Flammarion, Bartlett, and
  Jordan]{ma2021there}
Y.-A. Ma, N.~S. Chatterji, X.~Cheng, N.~Flammarion, P.~L. Bartlett, and M.~I.
  Jordan.
\newblock Is there an analog of {Nesterov} acceleration for gradient-based
  {MCMC}?
\newblock \emph{Bernoulli}, 27\penalty0 (3):\penalty0 1942--1992, 2021.

\bibitem[Mangoubi and Vishnoi(2019)]{mangoubi2019nonconvex}
O.~Mangoubi and N.~K. Vishnoi.
\newblock Nonconvex sampling with the {Metropolis}-{Adjusted} {Langevin}
  {Algorithm}.
\newblock In \emph{Conference on Learning Theory}, pages 2259--2293. PMLR,
  2019.

\bibitem[Mengersen et~al.(1996)Mengersen, Tweedie, et~al.]{mengersen1996rates}
K.~L. Mengersen, R.~L. Tweedie, et~al.
\newblock Rates of convergence of the {Hastings} and {Metropolis} algorithms.
\newblock \emph{Annals of Statistics}, 24\penalty0 (1):\penalty0 101--121,
  1996.

\bibitem[Metropolis et~al.(1953)Metropolis, Rosenbluth, Rosenbluth, Teller, and
  Teller]{metropolis1953equation}
N.~Metropolis, A.~W. Rosenbluth, M.~N. Rosenbluth, A.~H. Teller, and E.~Teller.
\newblock Equation of state calculations by fast computing machines.
\newblock \emph{The journal of chemical physics}, 21\penalty0 (6):\penalty0
  1087--1092, 1953.

\bibitem[Meyn and Tweedie(2012)]{meyn2012markov}
S.~P. Meyn and R.~L. Tweedie.
\newblock \emph{Markov chains and stochastic stability}.
\newblock Springer Science \& Business Media, 2012.

\bibitem[Montenegro(2006)]{montenegro2006eigenvalues}
R.~Montenegro.
\newblock Eigenvalues of non-reversible {Markov} chains: {Their} connection to
  mixing times, reversible {Markov} chains, and {Cheeger} inequalities.
\newblock \emph{arXiv preprint math/0604362}, 2006.

\bibitem[Montenegro and Tetali(2006)]{montenegro2006mathematical}
R.~R. Montenegro and P.~Tetali.
\newblock \emph{Mathematical aspects of mixing times in {Markov} chains}.
\newblock Now Publishers Inc, 2006.

\bibitem[Mortici(2011)]{mortici2011improved}
C.~Mortici.
\newblock Improved asymptotic formulas for the {Gamma} function.
\newblock \emph{Computers \& Mathematics with Applications}, 61\penalty0
  (11):\penalty0 3364--3369, 2011.

\bibitem[Mou et~al.(2021)Mou, Ma, Wainwright, Bartlett, and
  Jordan]{mou2021high}
W.~Mou, Y.-A. Ma, M.~J. Wainwright, P.~L. Bartlett, and M.~I. Jordan.
\newblock High-order {Langevin} diffusion yields an accelerated {MCMC}
  algorithm.
\newblock \emph{Journal of Machine Learning Research}, 22:\penalty0 42--1,
  2021.

\bibitem[Neal et~al.(2011)]{neal2011mcmc}
R.~M. Neal et~al.
\newblock {MCMC} using {Hamiltonian} dynamics.
\newblock \emph{Handbook of {Markov} {Chain} {Monte} {Carlo}}, 2\penalty0
  (11):\penalty0 2, 2011.

\bibitem[Nemirovskij and Yudin(1983)]{nemirovskij1983problem}
A.~S. Nemirovskij and D.~B. Yudin.
\newblock \emph{Problem complexity and method efficiency in optimization}.
\newblock Wiley-Interscience, 1983.

\bibitem[Nesterov(2003)]{nesterov2003introductory}
Y.~Nesterov.
\newblock \emph{Introductory lectures on convex optimization: A basic course},
  volume~87.
\newblock Springer Science \& Business Media, 2003.

\bibitem[Parisi(1981)]{parisi1981correlation}
G.~Parisi.
\newblock Correlation functions and computer simulations.
\newblock \emph{Nuclear Physics B}, 180\penalty0 (3):\penalty0 378--384, 1981.

\bibitem[Plummer et~al.(2003)]{plummer2003jags}
M.~Plummer et~al.
\newblock {JAGS}: A program for analysis of {Bayesian} graphical models using
  {Gibbs} sampling.
\newblock In \emph{Proceedings of the 3rd international workshop on distributed
  statistical computing}, volume 124, pages 1--10. Vienna, Austria., 2003.

\bibitem[Robert and Casella(2013)]{robert2013monte}
C.~Robert and G.~Casella.
\newblock \emph{Monte Carlo statistical methods}.
\newblock Springer Science \& Business Media, 2013.

\bibitem[Roberts and Rosenthal(1998)]{roberts1998optimal}
G.~O. Roberts and J.~S. Rosenthal.
\newblock Optimal scaling of discrete approximations to {Langevin} diffusions.
\newblock \emph{Journal of the Royal Statistical Society: Series B (Statistical
  Methodology)}, 60\penalty0 (1):\penalty0 255--268, 1998.

\bibitem[Roberts and Stramer(2002)]{roberts2002langevin}
G.~O. Roberts and O.~Stramer.
\newblock Langevin diffusions and {Metropolis}-{Hastings} algorithms.
\newblock \emph{Methodology and computing in applied probability}, 4\penalty0
  (4):\penalty0 337--357, 2002.

\bibitem[Roberts and Tweedie(1996{\natexlab{a}})]{roberts1996exponential}
G.~O. Roberts and R.~L. Tweedie.
\newblock Exponential convergence of {Langevin} distributions and their
  discrete approximations.
\newblock \emph{Bernoulli}, pages 341--363, 1996{\natexlab{a}}.

\bibitem[Roberts and Tweedie(1996{\natexlab{b}})]{roberts1996geometric}
G.~O. Roberts and R.~L. Tweedie.
\newblock Geometric convergence and central limit theorems for multidimensional
  {Hastings} and {Metropolis} algorithms.
\newblock \emph{Biometrika}, 83\penalty0 (1):\penalty0 95--110,
  1996{\natexlab{b}}.

\bibitem[Schmidler and Woodard(2011)]{schmidler2011lower}
S.~Schmidler and D.~B. Woodard.
\newblock Lower bounds on the convergence rates of adaptive {MCMC} methods,
  2011.

\bibitem[Sinclair and Jerrum(1989)]{sinclair1989approximate}
A.~Sinclair and M.~Jerrum.
\newblock Approximate counting, uniform generation and rapidly mixing {Markov}
  chains.
\newblock \emph{Information and Computation}, 82\penalty0 (1):\penalty0
  93--133, 1989.

\bibitem[Wainwright(2019)]{wainwright2019high}
M.~J. Wainwright.
\newblock \emph{High-dimensional statistics: A non-asymptotic viewpoint},
  volume~48.
\newblock Cambridge University Press, 2019.

\bibitem[Wilson(2004)]{wilson2004mixing}
D.~B. Wilson.
\newblock Mixing times of lozenge tiling and card shuffling {Markov} chains.
\newblock \emph{The Annals of Applied Probability}, 14\penalty0 (1):\penalty0
  274--325, 2004.

\bibitem[Woodard et~al.(2009)Woodard, Schmidler, and
  Huber]{woodard2009sufficient}
D.~Woodard, S.~Schmidler, and M.~Huber.
\newblock Sufficient conditions for torpid mixing of parallel and simulated
  tempering.
\newblock \emph{Electronic Journal of Probability}, 14:\penalty0 780--804,
  2009.

\end{thebibliography}

\end{document}